\documentclass[a4paper, 11pt, twoside]{article}
\usepackage{fullpage}

\usepackage[utf8]{inputenc} 
\usepackage[T1]{fontenc}    
\input{preamble.tex}

\newcommand{\squeeze}{}

\usepackage[authoryear]{natbib}

\begin{document}

\title{\textbf{Bernoulli-LoRA: A Theoretical Framework  \\ for Randomized Low-Rank Adaptation}}

\author{Igor Sokolov$^{1}$ \quad
 Abdurakhmon Sadiev$^1$ \quad
 Yury Demidovich$^1$ \\ 
 Fawaz S Al-Qahtani$^2$ \quad 
 Peter Richt\'{a}rik$^1$}
\date{$^1$ Center of Excellence for Generative AI, \\ King Abdullah University of Science and Technology (KAUST), Saudi Arabia\\
$^2$ Saudi Data \& AI Authority (SDAIA) \& National Center of AI (NCAI), Saudi Arabia
}

\maketitle
\begin{abstract}
Parameter-efficient fine-tuning (PEFT) has emerged as a crucial approach for adapting large foundational models to specific tasks, particularly as model sizes continue to grow exponentially. Among PEFT methods, \algname{Low-Rank Adaptation (LoRA)}~\citep{hu2021loralowrankadaptationlarge} stands out for its effectiveness and simplicity, expressing adaptations as a product of two low-rank matrices. While extensive empirical studies demonstrate \algname{LoRA}'s practical utility, theoretical understanding of such methods remains limited. Recent work on \algname{RAC-LoRA}~\citep{malinovsky2024randomizedasymmetricchainlora} took initial steps toward rigorous analysis. In this work, we introduce \algname{Bernoulli-LoRA}, a novel theoretical framework that unifies and extends existing \algname{LoRA} approaches. Our method introduces a probabilistic Bernoulli mechanism for selecting which matrix to update. This approach encompasses and generalizes various existing update strategies while maintaining theoretical tractability. Under standard assumptions from non-convex optimization literature, we analyze several variants of our framework: \algname{Bernoulli-LoRA-GD}, \algname{Bernoulli-LoRA-SGD}, \algname{Bernoulli-LoRA-PAGE}, and \algname{Bernoulli-LoRA-MVR}, \algname{Bernoulli-LoRA-QGD}, \algname{Bernoulli-LoRA-MARINA}, \algname{Bernoulli-LoRA-EF21}, establishing convergence guarantees for each variant. Additionally, we extend our analysis to convex non-smooth functions, providing convergence rates for both constant and adaptive (Polyak-type) stepsizes. Through extensive experiments on various tasks, we validate our theoretical findings and demonstrate the practical efficacy of our approach. This work is a step toward developing theoretically grounded yet practically effective PEFT methods.
\end{abstract}

\tableofcontents

\section{Introduction}
Fine-tuning is a transfer learning method, where a pre-trained neural network is trained on a new dataset. In modern deep learning, adapting large models to specific tasks via fine-tuning has become central, especially in natural language processing~\citep{peters2018, devlin2019}. While full fine-tuning often yields strong results, it is computationally intensive for large models. Parameter-Efficient Fine-Tuning (PEFT)~\citep{he2022unifiedviewparameterefficienttransfer} addresses this by updating only a small subset of parameters\citep{coordinatedescent2016, demidovich2023mastmodelagnosticsparsifiedtraining}, often with task-specific layers trained from scratch. PEFT offers performance close to full fine-tuning with reduced training time and resource use~\citep{Radford2019LanguageMA, LMFewShot, han2024parameterefficientfinetuninglargemodels}, making it widely adopted in practice. Research on PEFT, especially for large foundation and language models, is rapidly growing.

Pre-trained models are known to have an inherently low intrinsic dimensionality~\citep{li2018measuringintrinsicdimensionobjective, aghajanyan2020intrinsicdimensionalityexplainseffectiveness}. This means that fine-tuning can be effectively achieved within a reduced-dimensional subspace rather than the full parameter space. Among the various methods for utilizing this property, \algname{Low-Rank Adaptation} (\algname{LoRA})~\citep{hu2021loralowrankadaptationlarge} stands out as the most prominent reparameterization technique. \algname{LoRA} minimizes the need to update an entire large, dense weight matrix by leveraging the product of two trainable low-rank matrices. This method significantly reduces the number of parameters required for fine-tuning. The low-rank matrices are optimized so that their scaled product serves as the update applied to the weight matrix:
\begin{equation*}
\squeeze
    W = W^0 + \frac{\alpha}{r}BA,
\end{equation*}
where $W^0\in\R^{m\times n},$ $B\in\R^{m\times r},$ and $A\in\R^{r\times n}.$ The pre-trained weight matrix $W^0$ remains fixed, while $A$ and $B$ are the trainable matrices. Typically, $A$ is initialized randomly using a Gaussian distribution, while $B$ is set to zero to ensure that $\Delta W = 0$ at the start. Various alternative initialization strategies have been investigated by~\citep{zhu2024asymmetrylowrankadaptersfoundation, hayou2024impactinitializationlorafinetuning, meng2024pissaprincipalsingularvalues, wang2024miloraharnessingminorsingular}. 
The parameters of \algname{LoRA} include the low rank $r$ and the scaling factor $\alpha$. Since the dimensions $m$ and $n$ in deep learning models are usually large, selecting $r \ll \min\{m, n\}$ drastically reduces the number of trainable parameters. The scaling factor $\alpha$ acts as the stepsize. While \algname{LoRA} may not always achieve the performance of full fine-tuning, it is more effective at mitigating forgetting when compared to traditional regularization techniques like weight decay and dropout. Additionally, it enhances diversity in generated outputs~\citep{biderman2024loralearnsforgets}. Furthermore, \algname{LoRA} is straightforward to implement and achieves performance comparable to full fine-tuning across a wide range of downstream tasks~\citep{hu2021loralowrankadaptationlarge}. Complementing its algorithmic utility, research has also focused on enhancing the computational efficiency of \algname{LoRA}; for instance, \citet{cherniuk2023run} demonstrated that by optimizing the computation graph of LoRA operations based on layer dimensions and rank, significant speedups and memory savings can be achieved without sacrificing accuracy. For a detailed summary of recent advancements in \algname{LoRA}, refer to~\citep{MaoSurvey}.

To bridge the gap between full fine-tuning and \algname{LoRA}, \citet{xia2024chainloraefficientfinetuning} introduced \algname{Chain of LoRA (COLA)}, an iterative optimization framework that enhances model weights via higher-rank representations composed of multiple low-rank components—without added computational or memory cost. \algname{COLA} incrementally refines low-rank approximations by training a sequence of \algname{LoRA} modules through structured fine-tuning, merging, and extension. Each iteration adds a new low-rank component, forming a chain whose length reflects the number of optimized modules. The core idea involves applying \algname{LoRA} updates iteratively over $T$ steps: training a module, integrating its updates into fixed parameters, re-initializing, and repeating. This cyclic process builds higher-rank augmentations efficiently. In essence, \algname{COLA} applies successive \algname{LoRA} updates as: \begin{equation*} \squeeze W = W^0 + \frac{\alpha}{r}\sum \limits_{t=0}^{T-1}B^tA^t. \end{equation*} Each pair $(A^t, B^t)$ is initialized like standard \algname{LoRA}. Unlike traditional \algname{LoRA}, which may struggle with non-low-rank adaptations, \algname{COLA} uses sequential low-rank decompositions to approximate updates of intermediate-to-high rank. This leads to more accurate and efficient adaptation while simplifying optimization by avoiding a direct high-rank fit.

\section{Problem Statement}
Supervised machine learning typically frames the training process as an optimization problem, aiming to minimize a loss function that quantifies the discrepancy between model predictions and true targets. This research focuses on the intricacies of this optimization challenge within the fine-tuning paradigm, where a pre-trained model is adapted to a new, specific task or dataset. Effective fine-tuning hinges on making precise and efficient modifications to the model's parameters to enhance its performance on the target task. We explore a generalized formulation of this problem, which is independent of the specific architecture of the underlying model:
\begin{equation}\label{eq:problem-statement}
\squeeze
    \min_{\Delta W\in\R^{m \times n}}f(W^0+\Delta W).
\end{equation}
Here, $W^0 \in \mathbb{R}^{m \times n}$ represents the parameters of the pre-trained model (or, for instance, the parameters of a single linear layer if other layers are kept constant), and $\Delta W \in \mathbb{R}^{m \times n}$ is the adaptation term whose optimal value we seek. The function $f: \mathbb{R}^{m \times n} \to \mathbb{R}$ denotes the empirical loss computed over the specific target dataset. Given that the dimensionality $m \times n$ is typically very large in contemporary deep learning models, the adjustment $\Delta W$ must possess a sufficiently simple structure to be practically trainable and applicable.

For the stochastic optimization methods developed and analyzed in this paper, we consider objective functions with one of the following specific structures:
\begin{itemize}
    \item \textbf{Finite-Sum Setting:} The objective is an average of individual loss functions, a structure we address with methods like \algname{Bernoulli-LoRA-PAGE}:
    \begin{equation}\label{eq:problem-finite-sum}
    \squeeze
        f(W^0 + \Delta W) = \frac{1}{N}\sum_{i=1}^N f_i(W^0 + \Delta W),
    \end{equation}
    where each $f_i$ corresponds to the loss on a single data sample, and $N$ is the total number of data points in the training set.
    \item \textbf{Expectation Setting:} The objective is an expectation over a data distribution $\mathcal{D}$, relevant for methods such as \algname{Bernoulli-LoRA-MVR}:
    \begin{equation}\label{eq:problem-stoch-setting}
    \squeeze
        f(W^0 + \Delta W) = \mathbb{E}_{\xi\sim \mathcal{D}}\left[f_{\xi}(W^0+\Delta W)\right],
    \end{equation}
    where $f_{\xi}$ is the loss function associated with a data sample $\xi$ drawn from $\mathcal{D}$.
\end{itemize}

Moreover, this paper extends its investigation to the \textbf{distributed optimization setting}, which is central to the Federated Learning (FL) algorithms we propose (e.g., \algname{Fed-Bernoulli-LoRA-QGD}, \algname{Fed-Bernoulli-LoRA-MARINA}, and \algname{Fed-Bernoulli-LoRA-EF21}). In this context, we address problems formulated as:
\begin{equation}\label{eq:problem-distributed-setting}
\squeeze
    f(W^0 + \Delta W) = \frac{1}{M}\sum_{l=1}^M f_l(W^0 + \Delta W),
\end{equation}
where $M$ is the total number of participating clients, and $f_l$ represents the local loss function for client $l$, defined over its private dataset. The goal is to find a common adaptation $\Delta W$ that minimizes this global, federated objective.

\section{Motivation}
Despite the widespread adoption and empirical success of \algname{Low-Rank Adaptation (LoRA)} and its variants like \algname{Chain of LoRA (COLA)}, a comprehensive theoretical understanding underpinning these prevalent fine-tuning methods remains largely undeveloped. Several critical issues highlight this gap. Firstly, as pointed out by~\citet{sun2024improvingloraprivacypreservingfederated}, the \algname{LoRA} re-parameterization inherently transforms a smooth Lipschitz loss into a non-smooth one. This alteration introduces significant theoretical complexities beyond those associated with managing the low-rank structure of updates, forming a key barrier to establishing robust theoretical frameworks. Secondly, the existing theoretical analysis of \algname{COLA} by~\citet{xia2024chainloraefficientfinetuning} sidesteps the core mechanism of low-rank updates by focusing on full-rank matrix optimization ($\Delta W$). Such an approach is unsatisfactory as it fails to model or explain the very essence of \algname{LoRA}'s efficiency.

Consequently, most methods based on \algname{LoRA} are, in essence, heuristics, developed through empirical investigation without strong theoretical convergence guarantees. This is problematic, as these methods can be highly sensitive to hyperparameter choices~\citep{Khodak2021FederatedHT, Kuang2024}, and their reliability beyond current empirical validation is not assured. In fact, \citet{malinovsky2024randomizedasymmetricchainlora} provided a concrete example of \algname{COLA}'s potential divergence, further underscoring its heuristic nature. Their work introduced \algname{RAC-LoRA}, the first comprehensive optimization framework designed to rigorously evaluate and establish convergence rates for methods utilizing \algname{LoRA}-style updates, marking a significant step towards theoretically grounded PEFT.

However, while \algname{RAC-LoRA} provides a foundational theoretical lens, its scope does not encompass several critical aspects of modern optimization, particularly for non-convex problems and distributed settings. Specifically, the \algname{RAC-LoRA} framework does not utilize optimal variance-reduced techniques for non-convex optimization, nor does it delve into sophisticated Federated Learning (FL) setting that incorporate crucial practical techniques such as communication compression~\citep{alistarh2018convergence, Terngrad, Horvth2019NaturalCF,panferov2024correlatedquantizationfasternonconvex} and error feedback. Federated learning~\citep{Konecn2016FederatedLS,FEDOPT, McMahan2016CommunicationEfficientLO, Kairouz2019AdvancesAO} is a decentralized paradigm where multiple clients collaboratively train a model on their local, private data. The growing demand for training massive  deep neural networks with billions of parameters on vast datasets~\citep{LMFewShot, Kolesnikov2019BigT} has intensified the ML community's interest in distributed optimization. To achieve feasible training times~\citep{ChuanLi}, distributing computation, especially stochastic gradient evaluations, is essential, driving the adoption of scalable algorithms~\citep{Goyal2017AccurateLM, You2019LargeBO, le2023bloom}. Our work is motivated by the need to bridge this gap by extending a theoretically sound LoRA framework to these advanced and practically vital optimization scenarios.

\section{Contributions}
The performance of \algname{LoRA}-based methods is notably sensitive to the selection of hyperparameters~\citep{Khodak2021FederatedHT, Kuang2024}, and a robust theoretical understanding to guide their application is still developing. While recent work by~\citet{malinovsky2024randomizedasymmetricchainlora} on \algname{RAC-LoRA} provided initial steps towards a rigorous analytical framework, we aim to further advance the theoretical foundations and practical versatility of low-rank adaptation techniques.

In PEFT approaches based on low-rank adaptation, two matrices, $A$ and $B,$ are typically updated. Existing methods may update only $A,$ only $B,$ or alternate between them deterministically~\citep{malinovsky2024randomizedasymmetricchainlora, xia2024chainloraefficientfinetuning, zhu2024asymmetrylowrankadaptersfoundation}. Our primary contribution is the introduction of \algname{Bernoulli-LoRA}, a novel and generic low-rank adaptation framework. \algname{Bernoulli-LoRA} is characterized by its unique probabilistic update mechanism: at each step of the adaptation process, a Bernoulli trial (akin to a coin flip) determines which of the two matrices ($A$ or $B$) is selected for optimization, while the other matrix is sampled from a predefined distribution and remains fixed for that step. This randomized selection not only provides a flexible approach but also unifies and generalizes several existing update strategies within a single theoretical construct. Much like the iterative design of \algname{COLA}~\citep{xia2024chainloraefficientfinetuning}, the \algname{Bernoulli-LoRA} framework operates by applying a sequence of such probabilistically chosen low-rank updates.

Our theoretical analysis is grounded in standard assumptions common in non-convex optimization literature, such as $L$-smoothness of the objective function. We instantiate the \algname{Bernoulli-LoRA} framework by developing and analyzing several distinct algorithmic variants. These variants span a range of optimization techniques, from foundational gradient-based methods to more advanced stochastic, variance-reduced, and federated learning algorithms, each designed to address specific challenges in modern machine learning. For every proposed method within the \algname{Bernoulli-LoRA} framework, we establish rigorous convergence guarantees. Our key contributions, which advance the theoretical understanding and practical applicability of LoRA-type methods, include:

\renewcommand\labelitemi{\ding{117}}
\begin{itemize}[leftmargin=*]
    \item \textbf{Foundational Algorithmic Variants:} We begin by establishing the theoretical properties of \algname{Bernoulli-LoRA} with two fundamental optimization schemes. These methods lay the groundwork for understanding how the randomized selection of $A$ or $B$ interacts with standard descent procedures in the context of low-rank updates.
    \begin{itemize}[leftmargin=*, itemsep=0pt, topsep=3pt, parsep=3pt]
        \item \algname{Bernoulli-LoRA-GD} (Algorithm~\ref{alg:Bernoulli-LoRA-GD-smooth}) serves as the simplest instantiation, employing full gradient descent to update the trainable low-rank matrix. While computing the full gradient is often impractical for large-scale models, this variant provides crucial foundational understanding of the framework's convergence behavior under idealized conditions, navigating the optimization landscape defined by the LoRA reparameterization.
        \item \algname{Bernoulli-LoRA-SGD} (Algorithm~\ref{alg:Bernoulli-LoRA-SGD}) offers a more practical and widely applicable alternative by utilizing stochastic gradients. This variant addresses the computational burden of full gradient methods and is a cornerstone for larger-scale learning tasks, providing insights into the interplay of stochasticity and randomized matrix adaptation.
    \end{itemize}

    \item \textbf{Advanced Variance Reduction Techniques for Non-Convex Optimization:} Stochastic gradients, while efficient, introduce variance that can impede convergence. Integrating variance reduction (VR) into the \algname{LoRA} structure, particularly with the additional Bernoulli randomization, presents unique analytical challenges. Our work addresses this by developing specific VR-enhanced variants for \algname{Bernoulli-LoRA}. To the best of our knowledge, we provide the first theoretical analyses demonstrating provable benefits for LoRA-type methods incorporating advanced VR schemes in $L$-smooth non-convex settings. Specifically, we propose:
    \begin{itemize}[leftmargin=*, itemsep=0pt, topsep=3pt, parsep=3pt]
        \item \algname{Bernoulli-LoRA-PAGE} (Algorithm~\ref{alg:Bernoulli-LoRA-PAGE}): Tailored for the finite-sum setting~\eqref{eq:problem-finite-sum}, this method integrates the \algname{ProbAbilistic Gradient Estimator (PAGE)}~\citep{Li2020PAGE}. \algname{PAGE} is recognized for achieving optimal non-convex convergence rates and implementation simplicity, and we successfully adapt it to the \algname{Bernoulli-LoRA} context.
        \item \algname{Bernoulli-LoRA-MVR} (Algorithm~\ref{alg:Bernoulli-LoRA-MVR}): For the infinite-sum (expectation) setting, this variant employs \algname{Momentum Variance Reduction} techniques inspired by \algname{STORM}~\citep{STORM}. \algname{MVR} offers an efficient batch-free approach to VR, and our work demonstrates its compatibility and effectiveness within the \algname{Bernoulli-LoRA} paradigm.
    \end{itemize}

    \item \textbf{Communication-Efficient Federated Learning Extensions:} The application of PEFT methods like \algname{LoRA} in Federated Learning (FL) is promising but requires careful consideration of communication overhead and data heterogeneity. We extend \algname{Bernoulli-LoRA} to FL by designing three specialized algorithms that combine our randomized adaptation with established FL communication-saving techniques. To the best of our knowledge, this constitutes the first comprehensive theoretical analysis of LoRA-type methods integrated with established communication-efficient FL techniques such as quantization, gradient difference compression, and error feedback. Our FL extensions include:
    \begin{itemize}[leftmargin=*, itemsep=0pt, topsep=3pt, parsep=3pt]
        \item \algname{Fed-Bernoulli-LoRA-QGD} (Algorithm~\ref{alg:Fed-Bernoulli-LoRA-QGD}): This method tackles high communication bandwidth by incorporating \algname{QSGD}-style quantization~\citep{AlistrahQGD, Terngrad, Horvth2019NaturalCF, panferov2024correlatedquantizationfasternonconvex}, enabling clients to transmit compressed gradient information, a crucial feature for practical FL deployments.
        \item \algname{Fed-Bernoulli-LoRA-MARINA} (Algorithm~\ref{alg:Fed-Bernoulli-LoRA-MARINA}): We adapt the \algname{MARINA} communication compression strategy~\citep{MARINA}, which efficiently compresses gradient differences, to the \algname{Bernoulli-LoRA} framework. This is particularly beneficial for non-convex distributed learning over potentially heterogeneous datasets.
        \item \algname{Fed-Bernoulli-LoRA-EF21} (Algorithm~\ref{alg:Fed-Bernoulli-LoRA-EF21}): This algorithm integrates the modern \algname{EF21} error feedback mechanism~\citep{richtarik2021ef21}. Error feedback is vital for stabilizing training with contractive compressors, and we show how \algname{Bernoulli-LoRA} can leverage this for robust distributed fine-tuning.
    \end{itemize}
    
    \item \textbf{Analysis for Non-Smooth Convex Functions:} Recognizing that not all machine learning objectives are smooth, we broaden the applicability of our framework. To the best of our knowledge, we present the first theoretical analysis of LoRA-type methods specifically for the important class of non-smooth convex optimization problems. For this setting, we provide versions of \algname{Bernoulli-LoRA-GD} (Algorithm~\ref{alg:Bernoulli-LoRA_nonsmooth}) and establish their convergence rates for both constant stepsize policies and adaptive Polyak-type stepsizes, showcasing the versatility of the \algname{Bernoulli-LoRA} approach beyond smooth, non-convex settings.
\end{itemize}

\begin{table*}[t]
    \centering
    \small
    \caption{\small Summary of the convergence rates for the proposed methods, presented for smooth non-convex functions (``NC'') and for functions satisfying the P{\L}-condition (``P\L''). Absolute constant factors are omitted.
    Notation: $\Delta^0 \eqdef f(W^0) - f^{*}$; $\cG^0 \eqdef \sqfnorm{G^{0} - \nabla f(W^{0})}$; $\hat{\cG}^0 \eqdef  \frac{1}{M}\sum^M_{l=1}\sqfnorm{G^{0}_l - \nabla f_l(W^{0})}$; $T$ is the chain length; $\omega$ is the compression parameter; $\Delta^{*} \eqdef f^{*} - \frac{1}{M}\sum_{l=1}^M f_{l}^{*}$; $C_1$ is a constant from Asm.~\ref{as:ABC_assumption}; $q$ is the probability of a full gradient computation; $\beta$ is the contractive compression parameter; $b$ is the momentum parameter; ${\color{blue}\lambda_{\min}} =\lambda_{\min}^{p} := p\lambda_{\min}^{H_B} + (1-p)\lambda_{\min}^{H_A}$, and ${\color{red}\lambda_{\max}}= \lambda_{\max}^{p} := p\lambda_{\max}^{H_B} + (1-p)\lambda_{\max}^{H_A} $.}
    \label{tab:comparison}    
    \begin{threeparttable}
        \begin{tabular}{|l|l|c|c|}
            \hline
            \textbf{Setting} & \textbf{Method} & \textbf{NC convergence rate} & \textbf{P\L{} convergence rate}\\ 
            \hline\hline
            \eqref{eq:problem-statement} & \makecell[l]{\algname{Bernoulli-LoRA-GD} \\ (Alg. \ref{alg:Bernoulli-LoRA-GD-smooth})} & $\fr{\Delta^0}{\gamma \textcolor{blue}{\lambda_{\min}} T}$ & $\left(1 - {\gamma\mu\textcolor{blue}{\lambda_{\min}}}\right)^T\Delta^0$ \\
            \hline
            \eqref{eq:problem-statement} & \makecell[l]{\algname{Bernoulli-LoRA-SGD} \\ (Alg. \ref{alg:Bernoulli-LoRA-SGD})} & $\fr{\Delta^0}{\gamma \textcolor{blue}{\lambda_{\min}} T} + \fr{\gamma LC_1\textcolor{red}{\lambda_{\max}}}{\textcolor{blue}{\lambda_{\min}}}$ & $\left(1 - {\gamma\mu\textcolor{blue}{\lambda_{\min}}}\right)^T\Delta^0 + \fr{\gamma LC_1\textcolor{red}{\lambda_{\max}}}{\mu\textcolor{blue}{\lambda_{\min}}}$ \\
            \hline
            \eqref{eq:problem-statement}+\eqref{eq:problem-stoch-setting} & \makecell[l]{\algname{Bernoulli-LoRA-MVR} \\ (Alg. \ref{alg:Bernoulli-LoRA-MVR})} & $\fr{\Phi_1}{\gamma \textcolor{blue}{\lambda_{\min}} T} + \fr{b\sigma^2\textcolor{red}{\lambda_{\max}}}{(2-b)\textcolor{blue}{\lambda_{\min}}}$ \tnote{{\color{blue}(1)}} & $(1-\gamma\mu\textcolor{blue}{\lambda_{\min}})^T\Phi_1 + \fr{b\sigma^2\textcolor{red}{\lambda_{\max}}}{(2-b)\mu\textcolor{blue}{\lambda_{\min}}}$\tnote{{\color{blue}(1)}} \\
            \hline
            \eqref{eq:problem-statement}+\eqref{eq:problem-finite-sum} & \makecell[l]{\algname{Bernoulli-LoRA-PAGE} \\ (Alg. \ref{alg:Bernoulli-LoRA-PAGE})} & $\fr{\Phi_2}{\gamma \textcolor{blue}{\lambda_{\min}} T} $ \tnote{{\color{blue}(2)}} & $(1-\gamma\mu\textcolor{blue}{\lambda_{\min}})^T\Phi_2$\tnote{{\color{blue}(2)}} \\       
            \hline\hline
            \eqref{eq:problem-statement}+\eqref{eq:problem-distributed-setting} & \makecell[l]{\algname{Fed-Bernoulli-LoRA-QGD} \\ (Alg. \ref{alg:Fed-Bernoulli-LoRA-QGD})} & $\fr{\Delta^0}{\gamma \textcolor{blue}{\lambda_{\min}} T} + \fr{\gamma L\omega\Delta^{*}\textcolor{red}{\lambda_{\max}}}{M\textcolor{blue}{\lambda_{\min}}}$ &  $\left(1 - {\gamma\mu\textcolor{blue}{\lambda_{\min}}}\right)^T\Delta^0 + \fr{\gamma L^2\omega\textcolor{red}{\lambda_{\max}}}{M\mu\textcolor{blue}{\lambda_{\min}}}$ \\
            \hline
            \eqref{eq:problem-statement}+\eqref{eq:problem-distributed-setting} & \makecell[l]{\algname{Fed-Bernoulli-LoRA-MARINA} \\ (Alg. \ref{alg:Fed-Bernoulli-LoRA-MARINA})} & $\fr{\Phi_2}{\gamma \textcolor{blue}{\lambda_{\min}} T}$ \tnote{{\color{blue}(2)}} & $(1-\gamma\mu\textcolor{blue}{\lambda_{\min}})^T\Phi_2$\tnote{{\color{blue}(2)}} \\
            \hline
            \eqref{eq:problem-statement}+\eqref{eq:problem-distributed-setting} & \makecell[l]{\algname{Fed-Bernoulli-LoRA-EF21} \\ (Alg. \ref{alg:Fed-Bernoulli-LoRA-EF21})} & $\fr{\Phi_3}{\gamma \textcolor{blue}{\lambda_{\min}} T}$ \tnote{{\color{blue}(3)}} & $(1-\gamma\mu\textcolor{blue}{\lambda_{\min}})^T\Phi_3$\tnote{{\color{blue}(3)}} \\
            \hline
        \end{tabular}
        \begin{tablenotes}
            \scriptsize
            \item [{\color{blue}(1)}] $\Phi_1 \eqdef \Delta^0 + \frac{\gamma}{b(2-b)}\cG^0$;      
            \item [{\color{blue}(2)}] $\Phi_2 \eqdef \Delta^0 + \frac{\gamma}{q}\cG^0$;
            \item [{\color{blue}(3)}] $\Phi_3 \eqdef \Delta^0 + \frac{\gamma}{1 - \sqrt{1-\beta}}\hat{\cG}^0$.
        \end{tablenotes}
    \end{threeparttable}
\end{table*}
\section{Notation}
For matrices $W \in \mathbb{R}^{m \times n}$, where $m$ and $n$ denote the input and output dimensions respectively, we employ the Frobenius norm $\fnorm{\cdot}$, defined as $\fnorm{W} = \sqrt{\tr(W^\top W)}$, where $\tr(\cdot)$ denotes the matrix trace. The inner product between two matrices $A$ and $B$ is denoted by $\langle A, B \rangle = \tr(A^\top B)$. In our low-rank adaptation framework, $B \in \mathbb{R}^{m \times r}$ and $A \in \mathbb{R}^{r \times n}$ represent the factors of rank $r \ll \min\{m,n\}$. We use $\mathcal{O}(\cdot)$ to hide absolute constants. We denote
$\Delta^0 \eqdef f(W^0) - f^{*},$ $\cG^0 \eqdef \sqfnorm{G^{0} - \nabla f(W^{0})}$ and $\hat{\cG}^0 \eqdef  \frac{1}{M}\sum^M_{l=1}\sqfnorm{G^{0}_l - \nabla f_l(W^{0})}$.  
For differentiable functions $f$, the gradient $\nabla f(W) \in \mathbb{R}^{m \times n}$ is computed with respect to the trace inner product, while for non-smooth functions, the subgradient $\partial f(W) \in \mathbb{R}^{m \times n}$ is similarly defined. The superscript $\dagger$ denotes the Moore-Penrose pseudoinverse.

\section{Bernoulli-LoRA Framework}
In this section, we introduce the \algname{Bernoulli-LoRA} framework, a novel and generic approach for low-rank adaptation. The core idea is to perform a sequence of low-rank updates, where at each step, a probabilistic choice determines which of the two factor matrices ($A$ or $B$) is trained. This randomized mechanism, formalized in Algorithm \ref{alg:Bernoulli-LoRA}, not only provides a flexible and unifying theoretical construct for existing LoRA-style methods but also allows for a rigorous convergence analysis.

At each iteration, one of the two low-rank matrices is sampled from a fixed distribution and remains frozen, while the other is trained to minimize the objective. This strategy prevents optimization from being confined to a fixed subspace, reducing the risk of converging to a suboptimal point. We formalize these two configurations as Left and Right sketch updates.

\begin{definition}[Left Sketch]\label{def:left-sketch}
    The left sketch update rule is given by
    \begin{equation}
    \squeeze
        \Delta W = \frac{\alpha}{r}B_S\hat{A},
    \end{equation}
    where $B_S\sim\mathcal{D}_B$ is sampled from a fixed distribution over $\R^{m\times r}$ matrices, and only the matrix $\hat{A} \in \R^{r \times n}$ is adjustable.
\end{definition}

\begin{definition}[Right Sketch]\label{def:right-sketch}
    The right sketch update rule is given by
    \begin{equation}
    \squeeze
        \Delta W = \frac{\alpha}{r}\hat{B}A_S,
    \end{equation}
    where $A_S\sim\mathcal{D}_A$ is sampled from a fixed distribution over $\R^{r\times n}$ matrices, and only the matrix $\hat{B} \in \R^{m \times r}$ is adjustable.
\end{definition}

\begin{algorithm}[H]
\caption{\algname{Bernoulli-LoRA} Framework}\label{alg:Bernoulli-LoRA}
\begin{algorithmic}[1]
\STATE \textbf{Parameters:} pre-trained model $W^0 \in \mathbb{R}^{m \times n}$, rank $r \ll \min\{m,n\}$, scaling factor $\alpha > 0$, chain length $T$, sketch distributions $\mathcal{D}_S^B$ and $\mathcal{D}_S^A$, Bernoulli probability $p$.

\FOR{$t = 0, 1, \ldots, T-1$}
    \STATE Sample $c^t \sim \text{Be}(p)$ \hfill{\footnotesize {Bernoulli random variable}}
    \IF{$c^t = 1$}
        \STATE Sample $B_S^t \sim \mathcal{D}_S^B$ \hfill{\footnotesize (Left sketch)}
        \STATE Using a chosen optimizer, approximately solve $\hat{A}^t \approx \argmin_A f(W^t + \frac{\alpha}{r}B_S^tA)$.
        \STATE $W^{t+1} = W^t + \frac{\alpha}{r}B_S^t\hat{A}^t$.
    \ELSE
        \STATE Sample $A_S^t \sim \mathcal{D}_S^A$ \hfill{\footnotesize(Right sketch)}
        \STATE Using a chosen optimizer, approximately solve $\hat{B}^t \approx \argmin_B f(W^t + \frac{\alpha}{r}BA_S^t)$.
        \STATE $W^{t+1} = W^t + \frac{\alpha}{r}\hat{B}^tA_S^t$.
    \ENDIF
\ENDFOR
\end{algorithmic}
\end{algorithm}

\subsection{Reformulation as a Projected Gradient Step}\label{sec:ref_proj_step}
Building upon the work of \citet{malinovsky2024randomizedasymmetricchainlora} on their \algname{RAC-LoRA} framework, the update steps in Algorithm \ref{alg:Bernoulli-LoRA} can be reformulated as projected gradient steps. The subproblems in lines 6 and 10 are typically solved approximately, for instance, by taking a single step of a suitable optimizer like Gradient Descent (\algname{GD}) or its variants.

Following the approach of \citet{malinovsky2024randomizedasymmetricchainlora}, let's consider the update for the trainable matrix $\hat{A}^t$ in the Left Sketch case. Taking a single \algname{GD} step on the subproblem corresponds to minimizing a quadratic approximation of the objective. This yields the solution for $\hat{A}^t$:
\begin{align*}
    \hat{A}^t = -\eta \rb{\rbtop{B_S^t}B_S^t}^{\dagger}\rbtop{B_S^t} \nabla f(W^t),
\end{align*}
where $\eta$ is a learning rate for the subproblem and $\dagger$ denotes the Moore-Penrose pseudoinverse. Substituting this into the update for $W^{t+1}$ gives:
\begin{align*}
     W^{t+1} &= W^t + \frac{\alpha}{r}B_S^t\hat{A}^t = W^t - \frac{\alpha\eta}{r} B_S^t\rb{\rbtop{B_S^t}B_S^t}^{\dagger}\rbtop{B_S^t}\nabla f(W^t) \\
     &= W^t - \gamma H_B^t \nabla f(W^t),
\end{align*}
where we define the effective stepsize $\gamma \eqdef \frac{\alpha\eta}{r}$ and the projection matrix $H_B^t \eqdef B_S^t\rb{\rbtop{B_S^t}B_S^t}^{\dagger}\rbtop{B_S^t}$. A similar derivation for the Right Sketch case gives the update:
\begin{align*}
    W^{t+1} = W^t - \gamma \nabla f(W^t) H_A^t,
\end{align*}
where $H_A^t \eqdef \rbtop{A_S^t}\rb{A_S^t\rbtop{A_S^t}}^{\dagger}A_S^t$. This reformulation reveals that both Left and Right sketch updates are equivalent to applying a standard gradient-based update, but projected onto a randomly chosen low-rank subspace.

While \algname{RAC-LoRA} employs a deterministic choice for which matrix to update, our \algname{Bernoulli-LoRA} framework generalizes this concept by introducing a probabilistic selection at each step. This allows us to express the update for any of our proposed methods in a single, unified form:
\begin{align}
\label{eq:general_update_form}
    W^{t+1} = W^t - \gamma \hat{G}^t,
\end{align}
where $\hat{G}^t$ is the \textit{projected gradient estimator}. It is formed by taking a \textit{base gradient estimator} $G^t$ (e.g., a full gradient, a stochastic gradient, or a variance-reduced one) and projecting it based on the outcome of a Bernoulli trial:
\begin{align}
\label{eq:projected_estimator_form}
\hat{G}^t = \begin{cases}
H_B^t G^t, & \text{with probability } p \\
G^t H_A^t, & \text{with probability } 1-p
\end{cases}.
\end{align}
The specific choice of the base estimator $G^t$ defines the particular algorithm within the \algname{Bernoulli-LoRA} family. We summarize our proposed methods in Table \ref{tab:methods} and describe them next.

\begin{table*}[t]
    \centering
    \footnotesize
    \caption{\small Description of the methods developed and analyzed in this paper. All methods follow the general update rule $W^{t+1} = W^t - \gamma \hat{G}^t$, where the projected estimator $\hat{G}^t$ is defined in \eqref{eq:projected_estimator_form}. The table specifies the definition of the base gradient estimator $G^t$ for each method. The projection matrices are $H_A^t \eqdef \rbtop{A_S^t}\rb{A_S^t\rbtop{A_S^t}}^{\dagger}A_S^t$ and $H_B^t \eqdef B_S^t\rb{\rbtop{B_S^t}B_S^t}^{\dagger}\rbtop{B_S^t}$.}
    \label{tab:methods}
    \begin{threeparttable}
        \begin{tabular}{|l|l|l|c|}
            \hline
            \textbf{Setting} & \textbf{Method} & \textbf{Base Gradient Estimator} $G^t$ & \textbf{Thms. \#}\\
            \hline\hline
            \eqref{eq:problem-statement} & \makecell[l]{\algname{Bernoulli-LoRA-GD} \\ (Algs. \ref{alg:Bernoulli-LoRA-GD-smooth} \& \ref{alg:Bernoulli-LoRA_nonsmooth})} & $G^t = \nabla f(W^t)$ & \ref{thm:smooth_non_cvx} \& \ref{thm:smooth_pl} \& \ref{thm:nonsmooth_convergence}\\
            \hline
            \eqref{eq:problem-statement} & \makecell[l]{\algname{Bernoulli-LoRA-SGD} \\ (Alg. \ref{alg:Bernoulli-LoRA-SGD})} & $G^t = g(W^t)$ & \ref{th:B_LoRA_SGD} \& \ref{th:B_LORA_SGD_PL}\\
            \hline
            \eqref{eq:problem-statement}+\eqref{eq:problem-stoch-setting} & \makecell[l]{\algname{Bernoulli-LoRA-MVR} \\ (Alg. \ref{alg:Bernoulli-LoRA-MVR})} & $G^{t} = \nabla f_{\xi^{t}}(W^{t}) + (1-b)(G^{t-1} - \nabla f_{\xi^{t}}(W^{t-1}))$ & \ref{th:B_LoRA_MVR} \& \ref{th:B-LORA-MVR-PL}\\
            \hline
            \eqref{eq:problem-statement}+\eqref{eq:problem-finite-sum} & \makecell[l]{\algname{Bernoulli-LoRA-PAGE} \\ (Alg. \ref{alg:Bernoulli-LoRA-PAGE})} & $G^{t} = \begin{cases} \nabla f(W^{t}), & \text{w.p. } q \\ G^{t-1} + \nabla f_{i_{t}}(W^{t}) - \nabla f_{i_{t}}(W^{t-1}), & \text{w.p. } 1-q \end{cases}$ & \ref{th:B_LoRA_PAGE} \& \ref{th:B-LORA-PAGE-PL}\\
            \hline
            \eqref{eq:problem-statement}+\eqref{eq:problem-distributed-setting} & \makecell[l]{\algname{Fed-Bernoulli-LoRA-QGD} \\ (Alg. \ref{alg:Fed-Bernoulli-LoRA-QGD})} & $G^t = \frac{1}{M}\sum_{l=1}^M \cQ_l^t(\nabla f_l(W^t))$ & \ref{th:B_LoRA_QGD} \& \ref{th:B-LORA-QGD-PL}\\
            \hline
            \eqref{eq:problem-statement}+\eqref{eq:problem-distributed-setting} & \makecell[l]{\algname{Fed-Bernoulli-LoRA-MARINA} \\ (Alg. \ref{alg:Fed-Bernoulli-LoRA-MARINA})} & \makecell[l]{$\forall l: G_l^{t} = \begin{cases} \nabla f_l(W^{t}), & \text{w.p. } q \\ G_l^{t-1} + \cQ_l^t(\nabla f_l(W^{t}) - \nabla f_l(W^{t-1})), & \text{w.p. } 1-q \end{cases}$ \\ $G^{t} = \frac{1}{M}\sum_{l=1}^M G_l^{t}$} & \ref{th:B_LoRA_MARINA} \& \ref{th:B-LORA-MARINA-PL}\\
            \hline
            \eqref{eq:problem-statement}+\eqref{eq:problem-distributed-setting} & \makecell[l]{\algname{Fed-Bernoulli-LoRA-EF21} \\ (Alg. \ref{alg:Fed-Bernoulli-LoRA-EF21})} & \makecell[l]{$\forall l: G_l^{t} = G_l^{t-1} + \cC_l^t(\nabla f_l(W^{t}) - G_l^{t-1})$ \\ $G^{t} = \frac{1}{M}\sum_{l=1}^M G_l^{t}$} & \ref{th:B_LoRA_EF21} \& \ref{th:B-LORA-EF21-PL}\\
            \hline
        \end{tabular}
    \end{threeparttable}
\end{table*}

\subsection{Core Algorithmic Variants}

\paragraph{Bernoulli-LoRA-GD.}
The simplest instantiation of our framework is \algname{Bernoulli-LoRA-GD} (Algorithm~\ref{alg:Bernoulli-LoRA-GD-smooth}). This method serves as a foundational building block and a starting point for more elaborate variants. It uses the full gradient of the objective function as its base estimator, i.e., $G^t = \nabla f(W^t)$. While impractical for large-scale deep learning, its analysis provides crucial insights into the convergence behavior of the Bernoulli-LoRA mechanism under idealized, deterministic conditions.

\paragraph{Bernoulli-LoRA-SGD.}
\,
\algname{Stochastic Gradient Descent (SGD)}~\citep{robbins1951stochastic} is a highly effective and widely utilized algorithm for training a variety of machine learning models. The latest advancements in deep learning training methods are all based on different variations of \algname{SGD}~\citep{Sun2020OptimizationFD}. Its advantage over \algname{GD} is that it uses stochastic gradients for updates, rather than relying on full gradients. Within our framework, we develop \algname{Bernoulli-LoRA-SGD}, where the base estimator $G^t$ is a general unbiased stochastic gradient of $f$ at $W^t$.

\paragraph{Bernoulli-LoRA-PAGE.}
Several optimal algorithms exist for addressing non-convex optimization problems, such as \algname{SPIDER}\citep{Spider} and \algname{SARAH}\citep{Sarah}. However, their optimality is supported by a known lower bound that applies only in the small data setting. In contrast, \algname{ProbAbilistic Gradient Estimator (PAGE)}~\citep{Li2020PAGE} stands out for its simplicity, ease of implementation, and ability to achieve optimal convergence in non-convex optimization. \algname{PAGE} alternates between a full gradient update with probability $q_t$ and a low-cost gradient adjustment with probability $1 - q_t$. \algname{Bernoulli-LoRA-PAGE} is a new method based on \algname{PAGE} within our \algname{Bernoulli-LoRA} framework.

\paragraph{Bernoulli-LoRA-MVR.}
VR methods outperform \algname{SGD} in reaching first-order critical points but often require finely tuned learning rates and large batch sizes to be effective. To overcome these challenges, \algname{Momentum Variance Reduction (MVR)}~\citep{STORM} was introduced for server-only stochastic non-convex optimization. \algname{MVR} uses a modified momentum technique to reduce variance without relying on large batch sizes. Several works employ this powerful approach~\citep{tyurin2023dasha,karagulyan2024spamstochasticproximalpoint}. We propose \algname{Bernoulli-LoRA-MVR}, where the base estimator $G^t$ is updated using the MVR rule: a combination of the current stochastic gradient and a momentum term that incorporates the difference between past estimators and gradients.

\subsection{Extensions for Federated Learning}
\citet{sun2024improvingloraprivacypreservingfederated} identified instability in \algname{LoRA}, arising from the mismatch between local clients simultaneously optimizing two low-rank matrices and the central server aggregating them independently. Factors such as data heterogeneity, multi-step local updates, and the amplification of additive noise applied to gradients for ensuring differential privacy (DP) significantly impact the process. Additionally, the final performance is highly sensitive to hyperparameter choices. Their proposed solution centers on keeping the randomly initialized non-zero matrices fixed while exclusively fine-tuning the zero-initialized ones. Based on this asymmetric approach, \citet{malinovsky2024randomizedasymmetricchainlora} proposed a distributed method \algname{Fed-RAC-LoRA}.

We develop the theory further by incorporating compression, VR and EF techniques into FL methods for \algname{LoRA} within the novel \algname{Bernoulli-LoRA} framework.

The effectiveness of a distributed training method is primarily measured by its communication complexity, defined as the product of the required communication rounds and the communication volume per round. Following common practice, we assume client-to-server communication is the main bottleneck and exclude server-to-client communication from our analysis.

\paragraph{Fed-Bernoulli-LoRA-QGD.}
A key challenge for distributed methods lies in the high communication cost of gradient updates. Lossy compression techniques, such as \algname{QSGD}~\citep{AlistrahQGD}, address this by enabling clients to send quantized gradients. 
We design \algname{Fed-Bernoulli-LoRA-QGD} based on \algname{QSGD}. The clients send compressed versions of their gradients.
The base estimator $G^t$ is formed by averaging the compressed local gradients received from all clients.

\paragraph{{Fed-Bernoulli-LoRA-MARINA}.}
\,
\algname{MARINA}~\citep{MARINA} is a communication-efficient method for non-convex distributed learning on heterogeneous datasets that uses a novel gradient difference compression strategy. 
Its biased gradient estimator underpins its strong theoretical and practical performance, with proven communication complexity bounds surpassing all prior first-order methods.
We propose \algname{Fed-Bernoulli-LoRA-MARINA}, where each client's local estimator $G_l^t$ is updated either with a full local gradient (with probability $q$) or by adding a compressed gradient difference to its previous estimator. The server's base estimator $G^t$ is the average of these local estimators.

\paragraph{{Fed-Bernoulli-LoRA-EF21}.}
Error Feedback (EF)~\citep{Seide2014, stich2018sparsifiedsgdmemory, alistarh2018convergence, richtarik2021ef21} is a widely adopted technique for stabilizing training with contractive compressors. We propose \algname{Fed-Bernoulli-LoRA-EF21}, based on the modern \algname{EF21}. Here, each client updates its local estimator $G_l^t$ by adding a compressed version of the difference between the current local gradient and the previous local estimator. The server's base estimator $G^t$ is again the average of the clients' estimators.

\section{Convergence Results}
The convergence properties of our framework hinge on the spectral properties of the expected projection matrix, which is introduced in Section \ref{sec:ref_proj_step}. The magnitude of its eigenvalues, particularly the smallest (and in some cases, the largest), is a crucial factor that governs the optimization dynamics.

\begin{assumption}\label{as:projection_matrix}
(Positive Expected Projection) Consider a projection matrix $H$ generated through either Left Sketch (Definition \ref{def:left-sketch}) or Right Sketch (Definition \ref{def:right-sketch}). For the sampling distributions $\mathcal{D}_S^B$ and $\mathcal{D}_S^A$, the smallest eigenvalue of the expected projection matrix is strictly positive:
\begin{align*}
\squeeze
\lambda_{\min}^H = \lambda_{\min}[\mathbb{E}[H]] > 0.
\end{align*}
\end{assumption}

\begin{assumption}\label{as:bounded_below}
(Lower Bounded Function) The objective function $f$ has a finite infimum $f^* \in \mathbb{R}$.
\end{assumption}

\begin{remark}[On the Practicality of Assumption \ref{as:projection_matrix}]
Assumption \ref{as:projection_matrix} is a mild and standard requirement, as it is satisfied by common practical choices for the sampling distributions $\mathcal{D}_S^B$ and $\mathcal{D}_S^A$. For instance, a prevalent strategy \citep{xia2024chainloraefficientfinetuning, MaoSurvey} is to sample the entries of the fixed matrix from an i.i.d. Gaussian distribution. As shown in Appendix~\ref{sec:discussion_projection_matrix} (Lemma~\ref{le:exp_eig}), this choice leads to an expected projection matrix $\mathbb{E}[H] = \frac{r}{n}I_n$, where $r$ is the rank and $n$ is the relevant dimension. Consequently, $\lambda_{\min}^H = \frac{r}{n} > 0$, readily satisfying the assumption.
\end{remark}

Following classical optimization literature \citep{Nemirovski2009, beck2017first, duchi2018introductory, lan2020first, drusvyatskiy2020convex, nesterov2018lectures}, we characterize convergence guarantees for two distinct settings. In the non-smooth convex case, our objective is to find an $\eps$-suboptimal solution: a random matrix $\hat{W}\in \R^{m\times n}$ that satisfies
\begin{align}\label{eq:suboptimal_condition}
\squeeze
\Exp{f(\hat{W}) - f(W^*)} \leq \eps,
\end{align}
where $\Exp{\cdot}$ denotes the expectation with respect to the algorithm's randomness, and $W^*$ is any minimizer of $f$. This same measure of performance—function value suboptimality—is also used to characterize convergence under the Polyak-Łojasiewicz condition, which we introduce later. For the smooth non-convex setting, where finding global minima is generally intractable, we instead aim to locate an $\eps$-stationary point: a random matrix $\hat{W}\in \R^{m\times n}$ satisfying
\begin{align}\label{eq:stat_point}
\squeeze
\Exp{\sqfnorm{\nabla f(\hat{W})}} \leq \eps^2.
\end{align}
This condition guarantees that the expected squared norm of the gradient at our solution is sufficiently small, indicating proximity to a stationary point. To quantify the efficiency of our algorithms, we analyze their iteration complexity—the number of iterations required to achieve these criteria.

A fundamental assumption in the convergence analysis of gradient-based optimization is the Lipschitz continuity of the gradient \citep{bubeck2015convex, nesterov2018lectures, beck2017first, demidovich2023guide, khaled2022better}. This property, often referred to as Lipschitz smoothness, ensures the stability of the optimization trajectory and plays a crucial role in establishing convergence rates \citep{bottou2018optimization, Sun2020OptimizationFD}.

\begin{assumption}\label{as:lipschitz_smoothness}
(Lipschitz Smooth Gradient) A function $f$ is differentiable, and there exists a constant $L > 0$ such that
\begin{align*}
\squeeze
\fnorm{\nabla f(W) - \nabla f(V)} \leq L\fnorm{W - V},
\end{align*}
for all $W, V \in \mathbb{R}^{m\times n}.$
\end{assumption}

A significant challenge arises when applying \algname{LoRA} adaptation directly: the Lipschitz smoothness property is not preserved. Specifically, even if a function $f(W)$ satisfies Assumption \ref{as:lipschitz_smoothness}, its composition with the \algname{LoRA} parameterization, $f(W^0 + BA)$, generally fails to maintain Lipschitz smoothness with respect to the variables $\{B, A\}$. This breakdown complicates the analysis of standard gradient-based methods when applied directly to the LoRA parameterization, as formally demonstrated by \citet{sun2024improvingloraprivacypreservingfederated}. Our framework, by reformulating the updates as projected steps on the full parameter space, circumvents this issue.

To unify our analysis, we define a probability-weighted eigenvalue $\lambda_{\min(\max)}^{p} \eqdef p\lambda_{\min(\max)}^{H_B} + (1-p)\lambda_{\min(\max)}^{H_A}$. Let $\widetilde{W}^T$ be an iterate drawn randomly from the sequence $\{W^0, W^1, \ldots, W^{T-1}\}$, with the specific sampling distribution depending on the method.

We begin by presenting the convergence result for the foundational \algname{Bernoulli-LoRA-GD} method. The proof can be found in Appendix~\ref{sec:Bernoulli-LoRA}.

\begin{theorem}[Smooth Non-Convex Setting]\label{thm:smooth_non_cvx}
Let Assumptions \ref{as:projection_matrix}, \ref{as:bounded_below}, and \ref{as:lipschitz_smoothness} hold, and let the stepsize satisfy $0 < \gamma \leq \frac{1}{L}$. Then the iterates of \algname{Bernoulli-LoRA-GD} (Algorithm \ref{alg:Bernoulli-LoRA-GD-smooth}), with matrices $\hat{A}^t$ and $\hat{B}^t$ computed according to Lemma \ref{le:smooth_left_right_sketch}, satisfy
\begin{align*}
\squeeze
\mathbb{E}\left[\sqfnorm{\nabla f(\widetilde{W}^T)}\right] \leq \frac{2\Delta^0}{\gamma\lambda_{\min}^{p}  T},
\end{align*}
where $\Delta^0 \eqdef f(W^0) - f^*.$
\end{theorem}

While insightful, full-gradient methods are often impractical for large-scale problems. We therefore extend our analysis to the stochastic setting, where the gradient is replaced by an unbiased estimator $g(W)$. For this, we use the general {\em expected smoothness} assumption.

\begin{assumption}[Expected Smoothness~\citep{khaled2022better}] 
\label{as:ABC_assumption}
The stochastic gradient estimator $g(W)$ satisfies
    \begin{equation*}
    \squeeze
        \Exp{\sqfnorm{g(W)}}\leq 2A_1\rb{f(W) - f^{*}} + B_1\cdot\sqfnorm{\nabla f(W)} + C_1,
    \end{equation*}
    for some constants $A_1, B_1, C_1 \geq 0$ and all $W\in\R^{m\times n}.$
\end{assumption}

The following theorem establishes the convergence for \algname{Bernoulli-LoRA-SGD}. Its proof is in Appendix~\ref{apx:sgd}.
\begin{theorem}
  \label{th:B_LoRA_SGD}  
  Let Assumptions~\ref{as:bounded_below},~\ref{as:lipschitz_smoothness}, and ~\ref{as:ABC_assumption} hold, and let the stepsize satisfy $$0 < \gamma \leq  \min\left\{\frac{1}{\sqrt{L A_1 \lambda^p_{\max} T}}, \frac{1}{LB_1}\left(\frac{\lambda^p_{\max}}{\lambda^p_{\min}}\right)^{-1}\right\}.$$ Then the iterates generated by \algname{Bernoulli-LoRA-SGD} (Algorithm~\ref{alg:Bernoulli-LoRA-SGD}) satisfy
  \begin{equation*}
  \squeeze
       \Exp{\sqfnorm{\nabla f(\widetilde{W}^T)}} \leq \frac{6\Delta^0}{\gamma \lambda^p_{\min} T} + \gamma LC_1 \cdot \frac{\lambda^p_{\max}}{\lambda^p_{\min}},
  \end{equation*}
  where $\Delta^0 \eqdef f(W^0) - f^*$.
\end{theorem}

To analyze our variance-reduced methods, we also consider the more specific bounded variance assumption.
\begin{assumption}[Bounded Variance~\citep{Nemirovski2009}]
\label{as:bounded_variance}
    There exists a constant $\sigma > 0$ such that, for all $~W \in \R^{m\times n},$
    \begin{align*}
        \squeeze
        \Exp{\nabla f_{\xi}(W)} &= \nabla f(W),\\ 
    \Exp{\sqfnorm{\nabla f_{\xi}(W)  - \nabla f(W) }} &\leq \sigma^2.
    \end{align*}
\end{assumption}

The next result establishes convergence for \algname{Bernoulli-LoRA-MVR}. The proof is in Appendix~\ref{apx:B-LoRA-MVR}.
\begin{theorem}\label{th:B_LoRA_MVR}
    Let Assumptions~\ref{as:projection_matrix},~\ref{as:bounded_below}, ~\ref{as:lipschitz_smoothness}, and ~\ref{as:bounded_variance}  hold, and let the stepsize satisfy $
        0<\gamma \leq \frac{1}{L\left(1+\sqrt{\frac{2 \lambda^p_{\max}(1-b)^2}{b}}\right)}$. Then the iterates of \algname{Bernoulli-LoRA-MVR} (Algorithm~\ref{alg:Bernoulli-LoRA-MVR}) satisfy
        \begin{equation*}
        \squeeze
            \Exp{\sqfnorm{\nabla f(\widetilde{W}^T)}} \leq \frac{2\Delta^0}{\gamma\lambda^p_{\min}  T} + \left(\frac{\cG^0}{b T} +  \frac{2b\sigma^2}{2-b}\right)\cdot\frac{\lambda^p_{\max}}{\lambda^p_{\min}},
        \end{equation*}
        where $\Delta^0 \eqdef f(W^0) - f^*$ and $\cG^0 \eqdef \sqfnorm{G^{0} - \nabla f(W^{0})} $.
\end{theorem}

For the finite-sum setting, we analyze \algname{Bernoulli-LoRA-PAGE}, with its convergence detailed in the following theorem and proven in Appendix~\ref{apx:B-LoRA-PAGE}.
\begin{theorem}\label{th:B_LoRA_PAGE}
    Let Assumptions~\ref{as:projection_matrix},~\ref{as:bounded_below}, and ~\ref{as:lipschitz_smoothness} hold, and let the stepsize satisfy $
        0<\gamma \leq \frac{1}{L\left(1+\sqrt{\frac{1-q}{q}\lambda^p_{\max}}\right)}$. Then the iterates of \algname{Bernoulli-LoRA-PAGE} (Algorithm~\ref{alg:Bernoulli-LoRA-PAGE}) satisfy
        \begin{equation*}\squeeze
            \Exp{\sqfnorm{\nabla f(\widetilde{W}^T)}} \leq \frac{2\Delta^0}{\gamma\lambda^p_{\min}  T} + \frac{\cG^0}{ qT} \cdot\frac{\lambda^p_{\max}}{\lambda^p_{\min}},
        \end{equation*}
        where $\Delta^0 \eqdef f(W^0) - f^*$ and $\cG^0 \eqdef \sqfnorm{G^{0} - \nabla f(W^{0})} $.
\end{theorem}

We now shift to our Federated Learning variants. The following theorem provides convergence guarantees for \algname{Fed-Bernoulli-LoRA-QGD}, with the proof available in Appendix~\ref{apx:B-LoRA-QGD}.
\begin{theorem}
  \label{th:B_LoRA_QGD}  
  Let Assumptions~\ref{as:projection_matrix},~\ref{as:bounded_below}, ~\ref{as:lipschitz_smoothness}, and ~\ref{as:function_dissimilairy} hold, and let the stepsize satisfy \\$0 < \gamma \leq  \min\left\{\frac{1}{L\sqrt{ \frac{\omega}{M} \lambda^p_{\max} T}}, \frac{1}{L}\left(\frac{\lambda^p_{\max}}{\lambda^p_{\min}}\right)^{-1}\right\}$. Then the iterates of \algname{Fed-Bernoulli-LoRA-QGD} (Algorithm~\ref{alg:Fed-Bernoulli-LoRA-QGD}) satisfy
  \begin{equation*}
  \squeeze
       \Exp{\sqfnorm{\nabla f(\widetilde{W}^T)}} \leq \frac{6\Delta^0}{\gamma \lambda^p_{\min} T} + \frac{2\gamma L \omega \Delta^* }{M}\cdot\frac{\lambda^p_{\max}}{\lambda^p_{\min}},
  \end{equation*} 
  where $\Delta^0 \eqdef f(W^0) - f^*$.
\end{theorem}

Next, we present the convergence result for \algname{Fed-Bernoulli-LoRA-MARINA}. The proof can be found in Appendix~\ref{apx:B-LoRA-MARINA}.
\begin{theorem}\label{th:B_LoRA_MARINA}
    Let Assumptions~\ref{as:projection_matrix}, ~\ref{as:bounded_below}, and ~\ref{as:lipschitz_smoothness} hold, and let the stepsize satisfy $
        0<\gamma \leq \frac{1}{L\left(1+\sqrt{\lambda^p_{\max}\frac{1-q}{q} \cdot \frac{\omega}{M}}\right)}$. Then the iterates of \algname{Fed-Bernoulli-LoRA-MARINA} (Algorithm~\ref{alg:Fed-Bernoulli-LoRA-MARINA}) satisfy
        \begin{equation*}\squeeze
            \Exp{\sqfnorm{\nabla f(\widetilde{W}^T)}} \leq \frac{2\Delta^0}{\gamma\lambda^p_{\min}  T} + \frac{\cG^0}{ q T} \cdot \frac{\lambda^p_{\max}}{\lambda^p_{\min}},
        \end{equation*}
        where $\Delta^0 \eqdef f(W^0) - f^*$ and $\cG^0 \eqdef \sqfnorm{G^{0} - \nabla f(W^{0})} $.
\end{theorem}

The convergence of \algname{Fed-Bernoulli-LoRA-EF21} is established below, with a detailed proof in Appendix~\ref{apx:B-LoRA-EF21}.
\begin{theorem}\label{th:B_LoRA_EF21}
    Let Assumptions~\ref{as:projection_matrix},~\ref{as:bounded_below}, and~\ref{as:lipschitz_smoothness} hold, and let the  stepsize satisfy $0<\gamma \leq \frac{1}{L\left(1+\frac{\sqrt{\lambda^p_{\max} (1-\beta)}}{1-\sqrt{1-\beta}} \right)}$.  Then the iterates of \algname{Fed-Bernoulli-LoRA-EF21} (Algorithm~\ref{alg:Fed-Bernoulli-LoRA-EF21}) satisfy
        \begin{equation*}\squeeze
            \Exp{\sqfnorm{\nabla f(\widetilde{W}^T)}} \leq \frac{2\Delta^0}{\gamma \lambda^p_{\min} T} +  \frac{2\hat{\cG}^0}{ \beta T}\cdot\frac{\lambda^p_{\max}}{\lambda^p_{\min}},
        \end{equation*}
        where $\Delta^0 \eqdef f(W^0) - f^*$ and $\hat{\cG}^0 \eqdef  \frac{1}{M}\sum^M_{l=1}\sqfnorm{G^{0}_l - \nabla f_l(W^{0})}$.
\end{theorem}

To obtain stronger, linear convergence rates, we introduce the Polyak–\L ojasiewicz condition, a common generalization of strong convexity.
\begin{assumption}[Polyak–\L ojasiewicz condition~\citep{Polyak1963GradientMF, lojasiewicz1963topological}] 
\label{as:pl_condition}
There exists $\mu>0$ such that
    \begin{equation*}\squeeze
        \frac{1}{2}\sqfnorm{\nabla f(W)}\geq\mu\rb{f(W)-f^{*}}.
    \end{equation*}
\end{assumption}

The next theorem states the convergence of \algname{Bernoulli-LoRA-SGD} under this condition. It is proven in Appendix~\ref{apx:sgd}.
\begin{theorem}
\label{th:B_LORA_SGD_PL}
Let Assumptions~\ref{as:bounded_below},~\ref{as:lipschitz_smoothness}, ~\ref{as:ABC_assumption}, and \ref{as:pl_condition} hold, and let the stepsize satisfy \\$0 < \gamma \leq  \min\left\{\frac{\mu\lambda^p_{\min}}{2L A_1\lambda^p_{\max}}, \frac{2}{\mu \lambda^p_{\min}}, \frac{1}{LB_1}\left(\frac{\lambda^p_{\max}}{\lambda^p_{\min}}\right)^{-1}\right\}$. Then the iterates of \algname{Bernoulli-LoRA-SGD} (Algorithm~\ref{alg:Bernoulli-LoRA-SGD}) satisfy
  \begin{equation*}
  \squeeze
      \Exp{f(W^T) - f^*} \leq \left(1- \frac{\gamma \mu \lambda^p_{\min}}{2}\right)^{T} \Delta^0 + \frac{\gamma LC_1}{\mu}\cdot\frac{\lambda^p_{\max}}{\lambda^p_{\min}},
  \end{equation*}    
  where $\Delta^0 \eqdef f(W^0) - f^*$.
\end{theorem}

All other P\L-condition results are relegated to the Appendix.
\section{Experiments}\label{sec:exps}

To validate our theoretical findings, we conducted numerical experiments across multiple machine learning tasks. 
\subsection{Linear Regression with Non-convex Regularization.}
We begin with a controlled linear regression problem with non-convex regularization, split into pre-training and fine-tuning phases. We use $\widetilde{(\cdot)}$ for pre-training quantities and $\hat{(\cdot)}$ for fine-tuning. During the \textbf{pre-training phase}, we solve
\begin{eqnarray}\label{eq:linreg_task_pre}
\squeeze
\min_{x\in \R^n}\!\left\{
\widetilde{f}(x) \,\eqdef\,
\frac{1}{2\widetilde{m}}\norm{\widetilde{D}x-\widetilde{b}}_2^2
+\widetilde{\lambda}\sum_{j=1}^d \frac{x_j^2}{1+x_j^2}
\right\},
\end{eqnarray}

where $\widetilde{D}\in \R^{\widetilde{m}\times n}$, $\widetilde{b}\in \R^{\widetilde{m}}$, $\widetilde{m}=9\times10^4$, and $n=4096$. We set $\widetilde{\lambda}=\norm{\widetilde{D}}_2 \approx18.2$, giving $\widetilde{L}\approx54.7$. We optimize until $\norm{\nabla f(\widetilde{x}^*)}^2\le10^{-8}$ to obtain $\widetilde{x}^*$. For the \textbf{fine-tuning phase}, we use $\widetilde{x}^*$ as the initialization and then solve
\begin{eqnarray}\label{eq:linreg_task_ft}
\squeeze
\min_{x\in \R^n}\!\left\{
\hat{f}(x) \,\eqdef\,
\frac{1}{2\hat{m}}\norm{\hat{D}x-\hat{b}}_2^2
+\hat{\lambda}\sum_{j=1}^d \frac{x_j^2}{1+x_j^2}
\right\},
\end{eqnarray}
where $\hat{D}\in \R^{\hat{m}\times n}$, $\hat{b}\in \R^{\hat{m}}$, and $\hat{m}=10^4$. We keep $n=4096$ and set $\hat{\lambda}=\norm{\hat{D}}_2\approx4101.7$, yielding $\hat{L}\approx12305.3$. This second phase uses a dataset with notably different characteristics to mirror realistic domain shifts.

\textbf{Stochastic setting.}
We consider the stochastic setting, comparing \algname{RAC-LoRA-SGD}, \algname{Bernoulli-LoRA-SGD}, and \algname{Bernoulli-LoRA-PAGE}. In all experiments, we use a batch size of $100$, which corresponds to $1\%$ of the data.
\begin{figure}[h]
    \centering
        \includegraphics[width=0.49\textwidth]{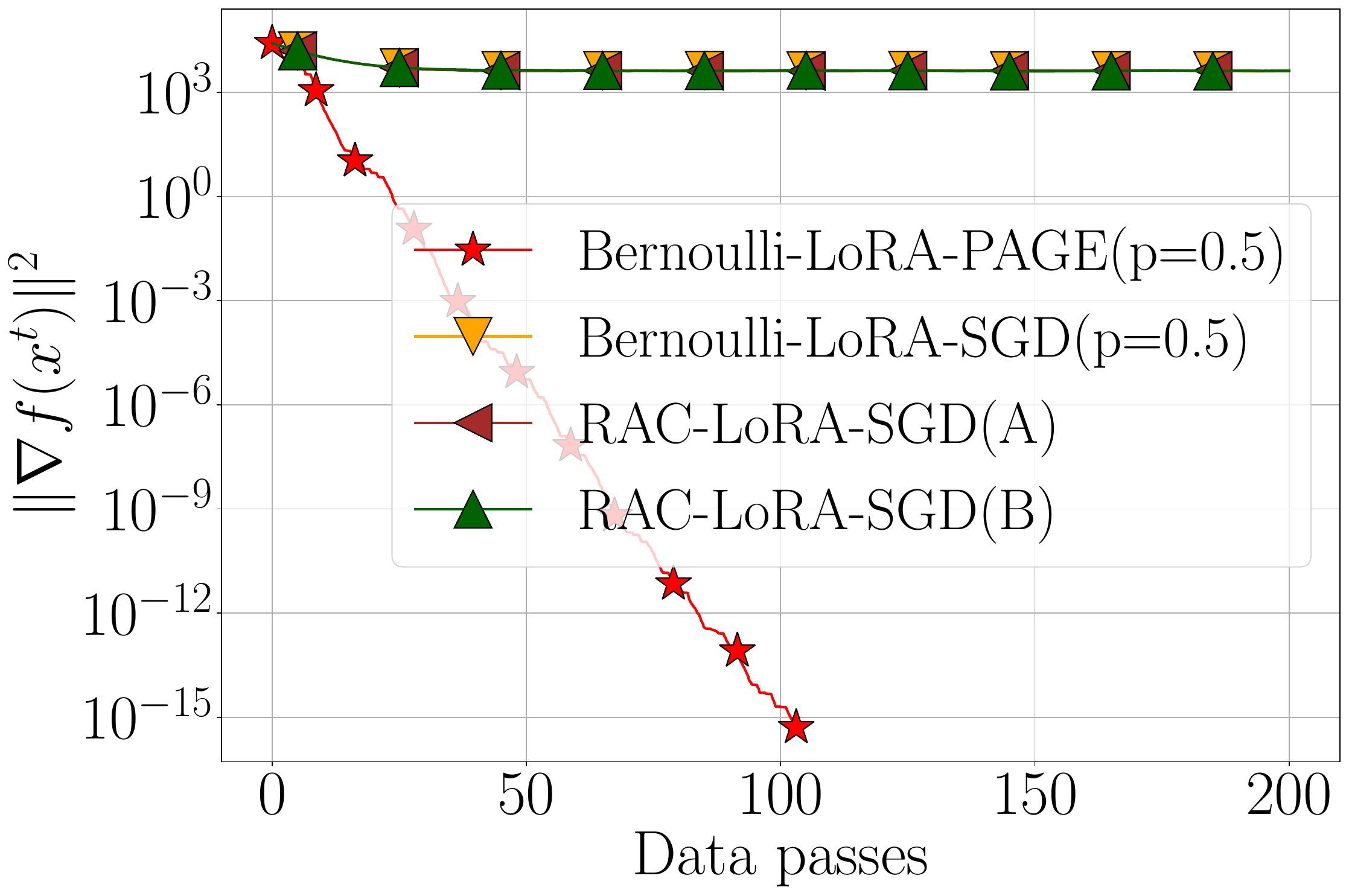}
    \caption{Comparison of \algname{RAC-LoRA-SGD}, \algname{Bernoulli-LoRA-SGD} and \algname{Bernoulli-LoRA-PAGE} on linear regression fine-tuning. Curves with $p=0.01,0.2,\dots$ indicate \algname{Bernoulli-LoRA} sampling parameters. \algname{RAC-LoRA-SGD(A)} trains $B$ after resampling $A$, while \algname{RAC-LoRA-SGD(B)} does the reverse. All methods use $\gamma = \nfr{c}{\hat{L}}$ with $c$ tuned individually.}
    \label{fig:bernoulli_lora_sgd}
\end{figure}

Figure~\ref{fig:bernoulli_lora_sgd} shows that \algname{Bernoulli-LoRA-PAGE} successfully reduces variance and converges to the target tolerance, whereas all \algname{SGD} variants stall at a certain accuracy. This underscores the practical advantage of \algname{Bernoulli-LoRA-PAGE} over the baseline \algname{RAC-LoRA-SGD} in the stochastic setting from an optimization standpoint.

\subsection{MLP on MNIST}
In this section, we evaluate \algname{Bernoulli-LoRA} against established baselines in parameter-efficient fine-tuning, following the setup of \citet{malinovsky2024randomizedasymmetricchainlora}. 
Source code of our experiments is available at \url{https://github.com/IgorSokoloff/Bernoulli-LoRA_experiments}.

\textbf{Methodology.}
We first pre-train a three-layer MLP on MNIST digits $0$--$4$ \citep{lecun1998gradient}, then adapt it with various \algname{LoRA}-type methods to classify digits $5$--$9$. Only unseen classes are used for evaluation. All adaptations use rank $r=1$ and train for $50$ epochs with \algname{AdamW} \citep{loshchilov2017decoupled} ($\beta_1=0.9$, $\beta_2=0.999$, $\epsilon=10^{-8}$), a fixed learning rate of $2\times 10^{-4}$, and batch size $128$. Each method is run $20$ times using different seeds, and Table~\ref{tab:mlp_comparison} reports the median accuracy (with standard deviation). For \algname{Bernoulli-LoRA}, we show the best median accuracy among all tested settings.
\begin{table}[ht]
\centering
\caption{Performance on MNIST classification using an MLP with rank $r$ and scaling $\alpha=1$. For \algname{AsymmLoRA} and \algname{RAC-LoRA}, only the zero-initialized matrix is trained.}
\begin{threeparttable}
\setlength{\tabcolsep}{1pt}
\begin{tabular}{lcccc}
\toprule
\textbf{Method} & $\mathcal{D}_A$ & $\mathcal{D}_B$ & \textbf{Acc. (test)} & \textbf{Train Params} \\
\midrule
\algname{FPFT}           & -         & -         & $99.5$           & 54,700       \\
\algname{LoRA}           & Gaussian  & Zero      & $85.69 \pm 1.60$  & 1K           \\
\algname{LoRA}           & Zero      & Gaussian  & $89.82 \pm 0.90$  & 1K           \\
\algname{COLA}           & Gaussian  & Zero      & $93.32 \pm 0.50$  & 1K           \\
\algname{COLA}           & Zero      & Gaussian  & $96.55 \pm 0.20$  & 1K           \\
\algname{AsymmLoRA}      & Gaussian  & Zero      & $64.04 \pm 6.90$  & 133          \\
\algname{AsymmLoRA}      & Zero      & Gaussian  & $74.52 \pm 7.20$  & 912          \\
\algname{RAC-LoRA}       & Gaussian  & Zero      & $93.02 \pm 0.50$  & 133          \\
\algname{RAC-LoRA}       & Zero      & Gaussian  & $96.49 \pm 0.20$  & 912          \\
\algname{Bernoulli-LoRA}\tnote{2} & Zero\tnote{1} & Gaussian & $96.46 \pm 0.17$ & $\approx{904}$    \\
\bottomrule
\end{tabular}
\begin{tablenotes}\footnotesize
\item[1] Although \algname{Bernoulli-LoRA} prescribes probabilistic selection from the first iteration, a deterministic assignment of fixed and trainable matrices at initialization yielded better performance.
\item[2] Achieved with $p=0.99$, giving an expected trainable parameter count $p\cdot 912 + (1-p)\cdot 133 \approx 904$. Here, $912$ and $133$ are the parameter counts for matrices $A$ and $B$, respectively.
\end{tablenotes}
\end{threeparttable}
\label{tab:mlp_comparison}
\end{table}

\textbf{Discussion.}
From Table~\ref{tab:mlp_comparison}, standard \algname{LoRA} attains roughly $86\%$ of full-parameter fine-tuning (FPFT) accuracy, indicating room for improvements via chaining. \algname{COLA} improves upon vanilla \algname{LoRA}, though both lack formal convergence guarantees. \algname{AsymmLoRA} approximates \algname{LoRA} in practice \citep{sun2024improvingloraprivacypreservingfederated} but similarly lacks convergence analysis, whereas \algname{RAC-LoRA} and \algname{Bernoulli-LoRA} both boost accuracy and have theoretical backing. Notably, \algname{Bernoulli-LoRA} matches \algname{RAC-LoRA} in generalization and also guarantees convergence. An additional benefit is that \algname{RAC-LoRA} and \algname{Bernoulli-LoRA} each train only one matrix per \algname{LoRA} block, whereas \algname{COLA} needs two. In \algname{RAC-LoRA}, either $A$ or $B$ is trained deterministically; in \algname{Bernoulli-LoRA}, the choice is probabilistic, yielding an expected $p m r + (1-p) r n$ trainable parameters. This advantage is especially valuable in resource-constrained settings such as Federated Learning.

Detailed configurations, hardware specs, and dataset descriptions are provided in Appendix~\ref{sec:experiments_extra}.

\section*{Impact Statement}
This paper presents work whose goal is to advance the field
of Machine Learning. There are many potential societal
consequences of our work, none which we feel must be
specifically highlighted here.


\section*{Acknowledgements}
The research reported in this publication was supported by funding from King Abdullah University of Science and Technology (KAUST): i) KAUST Baseline Research Scheme, ii) CRG Grant ORFS-CRG12-2024-6460, iii) Center of Excellence for Generative AI, under award number 5940, and iv) SDAIA-KAUST Center of Excellence in Artificial Intelligence and Data Science.
\newpage
\bibliographystyle{plainnat}
\bibliography{arXiv_p-lora}

\newpage
\appendix
\onecolumn
\part*{APPENDIX}


\section{Basic Facts and Useful Inequalities}


\paragraph{Tower property.} For any random variables $X$ and $Y$, we have
\begin{eqnarray}\label{eq:tower_property}
    \Exp{\Exp{X \mid Y}} = \Exp{X}. 
\end{eqnarray}

\paragraph{Cauchy-Bunyakovsky-Schwarz inequality.} For any random variables $X$ and $Y$, we have
\begin{eqnarray}\label{eq:cauchy_bunyakovsky_schwarz}
    \abs{\Exp{XY}} \le \sqrt{\Exp{X^2}\Exp{Y^2}}. 
\end{eqnarray}

\paragraph{Variance decomposition.} For any random vector $X \in \mathbb{R}^d$ and any non-random $c \in \mathbb{R}^d$, we have
\begin{eqnarray}\label{eq:bvd}
    \Exp{\twonorm{X - c}^2} = \Exp{\twonorm{X - \Exp{X}}^2} + \twonorm{\Exp{X} - c}^2.
\end{eqnarray}

\paragraph{Jensen’s inequality.} For any random vector $X \in \mathbb{R}^d$ and any convex function $g:\R^d\mapsto\R$, we have
\begin{eqnarray}\label{eq:jensen_general}
    g(\Exp{X}) \le \Exp{g(X)}.
\end{eqnarray}

\newpage

\section{Discussion on Positive Expected Projection (Assumption \ref{as:projection_matrix})}\label{sec:discussion_projection_matrix}

Assumption~\ref{as:projection_matrix} merits further discussion. While any single projection matrix has eigenvalues that are either 0 or 1 (with the smallest being 0), the expected value of a \textit{random} projection matrix can have all its eigenvalues strictly greater than zero. This property is crucial for ensuring stable convergence behavior in our framework.

Later in this section, we will utilize the following lemma, which is a classical result from linear algebra, often known as a direct consequence of Schur's Lemma \citep{hall2013lie, schur2024neue}.

\begin{lemma}[Rotational Invariance Implies Scalar Matrix]\label{lem:rotation_scalar}
Let $M \in \mathbb{R}^{n \times n}$ be a matrix satisfying
\begin{align}\label{eq:M_invariant}
M = QMQ^\top \quad \text{for all orthonormal matrices } Q \in \mathbb{R}^{n \times n}.
\end{align}
Then $M = \alpha I_n$ for some scalar $\alpha \in \mathbb{R}$.
\end{lemma}
\begin{proof}
The condition $M = QMQ^\top$ is equivalent to $MQ = QM$, which means that $M$ commutes with every orthonormal matrix $Q$.
Since $M$ is a real symmetric matrix, it is guaranteed to have at least one real eigenvector. Let $v$ be such an eigenvector with corresponding eigenvalue $\lambda$. We can normalize this eigenvector to create a unit vector $u_1 = v / \|v\|$, which is also an eigenvector with the same eigenvalue:
\begin{align*}
    M u_1 = M \left(\frac{v}{\|v\|}\right) = \frac{1}{\|v\|} Mv = \frac{1}{\|v\|} (\lambda v) = \lambda \left(\frac{v}{\|v\|}\right) = \lambda u_1.
\end{align*}
Now, let $u$ be any other arbitrary unit vector in $\mathbb{R}^n$. Because both $u_1$ and $u$ are unit vectors (i.e., they lie on the unit sphere), there always exists an orthonormal matrix $Q$ (specifically, a rotation) that maps $u_1$ to $u$. That is, $u = Qu_1$.

We now examine the action of $M$ on this arbitrary unit vector $u$:
\begin{align*}
    Mu = M(Qu_1) = (MQ)u_1 = (QM)u_1 = Q(Mu_1) = Q(\lambda u_1) = \lambda (Qu_1) = \lambda u.
\end{align*}

We have shown that any arbitrary unit vector $u$ is an eigenvector of $M$ with the same eigenvalue $\lambda$. If every unit vector is an eigenvector with eigenvalue $\lambda$, then for any non-zero vector $x \in \mathbb{R}^n$, we can write $x = \|x\| \cdot \frac{x}{\|x\|}$. Let $u_x \eqdef x/\|x\|$ be the corresponding unit vector. Then:
\begin{align*}
    Mx = M(\|x\| u_x) = \|x\| (M u_x) = \|x\| (\lambda u_x) = \lambda (\|x\| u_x) = \lambda x.
\end{align*}
Since $Mx = \lambda x$ for all vectors $x \in \mathbb{R}^n$, the matrix $M$ must be a scalar multiple of the identity matrix, i.e., $M = \lambda I_n$.
\end{proof}

In practice, LoRA-type methods often employ Gaussian sampling for the matrices $A_S$ or $B_S$ \citep{xia2024chainloraefficientfinetuning, MaoSurvey}. The following lemma, a standard result in multivariate statistics, demonstrates that under such Gaussian sampling, Assumption~\ref{as:projection_matrix} is naturally satisfied.

\begin{lemma}[Expected Eigenvalues of Random Projection Matrices]\label{le:exp_eig}
Consider a projection matrix $H_B$ generated by a random matrix $B \in \mathbb{R}^{n \times r}$ whose entries are i.i.d. $\mathcal{N}(0,1)$ with $r \leq n$, defined as:
\begin{align*}
H_B = B(B^\top B)^\dagger B^\top,
\end{align*}
where $\dagger$ denotes the Moore-Penrose pseudoinverse. Similarly, for a random matrix $A \in \mathbb{R}^{r \times n}$ with i.i.d. $\mathcal{N}(0,1)$ entries, we define:
\begin{align*}
H_A = A^\top(AA^\top)^\dagger A.
\end{align*}
For these matrices, we have:
\begin{align*}
\mathbb{E}[H_B] = \mathbb{E}[H_A] = \frac{r}{n}I_n,
\end{align*}
which implies:
\begin{align*}
\lambda_{\min}(\mathbb{E}[H_B]) = \lambda_{\min}(\mathbb{E}[H_A]) = \frac{r}{n}.
\end{align*}
\end{lemma}
\begin{proof}
The proof leverages the rotational invariance property of the standard Gaussian distribution. We will prove the result for $H_B$; the argument for $H_A$ is analogous.

First, we establish that $\mathbb{E}[H_B]$ must be a scalar multiple of the identity matrix. Let $Q \in \mathbb{R}^{n \times n}$ be an arbitrary orthonormal matrix. Due to the rotational invariance of the multivariate standard normal distribution, the random matrix $QB$ has the same distribution as $B$.

Consider the projection matrix $H_{QB}$ generated by $QB$:
\begin{align*}
H_{QB} &= (QB)\left((QB)^\top QB\right)^\dagger(QB)^\top \\
&= QB\left(B^\top Q^\top QB\right)^\dagger B^\top Q^\top \\
&= QB\left(B^\top B\right)^\dagger B^\top Q^\top \\
&= Q\left(B(B^\top B)^\dagger B^\top\right)Q^\top = QH_B Q^\top.
\end{align*}
Since $QB$ and $B$ are identically distributed, their expectations must be equal: $\mathbb{E}[H_{QB}] = \mathbb{E}[H_B]$. This implies:
\begin{align*}
\mathbb{E}[H_B] = Q\mathbb{E}[H_B]Q^\top,
\end{align*}
for every orthonormal matrix $Q$. By Lemma \ref{lem:rotation_scalar}, $\mathbb{E}[H_B]$ must be a scalar multiple of the identity matrix, so $\mathbb{E}[H_B] = \alpha I_n$ for some scalar $\alpha \in \R$.

To determine this scalar, we use the property that the trace of a projection matrix is equal to its rank. Since the columns of $B$ are drawn from a continuous distribution, they are linearly independent almost surely (as $r \leq n$). Thus, the rank of $H_B$ is $r$.
\begin{align*}
\mathbb{E}[\trace{H_B}] = \mathbb{E}[\text{rank}(H_B)] = r.
\end{align*}
By linearity of expectation and trace, we also have:
\begin{align*}
\mathbb{E}[\trace{H_B}] = \trace{\mathbb{E}[H_B]} = \trace{\alpha I_n} = \alpha n.
\end{align*}
Equating the two expressions gives $\alpha n = r$, which implies $\alpha = \frac{r}{n}$. Therefore,
\begin{align*}
\mathbb{E}[H_B] = \frac{r}{n}I_n.
\end{align*}
The same argument applies to $H_A$ by observing that $A^\top$ is an $n \times r$ matrix with i.i.d. $\mathcal{N}(0,1)$ entries, which completes the proof.
\end{proof}
\begin{remark}
    This result is foundational in the study of random projections and can be found in standard textbooks on multivariate statistics; for example, see Lemma 5.3.2 in \citep{vershynin2009high}.
\end{remark}

\newpage
\section{Proofs for Core Algorithmic Variants}
\subsection{Analysis of Bernoulli-LoRA-GD}\label{sec:Bernoulli-LoRA}

\begin{algorithm}[H]
\caption{\algname{Bernoulli-LoRA-GD}}\label{alg:Bernoulli-LoRA-GD-smooth}
\begin{algorithmic}[1]
\STATE \textbf{Parameters:} pre-trained model $W^0 \in \mathbb{R}^{m \times n}$, rank $r \ll \min\{m,n\}$, scaling factor $\alpha > 0$, stepsize $\gamma_t$ chain length $T$, sketch distribution $\mathcal{D}_S^B$ or $\mathcal{D}_S^A$, Bernoulli probability $p$

\FOR{$t = 0, 1, \ldots, T-1$}
    \STATE Sample $c^t \sim \text{Be}(p)$ \hfill{Bernoulli random variable}
    \IF{$c^t = 1$}
        \STATE Sample $B_S^t \sim \mathcal{D}_S^B$ \hfill{Left sketch}
        \STATE $\hat{A}^t = -\eta \rb{\rbtop{B_S^t}B_S^t}^{\dagger}\rbtop{B_S^t}\nf{W^t}$ 
        \STATE $W^{t+1} = W^t + \frac{\alpha}{r}B_S^t\hat{A}^t$
    \ELSE
        \STATE Sample $A_S^t \sim \mathcal{D}_S^A$ \hfill{Right sketch}
        \STATE $\hat{B}^t = -\eta \nf{W^t}\rbtop{A_S^t}\rb{A_S^t\rbtop{A_S^t}}^{\dagger}$
        \STATE $W^{t+1} = W^t + \frac{\alpha}{r}\hat{B}^tA_S^t$
    \ENDIF
\ENDFOR
\end{algorithmic}
\end{algorithm}



The following lemma establishes that the \algname{Bernoulli-LoRA} update can be reformulated as a standard projected gradient descent step, providing a crucial foundation for our subsequent convergence analysis.
\begin{lemma}\label{le:smooth_left_right_sketch}
Consider the updates $\hA^t$ and $\hB^t$ from Algorithm \ref{alg:Bernoulli-LoRA-GD-smooth} computed as solutions to the following optimization problems:
\begin{eqnarray}
\hat{A}^t &\eqdef& \argmin\limits_{A}\cb{f(\wt) + \frac{\alpha}{r}\linf{\nabla f(\wt), {B_S^t}{A}} + \frac{\alpha^2}{2 \gamma r^2}\sqfnorm{B_S^tA}}, \nonumber  \\
\hat{B}^t &\eqdef& \argmin\limits_{B}\cb{f(\wt) + \frac{\alpha}{r}\linf{\nabla f(\wt), {B} A_S^t} + \frac{\alpha^2}{2\gamma r^2}\sqfnorm{B A_S^t}}. \label{eq:cl;sfdc;s}
\end{eqnarray} 
Then the Left and Right sketch updates can be expressed as a gradient descent step:
\begin{eqnarray}
\wtpo = \wt - \gamma G^t,
\end{eqnarray}
where $G^t$ is defined by
\begin{eqnarray}
G^t = \begin{cases}
H_B^t\nabla f(\wt), & \text{with probability } p \\
\nabla f(\wt)H_A^t, & \text{with probability } 1-p
\end{cases}
\end{eqnarray}
with projection matrices $H_A^t$ and $H_B^t$ given by:
\begin{eqnarray}\label{eq:proj_matrices}
H_A^t \eqdef \rbtop{A_S^t}\rb{A_S^t\rbtop{A_S^t}}^{\dagger}A_S^t \quad \text{and} \quad H_B^t \eqdef B_S^t\rb{\rbtop{B_S^t}B_S^t}^{\dagger}\rbtop{B_S^t},
\end{eqnarray}
where ${\dagger}$ denotes the Moore-Penrose pseudoinverse.
\end{lemma}

\begin{proof}
Following Algorithm \ref{alg:Bernoulli-LoRA-GD-smooth}, at each iteration we randomly select either the Left sketch (with probability $p$) or the Right sketch (with probability $1-p$). We analyze both cases separately and then combine them into a unified update rule.

\textbf{Left Sketch Analysis.} When the Left sketch is selected, the update takes the form:
\begin{eqnarray}
\wtpo = \wt + \frac{\alpha}{r} B_S^t \hat{A}^t.
\end{eqnarray}



Minimizing the right-hand side with respect to $\hat{A}^t$ yields:
\begin{eqnarray}
\frac{\alpha}{r}\rbtop{B_S^t}\nabla f(\wt) + \frac{\alpha^2}{\gamma r^2}\rbtop{B_S^t}B_S^t\hat{A}^t &=& 0; \nonumber \\
\rbtop{B_S^t}B_S^t\hat{A}^t &=& -\frac{\gamma r}{\alpha}\rbtop{B_S^t}\nabla f(\wt); \nonumber \\
\hat{A}^t &=& -\frac{\gamma r}{\alpha}\rb{\rbtop{B_S^t}B_S^t}^{\dagger}\rbtop{B_S^t}\nabla f(\wt).
\end{eqnarray}

This leads to the Left sketch update:
\begin{eqnarray}\label{eq:smooth_wt_grad_step_left}
\wtpo &=& \wt + \frac{\alpha}{r}B_S^t\hat{A}^t \nonumber \\
&=& \wt - \gamma B_S^t\rb{\rbtop{B_S^t}B_S^t}^{\dagger}\rbtop{B_S^t}\nabla f(\wt)\nonumber \\
&=& \wt - \gamma H_B^t\nabla f(\wt),
\end{eqnarray}
where $H_B^t \eqdef B_S^t\rb{\rbtop{B_S^t}B_S^t}^{\dagger}\rbtop{B_S^t}$ is a projection matrix.

\textbf{Right Sketch Analysis.} For the Right sketch, we follow a similar approach. The update rule is:
\begin{eqnarray}
\wtpo = \wt + \frac{\alpha}{r}\hat{B}^t A_S^t.
\end{eqnarray}

First, observe that:
\begin{eqnarray}
\norm{\hat{B}^t A_S^t}_{\opern{F}}^2 = \linf{\hat{B}^t A_S^t, \hat{B}^t A_S^t} = \linf{A_S^t, \rbtop{\hat{B}^t}\hat{B}^t A_S^t}.
\end{eqnarray}


For the linear term from \eqref{eq:cl;sfdc;s}:
\begin{eqnarray}
\frac{\alpha}{r}\linf{\nabla f(\wt), \hat{B}^t A_S^t} = \frac{\alpha}{r} \trace{\rbtop{\nabla f(\wt)} \hat{B}^t A_S^t},
\end{eqnarray}
with gradient $\nabla f(\wt)\rbtop{A_S^t}$ with respect to $\hat{B}^t$. Using the matrix calculus identity $\nabla_X\norm{X}_{\opern{F}}^2 = 2X$, the gradient of the quadratic term is:
\begin{eqnarray}
\frac{\alpha^2}{\gamma r^2}\hat{B}^t A_S^t\rbtop{A_S^t}.
\end{eqnarray}

Setting the total gradient to zero and solving for $\hat{B}^t$:
\begin{eqnarray}
\hat{B}^t = -\frac{\gamma r}{\alpha}\nabla f(\wt)\rbtop{A_S^t}\rb{A_S^t\rbtop{A_S^t}}^{\dagger},
\end{eqnarray}
which yields the Right sketch update:
\begin{eqnarray}\label{eq:smooth_wt_grad_step_right}
\wtpo &=& \wt + \frac{\alpha}{r}\hat{B}^t A_S^t \nonumber \\
&=& \wt - \gamma \nabla f(\wt)\rbtop{A_S^t}\rb{A_S^t\rbtop{A_S^t}}^{\dagger}A_S^t \nonumber \\
&=& \wt - \gamma \nabla f(\wt)H_A^t,
\end{eqnarray}
where $H_A^t \eqdef \rbtop{A_S^t}\rb{A_S^t\rbtop{A_S^t}}^{\dagger}A_S^t$ is a projection matrix.

\textbf{Combined Update Rule.} Combining equations \eqref{eq:smooth_wt_grad_step_left} and \eqref{eq:smooth_wt_grad_step_right}, we obtain the unified update:
\begin{eqnarray}
\wtpo = \wt - \gamma G^t,
\end{eqnarray}
where $G^t$ takes the form given in the lemma statement, completing the proof.
\end{proof}

With these assumptions in place, we can now state our main convergence result for RAC-LoRA with Gradient Descent updates.

\subsubsection{Convergence for Smooth Non-Convex Functions}
\textbf{Theorem~\ref{thm:smooth_non_cvx}}.
\textit{Let Assumptions \ref{as:projection_matrix}, \ref{as:lipschitz_smoothness}, and \ref{as:bounded_below} hold, and let the stepsize satisfy $0 < \gamma \leq \frac{1}{L}$. 
Then the iterates of \algname{Bernoulli-LoRA-GD} (Algorithm \ref{alg:Bernoulli-LoRA-GD-smooth}), with matrices $\hat{A}^t$ and $\hat{B}^t$ computed according to Lemma \ref{le:smooth_left_right_sketch}, satisfy
\begin{eqnarray}
\mathbb{E}\left[\sqfnorm{\nabla f(\widetilde{W}^T)}\right] \leq \frac{2(f(W^0) - f^*)}{\gamma \lambda_{\min}^{p} T},
\end{eqnarray}
where $\lambda_{\min}^{p} := p\lambda_{\min}^{H_B} + (1-p)\lambda_{\min}^{H_A}$
and 
$\widetilde{W}^T$ is drawn uniformly at random from the iterate sequence $\{W^0, W^1, \ldots, W^{T-1}\}$.}

\begin{proof}
From Lemma \ref{le:smooth_left_right_sketch}, we know that Bernoulli-LoRA updates can be expressed as
\begin{eqnarray}
\wtpo = \wt - \gamma G^t,
\end{eqnarray}
where $G^t$ takes the form
\begin{eqnarray}
G^t = \begin{cases}
H_B^t\nabla f(\wt), & \text{with probability } p \\
\nabla f(\wt)H_A^t, & \text{with probability } 1-p
\end{cases}
\end{eqnarray}
with projection matrices $H_A^t$ and $H_B^t$ as defined in the lemma.

To analyze the convergence, we first compute the conditional expectation and second moment of $G^t$:
\begin{eqnarray}
\Exp{G^t \mid \wt, H^t} &=& p H_B^t\nabla f(\wt) + (1-p)\nabla f(\wt)H_A^t, \nonumber \\
\Exp{\fnorm{G^t}^2 \mid \wt, H^t} &=& p\fnorm{H_B^t\nabla f(\wt)}^2 + (1-p)\fnorm{\nabla f(\wt)H_A^t}^2,
\end{eqnarray}
where we defined $H^t \eqdef \cb{H_A^t, H_B^t}$.

We begin by establishing several key auxiliary bounds. For the Left sketch term:
\begin{eqnarray}\label{eq:non-cvx;smooth_H_B}
&&-\gamma p \linf{\nabla f(\wt), H_B^t\nabla f(\wt)} + \frac{L\gamma^2}{2}p\fnorm{H_B^t\nabla f(\wt)}^2 \nonumber \\
&&= -\gamma p \linf{\nabla f(\wt), H_B^t\nabla f(\wt)} + \frac{L\gamma^2}{2}p\linf{H_B^t\nabla f(\wt), H_B^t\nabla f(\wt)} \nonumber \\
&&= -\gamma p \linf{\nabla f(\wt), H_B^t\nabla f(\wt)} + \frac{L\gamma^2}{2}p\linf{\nabla f(\wt), \rbtop{H_B^t}H_B^t\nabla f(\wt)} \nonumber \\
&&= p\rb{-\gamma \linf{\nabla f(\wt), H_B^t\nabla f(\wt)} + \frac{L\gamma^2}{2}\linf{\nabla f(\wt), H_B^t\nabla f(\wt)}} \nonumber \\
&&\letext{$\gamma\le \nfr{1}{L}$} -\frac{\gamma}{2}p\linf{\nabla f(\wt), H_B^t\nabla f(\wt)}.
\end{eqnarray}

For any projection matrix $H_A^t$, we have:
\begin{eqnarray}
\linf{\nabla f(\wt)H_A^t, \nabla f(\wt)H_A^t} &=& \trace{\rbtop{H_A^t}\rbtop{\nabla f(\wt)}\nabla f(\wt)H_A^t} \nonumber \\
&=& \trace{ \rbtop{\nabla f(\wt)}\nabla f(\wt)H_A^t\rbtop{H_A^t}} \nonumber \\
&=& \trace{\rbtop{\nabla f(\wt)}\nabla f(\wt)H_A^t} \nonumber \\
&=& \linf{\nabla f(\wt), \nabla f(\wt)H_A^t}.
\end{eqnarray}

Therefore:
\begin{eqnarray}\label{eq:non-cvx;smooth_H_A}
&&-\gamma(1-p)\linf{\nabla f(\wt), \nabla f(\wt)H_A^t} + \frac{L\gamma^2}{2}(1-p)\fnorm{\nabla f(\wt)H_A^t}^2 \nonumber \\
&&= -\gamma(1-p)\linf{\nabla f(\wt), \nabla f(\wt)H_A^t} + \frac{L\gamma^2}{2}(1-p)\linf{\nabla f(\wt)H_A^t, \nabla f(\wt)H_A^t} \nonumber \\
&&= -\gamma(1-p)\linf{\nabla f(\wt), \nabla f(\wt)H_A^t} + \frac{L\gamma^2}{2}(1-p)\linf{\nabla f(\wt), \nabla f(\wt)H_A^t} \nonumber \\
&&\letext{$\gamma\le \nfr{1}{L}$} -\frac{\gamma}{2}(1-p)\linf{\nabla f(\wt), \nabla f(\wt)H_A^t}.
\end{eqnarray}

Using the Lipschitz gradient condition and the above bounds:
\begin{eqnarray}
\Exp{f(\wtpo) \mid \wt, H^t} &\leq& f(\wt) + \Exp{\linf{\nabla f(\wt), \wtpo - \wt} \mid \wt, H^t} \nonumber\\
&+& \frac{L}{2}\Exp{\sqfnorm{\wtpo - \wt} \mid \wt, H^t} \nonumber \\
&=& f(\wt) - \gamma\linf{\nabla f(\wt), \Exp{G^t \mid \wt, H^t}} + \frac{L\gamma^2}{2}\Exp{\sqfnorm{G^t} \mid \wt, H^t}
\nonumber \\
&=& f(\wt) - \gamma p\linf{\nabla f(\wt), H_B^t\nabla f(\wt)} - \gamma(1-p)\linf{\nabla f(\wt), \nabla f(\wt)H_A^t} \nonumber \\
& +& \frac{L\gamma^2}{2}p\sqfnorm{H_B^t\nabla f(\wt)} + \frac{L\gamma^2}{2}(1-p)\sqfnorm{\nabla f(\wt)H_A^t} \nonumber \\
&\letext{\eqref{eq:non-cvx;smooth_H_B},\eqref{eq:non-cvx;smooth_H_A}}& f(\wt) - \frac{\gamma}{2}\left(p\linf{\nabla f(\wt), H_B^t\nabla f(\wt)} + (1-p)\linf{\nabla f(\wt), \nabla f(\wt)H_A^t}\right).\nonumber \\
\end{eqnarray}

For the first term:
\begin{eqnarray}
-\linf{\nabla f(\wt), \Exp{H_B^t}\nabla f(\wt)} &=& -\tr\left(\rbtop{\nabla f(\wt)}\Exp{H_B^t}\nabla f(\wt)\right) \nonumber \\
&\leq& -\lambda_{\min}\left(\Exp{H_B^t}\right)\tr\left(\rbtop{\nabla f(\wt)}\nabla f(\wt)\right) \nonumber \\
&=& -\lambda_{\min}^{H_B}\sqfnorm{\nabla f(\wt)}.
\end{eqnarray}

Similarly, for the second term:
\begin{eqnarray}
-\linf{\nabla f(\wt), \nabla f(\wt)\Exp{H_A^t}} &=& -\tr\left(\rbtop{\nabla f(\wt)}\nabla f(\wt)\Exp{H_A^t}\right) \nonumber \\
&=& -\tr\left(\Exp{H_A^t}\rbtop{\nabla f(\wt)}\nabla f(\wt)\right) \nonumber \\
&\leq& -\lambda_{\min}^{H_A}\sqfnorm{\nabla f(\wt)}.
\end{eqnarray}

Therefore:
\begin{eqnarray}
\Exp{f(\wtpo) \mid \wt}  &=&\Exp{\Exp{f(\wtpo) \mid \wt, H^t} \mid \wt} \nonumber \\
&\leq& f(\wt) - \frac{\gamma}{2}\left(p\linf{\nabla f(\wt), \Exp{H_B^t}\nabla f(\wt)} + (1-p)\linf{\nabla f(\wt), \nabla f(\wt)\Exp{H_A^t}}\right) \nonumber \\
&\leq& f(\wt) - \frac{\gamma}{2}\left(p\lambda_{\min}^{H_B} + (1-p)\lambda_{\min}^{H_A}\right)\sqfnorm{\nabla f(\wt)} \nonumber \\
&=& f(\wt) - \frac{\gamma}{2}\lambda_{\min}^{p}\sqfnorm{\nabla f(\wt)},
\end{eqnarray}
where $\lambda_{\min}^{p} := p\lambda_{\min}^{H_B} + (1-p)\lambda_{\min}^{H_A}$.
Further,
\begin{eqnarray}
\Exp{\Exp{f(\wtpo) \mid \wt, H^t} \mid \wt} - f^{\star} \leq f(\wt) - f^{\star} - \frac{\gamma}{2}\lambda_{\min}^{p}\sqfnorm{\nabla f(\wt)}.
\end{eqnarray}

Taking the sum over $t=0,\ldots,T-1$ and using the tower property of expectation:
\begin{eqnarray}
\Exp{f(W^T) - f^{\star}} &\leq& f(W^0) - f^{\star} - \frac{\gamma}{2}\lambda_{\min}^{p}\sum_{t=0}^{T-1}\Exp{\sqfnorm{\nabla f(\wt)}}.
\end{eqnarray}

By rearranging terms, we get:
\begin{eqnarray}
\frac{\gamma}{2}\lambda_{\min}^{p}\sum_{t=0}^{T-1}\Exp{\sqfnorm{\nabla f(\wt)}} \leq f(W^0) - f^{\star}.
\end{eqnarray}

Finally, dividing both sides by $\frac{\gamma T}{2}\lambda_{\min}^{p}$ yields:
\begin{eqnarray}
\Exp{\sqfnorm{\nabla f(\widetilde{W}^T)}} \leq \frac{2(f(W^0) - f^{\star})}{\gamma\lambda_{\min}^{p} T},
\end{eqnarray}
where $\widetilde{W}^T$ is chosen uniformly at random from $\{W^0,W^1,\ldots,W^{T-1}\}$, completing the proof.

\end{proof}

\subsubsection{Convergence under Polyak-{\L}ojasiewicz Condition}


\begin{theorem}\label{thm:smooth_pl}
Let Assumptions \ref{as:projection_matrix}, \ref{as:bounded_below}, \ref{as:lipschitz_smoothness}, and \ref{as:pl_condition} hold, and let the stepsize satisfy $0 < \gamma \leq \frac{1}{L}$. Then the iterates of \algname{Bernoulli-LoRA-GD} (Algorithm \ref{alg:Bernoulli-LoRA-GD-smooth}), with matrices $\hat{A}^t$ and $\hat{B}^t$ computed according to Lemma \ref{le:smooth_left_right_sketch}, satisfy
\begin{align*}
\squeeze
\Exp{f(W^{T}) - f^*} \leq \left(1 - \gamma \mu \lambda_{\min}^{p}\right)^T \left(f(W^0) - f^*\right),
\end{align*}
where $\lambda_{\min}^{p} := p\lambda_{\min}^{H_B} + (1-p)\lambda_{\min}^{H_A}$.
\end{theorem}

\begin{proof}
We begin our analysis from a key inequality derived in the proof of Theorem~\ref{thm:smooth_non_cvx}:
\begin{align}
\label{eq:gd_one_step_progress}
\Exp{f(\wtpo) \mid \wt} \leq f(\wt) - \frac{\gamma}{2}\lambda_{\min}^{p}\sqfnorm{\nabla f(\wt)}.
\end{align}
By invoking the Polyak-Łojasiewicz condition (Assumption~\ref{as:pl_condition}), which states that $\frac{1}{2}\sqfnorm{\nabla f(W)} \geq \mu\rb{f(W)-f^{*}}$, we can further bound the right-hand side of the inequality \eqref{eq:gd_one_step_progress}:
\begin{align*}
\Exp{f(\wtpo) \mid \wt} &\leq f(\wt) - \gamma \lambda_{\min}^{p} \left( \mu \rb{f(\wt) - f^*} \right).
\end{align*}
Subtracting the optimal function value $f^*$ from both sides, we get a recursive relationship for the expected suboptimality gap:
\begin{align*}
\Exp{f(\wtpo) - f^* \mid \wt} &\leq \rb{f(\wt) - f^*} - \gamma \mu \lambda_{\min}^{p} \rb{f(\wt) - f^*} \\
&= \left(1 - \gamma \mu \lambda_{\min}^{p}\right)\left(f(\wt) - f^*\right).
\end{align*}
By taking the full expectation over all randomness up to iteration $t$ and applying the tower property, we obtain:
\begin{align*}
\Exp{f(W^{t+1}) - f^*} \leq \left(1 - \gamma \mu \lambda_{\min}^{p}\right)\Exp{f(W^t) - f^*}.
\end{align*}
Unrolling this recursion from $t=T-1$ down to $t=0$ yields the final linear convergence result:
\begin{align*}
\Exp{f(W^T) - f^*} \leq \left(1 - \gamma \mu \lambda_{\min}^{p}\right)^T \left(f(W^0) - f^*\right).
\end{align*}
This completes the proof.
\end{proof}

\subsubsection{Convergence for Non-Smooth Convex Functions}


\begin{algorithm}[H]
\caption{\algname{Bernoulli-LoRA-GD}(Non-smooth setting)}\label{alg:Bernoulli-LoRA_nonsmooth}
\begin{algorithmic}[1]
\STATE \textbf{Parameters:} pre-trained model $W^0 \in \mathbb{R}^{m \times n}$, rank $r \ll \min\{m,n\}$, scaling factor $\alpha > 0$, stepsize $\gamma_t$ chain length $T$, sketch distribution $\mathcal{D}_S^B$ or $\mathcal{D}_S^A$, Bernoulli probability $p$

\FOR{$t = 0, 1, \ldots, T-1$}
    \STATE Sample $c^t \sim \text{Be}(p)$ \hfill{Bernoulli random variable}
    \IF{$c^t = 1$}
        \STATE Sample $B_S^t \sim \mathcal{D}_S^B$ \hfill{Left sketch}
        \STATE $\hat{A}^t = \argmin_{A}\cb{f(\wt) + \frac{\alpha}{r}\linf{\pf{\wt}, {B_S^t}{A}} + \frac{\alpha^2}{2 {\gamma_t} r^2}\sqfnorm{B_S^tA}}$ 
        \STATE $W^{t+1} = W^t + \frac{\alpha}{r}B_S^t\hat{A}^t$
    \ELSE
        \STATE Sample $A_S^t \sim \mathcal{D}_S^A$ \hfill{Right sketch}
        \STATE $\hat{B}^t = \argmin_B \cb{f(\wt) + \frac{\alpha}{r}\linf{\pf{\wt}, {B} A_S^t} + \frac{\alpha^2}{2{\gamma_t} r^2}\sqfnorm{B A_S^t}}$
        \STATE $W^{t+1} = W^t + \frac{\alpha}{r}\hat{B}^tA_S^t$
    \ENDIF
\ENDFOR
\end{algorithmic}
\end{algorithm}

Our analysis relies on the following standard assumptions that are widely used in non-smooth optimization theory:

\begin{assumption}\label{as:existence_of_minimizer}
The function $f$ has at least one minimizer, denoted by $W^*$.
\end{assumption}

\begin{assumption}\label{as:convexity}
The function $f$ is convex.
\end{assumption}

\begin{assumption}[Lipschitz continuity]\label{as:lipschitzness}
The function $f$ is $L_0$-Lipschitz continuous. That is, there exists $L_0 > 0$ such that
\begin{eqnarray}
|f(W) - f(V)| \leq L_0 \fnorm{W - V}, \quad \forall W, V \in \mathbb{R}^{m \times n}.
\end{eqnarray}
\end{assumption}

The combination of convexity and Lipschitz continuity represents a standard framework in non-smooth optimization \citep{vorontsova2021convex, Nesterov2013, bubeck2015convex, beck2017first, duchi2018introductory, lan2020first, drusvyatskiy2020convex}. Notably, the $L_0$-Lipschitz continuity implies uniformly bounded subgradients \citep{beck2017first}, a property that plays a crucial role in our analysis:
\begin{eqnarray}\label{eq:bounded_subgrad}
\fnorm{\partial f(W)} \leq L_0, \quad \forall W \in \mathbb{R}^{m \times n}.
\end{eqnarray}
This boundedness of subgradients ensures the stability of our optimization process and enables us to establish rigorous convergence guarantees.

The following lemma establishes that the \algname{Bernoulli-LoRA} update in the non-smooth case can also be reformulated as a subgradient descent step, which plays a central role in our convergence analysis for non-smooth objectives.

\begin{lemma}\label{le:nonsmooth_left_right_sketch}
Consider the updates $\hA^t$ and $\hB^t$ from Algorithm \ref{alg:Bernoulli-LoRA_nonsmooth} computed as solutions to the following optimization problems:
\begin{eqnarray}
\hat{A}^t &\eqdef& \argmin\limits_{A}\cb{f(\wt) + \frac{\alpha}{r}\linf{\pf{\wt}, {B_S^t}{A}} + \frac{\alpha^2}{2 {\gamma_t} r^2}\sqfnorm{B_S^tA}},\nonumber \\
\hat{B}^t &\eqdef& \argmin\limits_{B}\cb{f(\wt) + \frac{\alpha}{r}\linf{\pf{\wt}, {B} A_S^t} + \frac{\alpha^2}{2{\gamma_t} r^2}\norm{{B} A_S^t}_{\opern{F}}^2}.
\end{eqnarray} 
Then the Left and Right sketch updates can be expressed as a subgradient descent step:
\begin{eqnarray}
\wtpo = \wt - {\gamma_t} G^t,
\end{eqnarray}
where $G^t$ is defined by
\begin{eqnarray}
G^t = \begin{cases}
H_B^t\pf{\wt}, & \text{with probability } p \\
\pf{\wt}H_A^t, & \text{with probability } 1-p
\end{cases}
\end{eqnarray}
with projection matrices $H_A^t$ and $H_B^t$ given by:
\begin{eqnarray}
H_A^t \eqdef \rbtop{A_S^t}\rb{A_S^t\rbtop{A_S^t}}^{\dagger}A_S^t \quad \text{and} \quad H_B^t \eqdef B_S^t\rb{\rbtop{B_S^t}B_S^t}^{\dagger}\rbtop{B_S^t},
\end{eqnarray}
where ${\dagger}$ denotes the Moore-Penrose pseudoinverse.
\end{lemma}
\begin{proof}

The proof follows a similar structure to that of Lemma \ref{le:smooth_left_right_sketch}, with subgradients replacing gradients throughout the analysis. We examine both sketch types separately before combining them into a unified update rule.

\textbf{Left Sketch Analysis.} When the Left sketch is selected, the update takes the form:
\begin{eqnarray}
\wtpo = \wt + \frac{\alpha}{r} B_S^t \hat{A}^t.
\end{eqnarray}

The matrix $\hat{A}^t$ is defined as the solution to the optimization problem:
\begin{eqnarray}
\hat{A}^t \eqdef \argmin_{A}\cb{f(\wt) + \frac{\alpha}{r}\linf{\pf{\wt}, {B_S^t}{A}} + \frac{\alpha^2}{2 {\gamma_t} r^2}\sqfnorm{B_S^tA}}.
\end{eqnarray}

By computing the gradient of the objective with respect to $A$ and setting it to zero, we obtain:
\begin{eqnarray}
\frac{\alpha}{r}\rbtop{B_S^t}\pf{\wt} + \frac{\alpha^2}{{\gamma_t} r^2}\rbtop{B_S^t}B_S^t\hat{A}^t &=& 0; \nonumber \\
\hat{A}^t &=& -\frac{{\gamma_t} r}{\alpha}\rb{\rbtop{B_S^t}B_S^t}^{\dagger}\rbtop{B_S^t}\pf{\wt}.
\end{eqnarray}

Substituting this expression back into the update equation yields the Left sketch update:
\begin{eqnarray}\label{eq:nonsmooth_wt_subgrad_step_left}
\wtpo &=& \wt + \frac{\alpha}{r}B_S^t\hat{A}^t \nonumber \\
&=& \wt - {\gamma_t} B_S^t\rb{\rbtop{B_S^t}B_S^t}^{\dagger}\rbtop{B_S^t}\pf{\wt}\nonumber \\
&=& \wt - {\gamma_t} H_B^t\pf{\wt}.
\end{eqnarray}

\textbf{Right Sketch Analysis.} For the Right sketch, we follow an analogous approach. The update rule takes the form:
\begin{eqnarray}
\wtpo = \wt + \frac{\alpha}{r}\hat{B}^t A_S^t.
\end{eqnarray}

Applying similar optimization steps but now with respect to matrix $B$, we obtain:
\begin{eqnarray}
\hat{B}^t = -\frac{{\gamma_t} r}{\alpha}\pf{\wt}\rbtop{A_S^t}\rb{A_S^t\rbtop{A_S^t}}^{\dagger},
\end{eqnarray}

which leads to the Right sketch update:
\begin{eqnarray}\label{eq:nonsmooth_wt_subgrad_step_right}
\wtpo &=& \wt + \frac{\alpha}{r}\hat{B}^t A_S^t \nonumber \\
&=& \wt - {\gamma_t} \pf{\wt}\rbtop{A_S^t}\rb{A_S^t\rbtop{A_S^t}}^{\dagger}A_S^t \nonumber \\
&=& \wt - {\gamma_t} \pf{\wt}H_A^t.
\end{eqnarray}

\textbf{Combined Update Rule.} By combining equations \eqref{eq:nonsmooth_wt_subgrad_step_left} and \eqref{eq:nonsmooth_wt_subgrad_step_right}, we arrive at the unified update rule:
\begin{eqnarray}
\wtpo = \wt - {\gamma_t} G^t,
\end{eqnarray}
where $G^t$ takes the form specified in the lemma statement, thus completing the proof.
\end{proof}
\begin{assumption}
\label{as:expected_projection_is_scaling}
    Consider a projection matrix $H$ generated through either Left Sketch (Definition \ref{def:left-sketch}) or Right Sketch (Definition \ref{def:right-sketch}). For the sampling distributions $\mathcal{D}_S^B$ and $\mathcal{D}_S^A$, the  expected projection matrix $H$ satisfies
\begin{eqnarray}
\mathbb{E}[H] = \alpha I,
\end{eqnarray}
where a constant $\alpha >0$. 
\end{assumption}


\begin{theorem}\label{thm:nonsmooth_convergence}
Let Assumptions \ref{as:projection_matrix}, \ref{as:existence_of_minimizer}, \ref{as:convexity}, \ref{as:lipschitzness}, and ~\ref{as:expected_projection_is_scaling}  hold. 
Let us define the following quantities:
$\avg{W}^T \eqdef \frac{1}{T} \sum_{t=0}^{T-1} W^t$ as the averaged iterate,
$R^2_0 \eqdef \sqfnorm{W^0 - W^*}$ as the initial distance to optimum.
Consider the sequence $\cb{W^t}$ produced by \algname{Bernoulli-LoRA} (Algorithm \ref{alg:Bernoulli-LoRA_nonsmooth}) with updates of $\hat{A}^t$ and $\hat{B}^t$ computed according to Lemma \ref{le:nonsmooth_left_right_sketch}. 

\textbf{1. (Constant stepsize).} If the stepsize is constant, i.e., $\gamma_t \eqdef \gamma > 0$, then
\begin{eqnarray}
\Exp{f(\avg{W}^T) - f(W^*)} \leq \frac{R^2_0}{2\gamma\alpha T} + \frac{\gamma L_0^2}{2}.
\end{eqnarray}
Moreover, with the optimal stepsize $\gamma_* = \sqrt{\frac{(R^0)^2}{T\alpha L_0^2}}$, we obtain:
\begin{eqnarray}
\Exp{f(\avg{W}^T) - f(W^*)} \leq \frac{R^0 L_0 }{\sqrt{\alpha T}}.
\end{eqnarray}

\textbf{2. (Polyak stepsize).} If the stepsize is chosen adaptively as
\begin{eqnarray}
\gamma_t = \frac{\rb{f(W^t) - f(W^*)}}{\sqfnorm{\partial f(W^t)}},
\end{eqnarray}
then
\begin{eqnarray}
\Exp{f(\avg{W}^T) - f(W^*)} \leq \frac{R^0 L_0}{\sqrt{\alpha T}}.
\end{eqnarray}
\end{theorem}

\begin{proof}
From Lemma \ref{le:nonsmooth_left_right_sketch}, we know that Bernoulli-LoRA updates in the non-smooth setting can be expressed as
\begin{eqnarray}
W^{t+1} = W^t - \gamma_t G^t,
\end{eqnarray}
where $G^t$ takes the form
\begin{eqnarray}
G^t = \begin{cases}
H_B^t\partial f(W^t), & \text{with probability } p \\
\partial f(W^t)H_A^t, & \text{with probability } 1-p
\end{cases}
\end{eqnarray}
with projection matrices $H_A^t$ and $H_B^t$ as defined in the lemma.

To analyze the convergence, we first compute the conditional expectation and second moment of $G^t$:
\begin{eqnarray}
\Exp{G^t \mid W^t, H^t} &=& p H_B^t\partial f(W^t) + (1-p)\partial f(W^t)H_A^t, \label{eq:nonsmooth_expectation} \\
\Exp{\fnorm{G^t}^2 \mid W^t, H^t} &=& p\fnorm{H_B^t\partial f(W^t)}^2 + (1-p)\fnorm{\partial f(W^t)H_A^t}^2, \label{eq:nonsmooth_second_moment}
\end{eqnarray}
where we defined $H^t \eqdef \{H_A^t, H_B^t\}$.

By the definition of subgradient, we have:
\begin{eqnarray}
f(W^*) &\geq& f(W^t) + \linf{\partial f(W^t), W^* - W^t},
\end{eqnarray}
which implies:
\begin{eqnarray}\label{eq:non-smooth-bound}
\linf{\partial f(W^t), W^t - W^*} &\geq& f(W^t) - f(W^*).
\end{eqnarray}

Let us establish key auxiliary bounds. First, for the inner product terms:
\begin{eqnarray}\label{eq:fnorm_inner_product}
-2\gamma_t\Exp{\linf{G^t, W^t - W^*} \mid W^t, H^t} &\eqtext{\eqref{eq:nonsmooth_expectation}}& -2\gamma_t p\linf{H_B^t\partial f(W^t), W^t - W^*} \nonumber \\
&& - 2\gamma_t(1-p)\linf{\partial f(W^t)H_A^t, W^t - W^*}. 
\end{eqnarray}

For projection matrices, we have the following properties:
\begin{eqnarray}\label{eq:H_A_t-norm}
\sqfnorm{\partial f(W^t)H_A^t} &=& \linf{\partial f(W^t)H_A^t, \partial f(W^t)H_A^t} \nonumber \\
&=& \trace{\rbtop{H_A^t}\rbtop{\partial f(W^t)}\partial f(W^t)H_A^t} \nonumber \\
&=& \trace{\rbtop{\nabla f(\wt)}\nabla f(\wt)H_A^t\rbtop{H_A^t}} \nonumber \\
&=& \trace{\rbtop{\partial f(W^t)}\partial f(W^t)H_A^t} \nonumber \\
&=& \linf{\partial f(W^t), \partial f(W^t)H_A^t},
\end{eqnarray}
and similarly, one can show that
\begin{eqnarray}\label{eq:H_B_t-norm}
\sqfnorm{H_B^t\partial f(W^t)} = \linf{\partial f(W^t), H_B^t\partial f(W^t)}.
\end{eqnarray}

This allows us to express the second moment term as:
\begin{eqnarray}\label{eq:fnorm_second_term}
\gamma_t^2\Exp{\sqfnorm{G^t} \mid W^t, H^t} &\eqtext{\eqref{eq:nonsmooth_second_moment}}& \gamma_t^2p\sqfnorm{H_B^t\partial f(W^t)} + \gamma_t^2(1-p)\sqfnorm{\partial f(W^t)H_A^t} \nonumber \\
&\eqtext{\eqref{eq:H_A_t-norm}, \eqref{eq:H_B_t-norm}}& \gamma_t^2p\linf{\partial f(W^t), H_B^t\partial f(W^t)}  + \gamma_t^2(1-p)\linf{\partial f(W^t), \partial f(W^t)H_A^t}. \nonumber \\
\end{eqnarray}

Combining these bounds, we can analyze the distance to the optimal solution:
\begin{eqnarray}\label{eq:nonsmooth_res_exp}
\Exp{\sqfnorm{W^{t+1} - W^*} \mid W^t, H^t} \nonumber &=& \Exp{\sqfnorm{W^t - \gamma_t G^t - W^*} \mid W^t, H^t} \nonumber \\
&=& \sqfnorm{W^t - W^*} - 2\gamma_t\Exp{\linf{G^t, W^t - W^*} \mid W^t, H^t} \nonumber\\
&&+ \gamma_t^2\Exp{\sqfnorm{G^t} \mid W^t, H^t} \nonumber \\
&\eqtext{\eqref{eq:fnorm_inner_product}, \eqref{eq:fnorm_second_term}}& \sqfnorm{W^t - W^*} -2\gamma_t p\linf{H_B^t\partial f(W^t), W^t - W^*} \nonumber \\
&&  - 2\gamma_t(1-p)\linf{\partial f(W^t)H_A^t, W^t - W^*}  + \gamma_t^2p\linf{\partial f(W^t), H_B^t\partial f(W^t)}\nonumber \\
&&  + \gamma_t^2(1-p)\linf{\partial f(W^t), \partial f(W^t)H_A^t}.
\end{eqnarray}

For the expected projection matrices (see Assumption~\ref{as:expected_projection_is_scaling}), we have:
\begin{eqnarray}\label{eq:nonsmooth_inner_product_B_exp}
\linf{\partial f(W^t), \Exp{H_B^t}\partial f(W^t)} &=& \tr\left(\rbtop{\partial f(W^t)}\Exp{H_B^t}\partial f(W^t)\right) \nonumber \\
&=& \alpha\tr\left(\rbtop{\partial f(W^t)}\partial f(W^t)\right) \nonumber \\
&=& \alpha\sqfnorm{\partial f(W^t)},
\end{eqnarray}
and similarly,
\begin{eqnarray}\label{eq:nonsmooth_inner_product_A_exp}
\linf{\partial f(W^t), \partial f(W^t)\Exp{H_A^t}} &=& \alpha\sqfnorm{\partial f(W^t)}.
\end{eqnarray}

Taking expectation of both sides of \eqref{eq:nonsmooth_res_exp} again, we get
\small
\begin{eqnarray}\label{eq:nonsmooth_res_exp_exp}
\Exp{\sqfnorm{\wtpo - \ws} \mid \wt} &=&\Exp{\Exp{\sqfnorm{\wtpo - \ws} \mid \wt, H^t}\mid \wt}  \\
&=& \sqfnorm{\wt - \ws} -2{\gamma_t} p\linf{\Exp{H_B^t}\pf{\wt}, \wt - \ws}\\
&& - 2{\gamma_t}(1-p)\linf{\pf{\wt}\Exp{H_A^t}, \wt - \ws} \nonumber \\
&&  + \gamma^2_tp\linf{\pf{\wt}, \Exp{H_B^t}\pf{\wt}} + \gamma^2_t(1-p)\linf{\pf{\wt}, \pf{\wt}\Exp{H_A^t}}\nonumber \\
&\eqtext{\eqref{eq:nonsmooth_inner_product_B_exp},\eqref{eq:nonsmooth_inner_product_A_exp}}& \sqfnorm{\wt - \ws} -2{\gamma_t} p\alpha\linf{\pf{\wt}, \wt - \ws} \\
&&- 2{\gamma_t}(1-p)\alpha\linf{\pf{\wt}, \wt - \ws} + \gamma^2_t\alpha\sqfnorm{\pf{\wt}}\nonumber \\
&=& \sqfnorm{\wt - \ws} -2{\gamma_t}\alpha\linf{\pf{\wt}, \wt - \ws} + \gamma^2_t\alpha\sqfnorm{\pf{\wt}} \nonumber \\
&\eqtext{\eqref{eq:non-smooth-bound}}& \sqfnorm{\wt - \ws} -2{\gamma_t}\alpha\rb{f(\wt) - f(\ws)} + \gamma^2_t\alpha\sqfnorm{\pf{\wt}}.
\end{eqnarray}
\normalsize
By Assumption \ref{as:lipschitzness}, subgradients are uniformly bounded (see \citep{beck2017first}):
\begin{eqnarray}\label{eq:L_0_bound}
\fnorm{\partial f(W)} \leq L_0 \quad \forall W \in \mathbb{R}^{m\times n}.
\end{eqnarray}

Now we analyze both stepsize strategies separately.

\textbf{1. (Constant stepsize).} Let us first consider using a fixed stepsize $\gamma_t \eqdef \gamma > 0$. 
Taking expectation of both sides of \eqref{eq:nonsmooth_res_exp_exp} again, applying tower property \eqref{eq:tower_property} and using the bound \eqref{eq:L_0_bound}, we obtain:
\begin{eqnarray}\label{eq:const_step_iter}
\Exp{\sqfnorm{W^{t+1} - W^*}} \leq \Exp{\sqfnorm{W^t - W^*}} - 2\gamma\alpha\Exp{f(W^t) - f(W^*)} + \gamma^2\alpha L_0^2.
\end{eqnarray}

Rearranging terms in \eqref{eq:const_step_iter}:
\begin{eqnarray}\label{eq:const_step_rearr}
2\gamma\alpha\Exp{f(W^t) - f(W^*)} \leq \Exp{\sqfnorm{W^t - W^*}} - \Exp{\sqfnorm{W^{t+1} - W^*}} + \gamma^2\alpha L_0^2.
\end{eqnarray}

Summing inequality \eqref{eq:const_step_rearr} for $t = 0,\ldots,T-1$:
\begin{eqnarray}\label{eq:const_step_sum}
2\gamma\alpha\sum_{t=0}^{T-1}\Exp{f(W^t) - f(W^*)} &\leq& \sum_{t=0}^{T-1}\left(\Exp{\sqfnorm{W^t - W^*}} - \Exp{\sqfnorm{W^{t+1} - W^*}}\right) \nonumber \\
&& + T\gamma^2\alpha L_0^2 \nonumber \\
&=& \Exp{\sqfnorm{W^0 - W^*}} - \Exp{\sqfnorm{W^T - W^*}} + T\gamma^2\alpha L_0^2 \nonumber \\
&\leq& \sqfnorm{W^0 - W^*} + T\gamma^2\alpha L_0^2,
\end{eqnarray}
where the last inequality follows from the non-negativity of $\sqfnorm{W^T - W^*}$.

For the averaged iterate $\avg{W}^T \eqdef \frac{1}{T} \sum_{t=0}^{T-1} W^t$, by convexity of $f$ we have:
\begin{eqnarray}\label{eq:const_step_avg}
\Exp{f(\avg{W}^T) - f(W^*)} &\leq& \frac{1}{T}\sum_{t=0}^{T-1}\Exp{f(W^t) - f(W^*)} \nonumber \\
&\letext{\eqref{eq:const_step_sum}}& \frac{\sqfnorm{W^0 - W^*}}{2\gamma\alpha T} + \frac{\gamma L_0^2}{2} \nonumber \\
&=& \frac{(R^0)^2}{2\gamma\alpha T} + \frac{\gamma L_0^2}{2},
\end{eqnarray}
where we denoted $(R^0)^2 \eqdef \sqfnorm{W^0 - W^*}$.

To optimize this bound, we minimize it with respect to $\gamma$. The optimal stepsize $\gamma_*$ solves:
\begin{eqnarray}\label{eq:const_step_opt}
\gamma_* &=& \argmin_{\gamma > 0}\left(\frac{(R^0)^2}{2\gamma \alpha T} + \frac{\gamma L_0^2}{2}\right) \nonumber \\
&=& \sqrt{\frac{(R^0)^2}{T\alpha L_0^2}}.
\end{eqnarray}

Substituting $\gamma_*$ back into \eqref{eq:const_step_avg}, we obtain the optimal convergence rate:
\begin{eqnarray}\label{eq:const_step_final}
\Exp{f(\avg{W}^T) - f(W^*)} \leq \frac{R^0 L_0 }{\sqrt{\alpha T}}.
\end{eqnarray}

\textbf{2. (Polyak stepsize).} For this strategy, we choose the stepsize adaptively based on the current function value:
\begin{eqnarray}\label{eq:polyak_step_def}
\gamma_t &=& \argmin\limits_{\gamma>0} \cb{\sqfnorm{\wt - \ws} -2{\gamma}\alpha \rb{f(\wt) - f(\ws)} + {\gamma}^2\alpha \sqfnorm{\pf{\wt}}}\nonumber \\
&=& \frac{\rb{f(W^t) - f(W^*)}}{\sqfnorm{\partial f(W^t)}}.
\end{eqnarray}

Substituting this stepsize into inequality \eqref{eq:nonsmooth_res_exp_exp}:
\begin{eqnarray}\label{eq:polyak_step_subst}
\Exp{\sqfnorm{\wtpo - \ws} \mid \wt} &=&\Exp{\Exp{\sqfnorm{\wtpo - \ws} \mid \wt, H^t}\mid \wt} \nonumber \\ 
&\leq& \sqfnorm{\wt - \ws} -2{\gamma_t}\alpha\rb{f(\wt) - f(\ws)} + \gamma^2_t\alpha\sqfnorm{\pf{\wt}} \nonumber \\
&\eqtext{\eqref{eq:polyak_step_def}}& \sqfnorm{\wt - \ws} - \frac{\alpha\rb{f(W^t) - f(W^*)}^2}{\sqfnorm{\partial f(W^t)}}\nonumber \\
&\letext{\eqref{eq:L_0_bound}}& \sqfnorm{\wt - \ws} -  \frac{\alpha\rb{f(W^t) - f(W^*)}^2}{ L_0^2 }.
\end{eqnarray}

Taking expectation of both sides of \eqref{eq:polyak_step_subst} again and applying the tower property
\begin{eqnarray}\label{eq:polyak_step_subst_exp}
\Exp{\sqfnorm{\wtpo - \ws}} \le \Exp{\sqfnorm{\wt - \ws}} -  \frac{\alpha\Exp{\rb{f(W^t) - f(W^*)}^2} }{ L_0^2 }
\end{eqnarray}
Since $f$ is convex, by Jensen's inequality \eqref{eq:jensen_general} and the Cauchy-Bunyakovsky-Schwarz inequality \eqref{eq:cauchy_bunyakovsky_schwarz} with $X\eqdef  f(\wt) - f(\ws)$ and $Y\eqdef 1$, we have 
\begin{eqnarray}
\Exp{ f_i(\avg{W}^T) - f(\ws)} &\letext{\eqref{eq:jensen_general}}& \Exp{ \frac{1}{T} \sum_{t=0}^{T-1} f(\wt) - f(\ws)} \nonumber \\
&\le& \frac{1}{T}\sum_{t=0}^{T-1} \Exp{ f(\wt) - f(\ws)}\nonumber \\
&\letext{\eqref{eq:cauchy_bunyakovsky_schwarz}}& \frac{1}{T}\sum_{t=0}^{T-1} \sqrt{\Exp{\rb{ f(\wt) - f(\ws)}^2 }}\nonumber \\
&\le& \sqrt{\frac{1}{T}\sum_{t=0}^{T-1}\Exp{\rb{ f(\wt) - f(\ws)}^2 }}\nonumber \\
&\letext{\eqref{eq:polyak_step_subst_exp}}& \frac{R^0 L_0}{\sqrt{\alpha T}},
\end{eqnarray}
which matches the optimal rate achieved by the constant stepsize strategy with optimal tuning.
\end{proof}

\newpage
\subsection{Analysis of {Bernoulli-LoRA-SGD}}\label{apx:sgd}

\begin{algorithm}[H]
\caption{\algname{Bernoulli-LoRA-SGD}}\label{alg:Bernoulli-LoRA-SGD}
\begin{algorithmic}[1]
\STATE \textbf{Parameters:} pre-trained model $W^0 \in \mathbb{R}^{m \times n}$, rank $r \ll \min\{m,n\}$, scaling factor $\alpha > 0$, chain length $T$, sketch distribution $\mathcal{D}_S^B$ or $\mathcal{D}_S^A$, Bernoulli probability $p$

\FOR{$t = 0, 1, \ldots, T-1$}
    \STATE Sample $c^t \sim \text{Be}(p)$ \hfill{Bernoulli random variable}
    \IF{$c^t = 1$}
        \STATE Sample $B_S^t \sim \mathcal{D}_S^B$ \hfill{Left sketch}
        \STATE $\hat{A}^t = -\eta \rb{\rbtop{B_S^t}B_S^t}^{\dagger}\rbtop{B_S^t}g(W^t)$ 
        \STATE $W^{t+1} = W^t + \frac{\alpha}{r}B_S^t\hat{A}^t$
    \ELSE
        \STATE Sample $A_S^t \sim \mathcal{D}_S^A$ \hfill{Right sketch}
        \STATE $\hat{B}^t = -\eta g(W^t)\rbtop{A_S^t}\rb{A_S^t\rbtop{A_S^t}}^{\dagger}$
        \STATE $W^{t+1} = W^t + \frac{\alpha}{r}\hat{B}^tA_S^t$
    \ENDIF
\ENDFOR
\end{algorithmic}
\end{algorithm}

Earlier findings were derived utilizing full gradient computations. Nonetheless, this method proves impractical in deep learning applications, where obtaining full gradients is rarely feasible. Our focus moves to a framework that employs \algname{Stochastic Gradient Descent (SGD)} while incorporating a more flexible and generalized data sampling strategy, enabling greater adaptability in the selection and utilization of data throughout the training process. General sampling techniques for strongly convex functions have been thoroughly examined in~\citep{GowerSGD}. For broader convex optimization problems, \citet{KhaledUnified} provide a comprehensive study of how \algname{SGD} performs under different sampling strategies. In non-convex scenarios, the works of \citet{khaled2022better} and \citep{demidovich2023guide} investigate the effects of generalized sampling methods on \algname{SGD}’s convergence and efficiency, offering valuable insights into its adaptability for diverse machine learning applications. In this section we focus on \algname{Bernoulli-LoRA-SGD}, a method, designed in the scope of \algname{Bernoulli-LoRA} framework, based on the classical \algname{SGD} algorithm.

For convergence analysis, we notice the gradient step in Algorithm~\ref{alg:Bernoulli-LoRA-SGD} is equivalent to the following update
\begin{eqnarray}
\label{eq:sgd_simple}
    W^{t+1} &=& W^t - \gamma \hat{G}^t, \quad \text{where}\quad \hat{G}^t = \begin{cases}
        H^t_B G^t,& \text{with probability}~~ p\\
        G^t H^t_A,& \text{with probability}~~ 1-p
    \end{cases},
\end{eqnarray}
where $G^t  = g(W^t)$ is an unbiased stochastic gradient, which satisfies Assumption~\ref{as:ABC_assumption}. 

\subsubsection{Convergence for Smooth Non-Convex Functions}
\begin{theorem}
  \label{th:B_LoRA_SGD}  
  Let Assumptions~\ref{as:bounded_below},~\ref{as:lipschitz_smoothness}, and ~\ref{as:ABC_assumption} hold, and stepsize satisfy $$0 < \gamma \leq  \min\left\{\frac{1}{\sqrt{L A_1 \lambda^p_{\max} T}}, \frac{1}{LB_1}\left(\frac{\lambda^p_{\max}}{\lambda^p_{\min}}\right)^{-1}\right\}.$$ Then iterates generated by \algname{Bernoulli-LoRA-SGD} (Algorithm~\ref{alg:Bernoulli-LoRA-SGD}) satisfy
  \begin{equation*}
       \Exp{\sqfnorm{\nabla f(\widetilde{W}^T)}} \leq \frac{6(f(W^0) - f^*)}{\gamma \lambda^p_{\min} T} + \gamma LC_1 \frac{\lambda^p_{\max}}{\lambda^p_{\min}}, 
  \end{equation*}
  where $\lambda^p_{\min} \eqdef p\lambda^{H_B}_{\min} + (1-p)\lambda^{H_A}_{\min}$, $\lambda^p_{\max} \eqdef p\lambda^{H_B}_{\max} + (1-p)\lambda^{H_A}_{\max}$, and $\widetilde{W}^T$ is chosen at random from $\left\{W^0, W^1,\dots, W^{T-1}\right\}$ with probabilities $\{\frac{w_t}{\mathcal{W}_{T-1}}\}_{t = 0}^{T-1}$, where $w_t = \frac{w_{t-1}}{(1+\gamma^2LA_1\lambda^p_{\max})}$, $\mathcal{W}_{T-1} = \sum^{T-1}_{t=0}w_t$, and $w^{-1} > 0$. 
\end{theorem}
\begin{proof}
    We start with smoothness of function $f$:
    \begin{eqnarray}
        f(W^{t+1}) &\leq& f(W^t) + \la \nabla f(W^t) , W^{t+1} - W^t\ra +\frac{L}{2}\sqfnorm{W^{t+1}-W^t} \notag\\
        &\overset{\eqref{eq:sgd_simple}}{=}& f(W^t) -\gamma \la \nabla f(W^t), \hat{G}^t\ra +\frac{\gamma^2 L}{2}\sqfnorm{\hat{G}^t}. \label{eq:sankasnkjs}
    \end{eqnarray}
    Taking a conditional expectation by $W^t$, we bound the second and the third terms from inequality \eqref{eq:sankasnkjs}:
    \begin{eqnarray}
        \Exp{\la \nabla f(W^t), \hat{G}^t\ra|W^t} &=& \la \nabla f(W^t), \Exp{\hat{G}^t|W^t}\ra \notag\\
        &\overset{\eqref{eq:sgd_simple}}{=}& p \la \nabla f(W^t), \Exp{H^t_BG^t|W^t}\ra + (1-p)\la \nabla f(W^t), \Exp{G^t H^t_A|W^t}\ra \notag\\
        &\overset{(\ast)}{=}& p \la \nabla f(W^t), \Exp{H^t_B|W^t}\Exp{G^t|W^t}\ra + (1-p)\la \nabla f(W^t), \Exp{G^t|W^t}\Exp{ H^t_A|W^t}\ra \notag\\
        &=&  p \la \nabla f(W^t), \Exp{H^t_B|W^t}\nabla f(W^t)\ra + (1-p)\la \nabla f(W^t), \nabla f(W^t)\Exp{ H^t_A|W^t}\ra \notag\\
        &\geq& \underbrace{\left(p \lambda_{\min}(\Exp{H^t_B}) + (1-p)\lambda_{\min}(\Exp{H^t_A})\right)}_{\eqdef \lambda_{\min}^p} \sqfnorm{\nabla f(W^t)} \notag\\
        &=& \lambda_{\min}^p \sqfnorm{\nabla f(W^t)}, \label{eq:hbfjhbf}
    \end{eqnarray}
    where in $(\ast)$ we used that $H^t_B$, $H^t_A$ and $G^t$ are independent. Now we bound the third term:
    \begin{eqnarray*}
        \Exp{\sqfnorm{\hat{G}^t}|W^t} &\overset{\eqref{eq:sgd_simple}}{=}& p \Exp{\sqfnorm{H^t_BG^t}|W^t} + (1-p)\Exp{\sqfnorm{G^t H^t_A}|W^t} \\
        &=& p \Exp{\la H^t_BG^t, H^t_BG^t\ra|W^t} + (1-p)\Exp{\la G^t H^t_A,  G^t H^t_A\ra|W^t}\\
        &\overset{(\ast\ast)}{=}& p \Exp{\la G^t, H^t_BG^t\ra|W^t} + (1-p)\Exp{\la G^t ,  G^t H^t_A\ra|W^t},
    \end{eqnarray*}
    where in $(\ast\ast)$ we used property of projection matrices $H^t_B, H^t_B$. By the independence of $H^t_B, H^t_A, G^t$, we obtain
    \begin{eqnarray}
        \Exp{\sqfnorm{\hat{G}^t}|W^t} 
        &=& p \Exp{\la G^t, \Exp{H^t_B|W^t}G^t\ra|W^t} + (1-p)\Exp{\la G^t ,  G^t \Exp{H^t_A|W^t}\ra|W^t} \notag\\
        &\leq& p\lambda_{\max}(\Exp{H^t_B|W^t}) \Exp{\sqfnorm{G^t}|W^t} + (1-p) \lambda_{\max}(\Exp{H^t_A|W^t}) \Exp{\sqfnorm{G^t}|W^t}\notag\\
        &=&  \underbrace{(p\lambda_{\max}(\Exp{H^t_B|W^t})  + (1-p) \lambda_{\max}(\Exp{H^t_A|W^t}))}_{\eqdef \lambda_{\max}^p} \Exp{\sqfnorm{G^t}|W^t}\notag\\
        &=& \lambda_{\max}^p\Exp{\sqfnorm{G^t}|W^t}. \label{eq:iweipwewp}
    \end{eqnarray}
    Plugging \eqref{eq:hbfjhbf} and \eqref{eq:iweipwewp} into \eqref{eq:sankasnkjs}, we obtain 
    \begin{eqnarray*}
        \Exp{f(W^{t+1})|W^t} &\leq& f(W^t) -\gamma \Exp{\la \nabla f(W^t), \hat{G}^t\ra|W^t} +\frac{\gamma^2 L}{2}\Exp{\sqfnorm{\hat{G}^t}|W^t}\\
        &\leq&  f(W^t) - \gamma \lambda_{\min}^p \sqfnorm{\nabla f(W^t)} + \frac{\gamma^2 \lambda_{\max}^p L}{2} \Exp{\sqfnorm{G^t}|W^t}.
    \end{eqnarray*}
    By Assumption~\ref{as:ABC_assumption}, 
    \begin{eqnarray*}
        \Exp{f(W^{t+1}) - f^*|W^t} &\leq& f(W^t) -\gamma \Exp{\la \nabla f(W^t), \hat{G}^t\ra|W^t} +\frac{\gamma^2 L}{2}\Exp{\sqfnorm{\hat{G}^t}|W^t}\\
        &\leq&  f(W^t) - f^* - \gamma \lambda_{\min}^p \sqfnorm{\nabla f(W^t)} \\
        && + \frac{\gamma^2 \lambda_{\max}^p L}{2} \left(2A_1(f(W^t)- f^*)+ B_1\sqfnorm{\nabla f(W^t)} + C_1\right)\\
        &\leq& \left(1 +\gamma^2 \lambda_{\max}^p LA_1\right)\left( f(W^t) - f^*\right) - \gamma \lambda_{\min}^p \left(1-\frac{\gamma LB_1 \lambda_{\max}^p}{2\lambda^p_{\min}}\right) \sqfnorm{\nabla f(W^t)}\\
        && + \frac{\gamma^2 \lambda_{\max}^p L C_1}{2}.
    \end{eqnarray*}
    Taking mathematical expectation and selecting a stepsize as $0 <\gamma \leq \frac{1}{LB_1 }\left( \frac{\lambda^p_{\max}}{\lambda^p_{\min}}\right)^{-1}$, we get 
    \small
    \begin{eqnarray}\label{eq:useful_inequality_from_sgd_proof}
        \Exp{f(W^{t+1}) - f^*} &\leq& \left(1 +\gamma^2 \lambda_{\max}^p LA_1\right)\Exp{f(W^t) - f^*}\notag\\
        &-& \frac{\gamma \lambda_{\min}^p}{2} \Exp{\sqfnorm{\nabla f(W^t)}} + \frac{\gamma^2 \lambda_{\max}^p L C_1}{2}. 
    \end{eqnarray}
    \normalsize
    Defining  $\delta^t \eqdef \Exp{f(W^t) - f^*}$, $r^t \eqdef \Exp{\sqfnorm{\nabla f(W^t)}}$ for every $t \geq 0$, we have 
    \begin{eqnarray*}
        \delta^{t+1} &\leq& \left(1 +\gamma^2 \lambda_{\max}^p LA_1\right)\delta^t - \frac{\gamma \lambda_{\min}^p}{2} r^t + \frac{\gamma^2 \lambda_{\max}^p L C_1}{2}. 
    \end{eqnarray*}
    Fixing $w^{-1} >0$ and defining $w_{t} = \frac{w_{t-1}}{1+\gamma^2LA_1\lambda_{\max}^p}$ for all $t \geq 0$, we have 
    \begin{eqnarray*}
        \frac{1}{2}\lambda^p_{\min} w_t r^t &\leq& \frac{w_t}{\gamma}\left(1+\gamma^2\lambda^p_{\max}LA_1\right)\delta^t - \frac{w_{t}}{\gamma}\delta^{t+1} + \frac{1}{2}\gamma LC_1 \lambda^p_{\max} w_t\\
        &=&  \frac{w_{t-1}\delta^t}{\gamma} - \frac{w_{t}\delta^{t+1}}{\gamma} + \frac{1}{2}\gamma LC_1 \lambda^p_{\max} w_t.
    \end{eqnarray*}
    Summing over $t$ from $0$ to $T-1$, we have 
    \begin{eqnarray*}
        \sum^{T-1}_{t=0} w_t r^t &\leq& \frac{2w_{-1}\delta^0}{\gamma\lambda^p_{\min}} -\frac{2w_{T-1}\delta^{T}}{\gamma\lambda^p_{\min}} + \gamma L C_1 \frac{\lambda^p_{\max}}{\lambda^p_{\min}} \sum^{T-1}_{t=0} w_t.
    \end{eqnarray*}
    Defining $\mathcal{W}_{T-1} = \sum^{T-1}_{t=0}w_t $, we acquire 
    \begin{eqnarray*}
        \sum^{T-1}_{t=0}\frac{w_t}{\mathcal{W}^{T-1}} r^t &\leq& \frac{2w_{-1}\delta^0}{\gamma \lambda^p_{\min} \mathcal{W}_{T-1}} + \gamma LC_1 \frac{\lambda^p_{\max}}{\lambda^p_{\min}}. 
    \end{eqnarray*}
    Using the next chain of inequalities
    \begin{equation*}
        W_{T-1} = \sum^{T-1}_{t=0} w_t \geq T \min_{0 \leq t \leq T-1} w_t  = T w_{T-1} = \frac{T w_{-1}}{(1+\gamma^2 \lambda^p_{\max} LA_1)^T},
    \end{equation*}
    we have 
    \begin{eqnarray*}
        \sum^{T-1}_{t=0}\frac{w_t}{\mathcal{W}^{T-1}} r^t &\leq& \frac{2(1+\gamma^2 \lambda^p_{\max} LA_1)^T}{\gamma T \lambda^p_{\min}} (f(W^0) - f^*) + \gamma LC_1 \frac{\lambda^p_{\max}}{\lambda^p_{\min}}.
    \end{eqnarray*}
    Selecting $0 < \gamma \leq \frac{1}{\sqrt{LA_1 \lambda^p_{\max} T}}$, and using $(1+\gamma^2 \lambda^p_{\max} LA_1)^T \leq \exp\left(\gamma^2 \lambda^p_{\max} L A_1 T\right) \leq \exp{(1)} \leq 3$, we obtain  
    \begin{eqnarray*}
        \sum^{T-1}_{t=0}\frac{w_t}{\mathcal{W}^{T-1}} r^t &\leq& \frac{6\delta^0}{\gamma T \lambda^p_{\min}}  + \gamma LC_1 \frac{\lambda^p_{\max}}{\lambda^p_{\min}}. 
    \end{eqnarray*}
\end{proof}

Next we show convergence of \algname{Bernoulli-LoRA-SGD} under additional Assumption~\ref{as:pl_condition}.  
\subsubsection{Convergence under Polyak-{\L}ojasiewicz Condition}
\begin{theorem}
\label{th:B_LORA_SGD_PL}
Let Assumptions~\ref{as:bounded_below},~\ref{as:lipschitz_smoothness}, ~\ref{as:ABC_assumption}, and \ref{as:pl_condition} hold, and stepsize satisfy 

$0 < \gamma \leq  \min\left\{\frac{\mu\lambda^p_{\min}}{2L A_1\lambda^p_{\max}}, \frac{2}{\mu \lambda^p_{\min}}, \frac{1}{LB_1}\left(\frac{\lambda^p_{\max}}{\lambda^p_{\min}}\right)^{-1}\right\}$. Then iterates generated by \algname{Bernoulli-LoRA-SGD} (Algorithm~\ref{alg:Bernoulli-LoRA-SGD}) satisfy
  \begin{equation*}
      \Exp{f(W^T) - f^*} \leq \left(1- \frac{1}{2}\gamma \mu \lambda^p_{\min}\right)^{T} \left(f(W^0) - f^*\right) + \frac{\gamma LC_1}{\mu }\cdot \frac{\lambda^p_{\max}}{\lambda^p_{\min}}, 
  \end{equation*}
  where $\lambda^p_{\min} \eqdef p\lambda^{H_B}_{\min} + (1-p)\lambda^{H_A}_{\min}$, $\lambda^p_{\max} \eqdef p\lambda^{H_B}_{\max} + (1-p)\lambda^{H_A}_{\max}$.     
\end{theorem}
\begin{proof}
    We start our proof with inequality~\ref{eq:useful_inequality_from_sgd_proof}. Using PL-inequality (see Assumption~\ref{as:pl_condition}), we have 
    \begin{eqnarray*}
        \Exp{f(W^{t+1}) - f^*} &\leq& \left(1 +\gamma^2 \lambda_{\max}^p LA_1\right)\Exp{f(W^t) - f^*} - \frac{\gamma \lambda_{\min}^p}{2} \Exp{\sqfnorm{\nabla f(W^t)}} + \frac{\gamma^2 \lambda_{\max}^p L C_1}{2}\\
        &\leq& \left(1  - \gamma \mu \lambda^p_{\min}+\gamma^2 \lambda_{\max}^p LA_1\right)\Exp{f(W^t) - f^*} + + \frac{\gamma^2 \lambda_{\max}^p L C_1}{2}.
    \end{eqnarray*}
    Taking the stepsize as $0 < \gamma \leq \min\left\{\frac{\mu\lambda^p_{\min}}{2L A_1\lambda^p_{\max}}, \frac{2}{\mu \lambda^p_{\min}}\right\}$, we obtain
    \begin{eqnarray*}
        \Exp{f(W^{t+1}) - f^*} &\leq& \left(1- \frac{1}{2}\gamma \mu \lambda^p_{\min}\right) \Exp{f(W^t) - f^*} +\frac{\gamma^2 \lambda^p_{\max}LC_1}{2}\\
        &\leq& \left(1- \frac{1}{2}\gamma \mu \lambda^p_{\min}\right)^{t+1} \Exp{f(W^0) - f^*} + \frac{\gamma^2 \lambda^p_{\max}LC_1}{2} \sum^{t}_{\tau = 0}\left(1- \frac{1}{2}\gamma \mu \lambda^p_{\min}\right)^{t-\tau}\\
        &\leq& \left(1- \frac{1}{2}\gamma \mu \lambda^p_{\min}\right)^{t+1} \Exp{f(W^0) - f^*} + \frac{\gamma^2 \lambda^p_{\max}LC_1}{2} \sum^{\infty}_{\tau = 0}\left(1- \frac{1}{2}\gamma \mu \lambda^p_{\min}\right)^{\tau}\\
        &=& \left(1- \frac{1}{2}\gamma \mu \lambda^p_{\min}\right)^{t+1} \Exp{f(W^0) - f^*} + \frac{\gamma^2 \lambda^p_{\max}LC_1}{\gamma \mu \lambda^p_{\min}},
    \end{eqnarray*}
    where in the last equation we use the formula of the sum of  geometric progression. 
\end{proof}

\newpage

\subsection{Analysis of {Bernoulli-LoRA-MVR}}\label{apx:B-LoRA-MVR}

\begin{algorithm}[H]
\caption{\algname{Bernoulli-LoRA-MVR}}\label{alg:Bernoulli-LoRA-MVR}
\begin{algorithmic}[1]
\STATE \textbf{Parameters:} pre-trained model $W^0 \in \mathbb{R}^{m \times n}$, $G^0 \in \mathbb{R}^{m \times n}$ rank $r \ll \min\{m,n\}$, scaling factor $\alpha > 0$, chain length $T$, sketch distribution $\mathcal{D}_S^B$ or $\mathcal{D}_S^A$, Bernoulli probability $p$, momentum parameter $b \in [0,1]$

\FOR{$t = 0, 1, \ldots, T-1$}
    \STATE Sample $c^t \sim \text{Be}(p)$ \hfill{Bernoulli random variable}
    \IF{$c^t = 1$}
        \STATE Sample $B_S^t \sim \mathcal{D}_S^B$ \hfill{Left sketch}
        \STATE $\hat{A}^t = -\eta \rb{\rbtop{B_S^t}B_S^t}^{\dagger}\rbtop{B_S^t} G^t$ 
        \STATE $W^{t+1} = W^t + \frac{\alpha}{r}B_S^t\hat{A}^t$
    \ELSE
        \STATE Sample $A_S^t \sim \mathcal{D}_S^A$ \hfill{Right sketch}
        \STATE $\hat{B}^t = -\eta G^t\rbtop{A_S^t}\rb{A_S^t\rbtop{A_S^t}}^{\dagger}$
        \STATE $W^{t+1} = W^t + \frac{\alpha}{r}\hat{B}^tA_S^t$
    \ENDIF
    \STATE Sample $\xi^{t+1} \sim \cD$
    \STATE $G^{t+1} = \nabla f_{\xi^{t+1}}(W^{t+1}) + (1-b)\left(G^t -\nabla f_{\xi^{t+1}}(W^t)\right)$
\ENDFOR
\end{algorithmic}
\end{algorithm}

Recently, there has been a significant surge of interest in variance-reduced methods for addressing finite-sum problems~\citep{ReddiVR,Shang2018VRSGDAS,malinovsky2022variance,richtarik2024unified}. It has gained prominence as a formidable alternative to stochastic gradient descent (SGD) in tackling non-convex optimization problems. Notably, it has been pivotal in introducing the first algorithms capable of surpassing \algname{SGD}’s convergence rate for locating first-order critical points. Despite these advancements, variance reduction methods often come with challenges, including the necessity for meticulously tuned learning rates and the reliance on overly large batch sizes to realize their benefits. To address some of these limitations, Momentum Variance Reduction (MVR) was proposed specifically for server-only stochastic non-convex optimization~\citep{STORM}. This approach leverages a modified form of momentum to achieve variance reduction while eliminating the dependence on large batch sizes. A proof on MVR technique with better dependence on momentum parameter was obtained by \citet{tyurin2023dasha}. In the context of Federated Learning, \citet{karagulyan2024spamstochasticproximalpoint} proposed the SPAM method. On the server side, MVR is utilized to enhance optimization efficiency, while the client side incorporates the Stochastic Proximal Point Method updates. This section is devoted to \algname{Bernoulli-LoRA-MVR}, a method, designed in the scope of \algname{Bernoulli-LoRA} framework, based on the MVR technique.



To show convergence guarantees for \algname{Bernoulli-LoRA-MVR}, the iterates of the method can be rewritten in following way 
\begin{eqnarray}
    W^{t+1} &=& W^t - \gamma \hat{G}^t, \quad \text{where}\quad \hat{G}^t = \begin{cases}
        H^t_B G^t,& \text{with probability}~~ p\\
        G^t H^t_A,& \text{with probability}~~ 1-p
    \end{cases} \label{eq:general_page_update}\\
    G^{t+1} &=& \nabla f_{\xi^{t+1}}(W^{t+1}) + (1-b)\left(G^t - \nabla f_{\xi^{t+1}}(W^{t})\right). \label{eq:est_mvr}
\end{eqnarray}

First of all, we reprove descent lemma from the paper of \cite{Li2020PAGE} for generic gradient step \eqref{eq:general_page_update}. 
\begin{lemma}
\label{lem:descent_lemma_page}
    Let Assumptions~\ref{as:projection_matrix}, ~\ref{as:lipschitz_smoothness} hold. Then, iterates defined as \eqref{eq:general_page_update} satisfy 
    \begin{eqnarray*}
        \Exp{f(W^{t+1}) - f^*~|~W^t} &\leq& f(W^t)- f^* -\frac{\gamma\lambda^p_{\min}}{2}\sqfnorm{\nabla f(W^t)}\\
        && +\frac{\gamma \lambda^p_{\max}}{2}\sqfnorm{G^t - \nabla f(W^t)} - \left(\frac{1}{2\gamma} -\frac{L}{2}\right)\Exp{\sqfnorm{W^{t+1}-W^t}~|~W^t}. 
    \end{eqnarray*}
\end{lemma}
\begin{proof}
    By Assumption~\ref{as:lipschitz_smoothness}, we have 
    \begin{eqnarray}
        f(W^{t+1}) &\leq& f(W^t) +\la \nabla f(W^t), W^{t+1}-W^t\ra_{F} + \frac{L}{2}\sqfnorm{W^{t+1}-W^t} \notag\\
        &=& f(W^t) -\gamma\la \nabla f(W^t), \hat{G}^t\ra_{F} + \frac{L}{2}\sqfnorm{W^{t+1}-W^t}. \label{eq:wdnieocnoiewo}
    \end{eqnarray}
To continue our proof, we need to bound the second term from \eqref{eq:wdnieocnoiewo}. 
Taking conditional expectation by $H^t, W^t$, we obtain
\begin{eqnarray*}
    \Exp{\la \nabla f(W^t), \hat{G}^t\ra_{F}~|~ H^t, W^t} &\overset{\eqref{eq:general_page_update}}{=}& p\la \nabla f(W^t), H^t_B G^t\ra_{F} + (1-p)\la \nabla f(W^t), G^t H^t_A\ra_{F}\\
    &=& p\la  H^t_B\nabla f(W^t), H^t_B G^t\ra_{F} + (1-p)\la \nabla f(W^t)H^t_A, G^t H^t_A\ra_{F}\\
    &=& \frac{p}{2}\left(\sqfnorm{H^t_B\nabla f(W^t)} + \sqfnorm{H^t_B G^t} - \sqfnorm{H^t_B G^t - H^t_B\nabla f(W^t)}\right)\\
    && + \frac{1-p}{2}\left(\sqfnorm{\nabla f(W^t) H^t_A} + \sqfnorm{ G^t H^t_A} - \sqfnorm{ G^t H^t_A -\nabla f(W^t)H^t_A}\right)\\
    &\geq& \frac{1}{2}\left( p\sqfnorm{H^t_B\nabla f(W^t)} + (1-p)\sqfnorm{\nabla f(W^t) H^t_A} \right) + \frac{1}{2}\Exp{\sqfnorm{\hat{G}^t}~|~ H^t, W^t}\\
    && -\frac{1}{2}\left(p \sqfnorm{H^t_B G^t - H^t_B\nabla f(W^t)} + (1-p)\sqfnorm{ G^t H^t_A -\nabla f(W^t)H^t_A}\right).
\end{eqnarray*}
Taking conditional expectation by $W^t$, we have 
\small
\begin{eqnarray} 
    \Exp{\la \nabla f(W^t), \hat{G}^t\ra_{F}~| W^t} &\geq& \frac{1}{2}\left( p\Exp{\sqfnorm{H^t_B\nabla f(W^t)}~| W^t} + (1-p)\Exp{\sqfnorm{\nabla f(W^t) H^t_A}~| W^t} \right) + \frac{1}{2}\Exp{\sqfnorm{\hat{G}^t}~|~W^t} \notag\\
    && -\frac{1}{2}\left(p \Exp{\sqfnorm{H^t_B G^t - H^t_B\nabla f(W^t)}~| W^t} + (1-p)\Exp{\sqfnorm{ G^t H^t_A -\nabla f(W^t)H^t_A}~| W^t}\right)\notag\\
    &\overset{(\ast)}{\geq}&\frac{1}{2}\underbrace{\left(p\lambda_{\min}(\Exp{H^t_B})+ (1-p)\lambda_{\min}(\Exp{H^t_A})\right)}_{\eqdef \lambda^p_{\min}}\sqfnorm{\nabla f(W^t)} +\frac{1}{2}\Exp{\sqfnorm{\hat{G}^t}~|~W^t}\notag\\
    &&-\frac{1}{2}\underbrace{(p\lambda_{\max}(\Exp{H^t_B}) + (1-p)\lambda_{\max}(\Exp{H^t_A}))}_{\eqdef \lambda^p_{\max}}\sqfnorm{G^t - \nabla f(W^t)} \notag\\
    &\overset{\eqref{eq:general_page_update}}{=}& \frac{\lambda^p_{\min}}{2}\sqfnorm{\nabla f(W^t)} + \frac{1}{2\gamma^2} \Exp{\sqfnorm{W^{t+1} - W^t}~|~W^t} -\frac{\lambda^p_{\max}}{2}\sqfnorm{G^t - \nabla f(W^t)}, \label{eq:oewklmcfdkl}
\end{eqnarray}
\normalsize
where in $(\ast)$ we used the following inequalities for any matrix $V \in \R^{m\times n} $
\begin{eqnarray*}
    \Exp{\sqfnorm{H^t_B V}} &=& \Exp{\la H^t_B V, H^t_B V\ra_F} = \la \Exp{H^t_B} V, V\ra_F  \geq \lambda_{\min}\left(\Exp{H^t_B}\right) \sqfnorm{V},\\
    \Exp{\sqfnorm{H^t_B V}} &\leq& \lambda_{\max}\left(\Exp{H^t_B}\right) \sqfnorm{V},\\
    \Exp{\sqfnorm{ V H^t_A}} &=& \Exp{\la VH^t_A, V H^t_A\ra_F} = \la  V \Exp{H^t_A}, V\ra_F  \geq \lambda_{\min}\left(\Exp{H^t_A}\right) \sqfnorm{V},\\
    \Exp{\sqfnorm{ VH^t_A}} &\leq& \lambda_{\max}\left(\Exp{H^t_A}\right) \sqfnorm{V}.
\end{eqnarray*}
Plugging in \eqref{eq:oewklmcfdkl} into \eqref{eq:wdnieocnoiewo}, we get 
\begin{eqnarray*}
    \Exp{f(W^{t+1})~|~W^t} &\leq&  f(W^t) -\frac{\gamma\lambda^p_{\min}}{2}\sqfnorm{\nabla f(W^t)} -\frac{1}{2\gamma}\Exp{\sqfnorm{W^{t+1}-W^t}~|~W^t} \\
    &&+\frac{\gamma \lambda^p_{\max}}{2}\sqfnorm{G^t - \nabla f(W^t)} + \frac{L}{2}\Exp{\sqfnorm{W^{t+1}-W^t}~|~W^t}.
\end{eqnarray*}
\end{proof}

\begin{lemma}
\label{lem:aux_lem_mvr}
    Let Assumptions~\ref{as:lipschitz_smoothness},~\ref{as:bounded_variance} hold. Then, iterates generated by \algname{Bernoulli-LoRA-MVR} (Algorithm~\ref{alg:Bernoulli-LoRA-MVR}) satisfy
    \begin{equation}
         \Exp{\sqfnorm{G^{t+1} - \nabla f(W^{t+1})}} \leq (1-b)^2\Exp{\sqfnorm{G^{t} - \nabla f(W^{t}) }} +  2(1-b)^2L^2\Exp{\sqfnorm{W^{t+1} - W^t}} + 2b^2\sigma^2
    \end{equation}
\end{lemma}
\begin{proof}
    Taking conditional expectation by $\cF^{t+1} = \{W^{t+1}, G^t\}$, we obtain 
    \small
    \begin{eqnarray*}
        \Exp{\sqfnorm{G^{t+1} - \nabla f(W^{t+1})}|\cF^{t+1}} &\overset{\eqref{eq:est_mvr}}{=}& \Exp{\sqfnorm{\nabla f_{\xi^{t+1}}(W^{t+1})  - \nabla f(W^{t+1}) + (1-b)\left(G^t - \nabla f_{\xi^{t+1}}(W^{t})\right) }|\cF^{t+1}}\\
        &\overset{\eqref{eq:bvd}}{=}& (1-b)^2\sqfnorm{G^t - \nabla f(W^t)} \\
        && + \Exp{\sqfnorm{\nabla f_{\xi^{t+1}}(W^{t+1})  - \nabla f(W^{t+1}) + (1-b)\left(\nabla f(W^t) - \nabla f_{\xi^{t+1}}(W^{t})\right) }|\cF^{t+1}}\\
        &\leq& (1-b)^2\sqfnorm{G^t - \nabla f(W^t)} + 2b^2\Exp{\sqfnorm{\nabla f_{\xi^{t+1}}(W^{t+1})  - \nabla f(W^{t+1}) }|\cF^{t+1}}\\
        && + 2(1-b)^2 \Exp{\sqfnorm{\nabla f_{\xi^{t+1}}(W^{t+1}) - \nabla f_{\xi^{t+1}}(W^{t})  - \nabla f(W^{t+1}) + \nabla f(W^t) }|\cF^{t+1}}\\
        &\leq& (1-b)^2\sqfnorm{G^t - \nabla f(W^t)} + 2b^2\Exp{\sqfnorm{\nabla f_{\xi^{t+1}}(W^{t+1})  - \nabla f(W^{t+1}) }|\cF^{t+1}}\\
        && + 2(1-b)^2 \Exp{\sqfnorm{\nabla f_{\xi^{t+1}}(W^{t+1}) - \nabla f_{\xi^{t+1}}(W^{t}) }|\cF^{t+1}}\\
        &\leq& (1-b)^2\sqfnorm{G^t - \nabla f(W^t)} + 2(1-b)^2L^2\sqfnorm{W^{t+1}-W^t} + 2b^2\sigma^2,
    \end{eqnarray*}
    \normalsize
    where in the last inequality we used smoothness of $f_{\xi}$ and bounded variance assumption. 
    Taking math expectation, we conclude the proof. 
\end{proof}

\subsubsection{Convergence for Smooth Non-Convex Functions}
\begin{theorem}
    Let Assumptions~\ref{as:projection_matrix},~\ref{as:bounded_below}, ~\ref{as:lipschitz_smoothness}, and ~\ref{as:bounded_variance}  hold, and let the  stepsize satisfy $
        0<\gamma \leq \frac{1}{L\left(1+\sqrt{\frac{2 \lambda^p_{\max}(1-b)^2}{b}}\right)}$. Then the iterates of \algname{Bernoulli-LoRA-MVR} (Algorithm~\ref{alg:Bernoulli-LoRA-MVR}) satisfy
        \begin{equation}
            \Exp{\sqfnorm{\nabla f(\widetilde{W}^T)}} \leq \frac{2(f(W^{0}) - f^*)}{\lambda^p_{\min} \gamma T} + \frac{\sqfnorm{G^{0} - \nabla f(W^{0})}}{b(2-b) T}\cdot\frac{\lambda^p_{\max}}{\lambda^p_{\min}} +  \frac{2b\sigma^2}{2-b}\cdot\frac{\lambda^p_{\max}}{\lambda^p_{\min}} ,
        \end{equation}
        where $\lambda_{\min}^{p} := p\lambda_{\min}^{H_B} + (1-p)\lambda_{\min}^{H_A}$, $\lambda_{\max}^{p} := p\lambda_{\max}^{H_B} + (1-p)\lambda_{\max}^{H_A}$, $\widetilde{W}^T$ is drawn uniformly at random from the iterate sequence $\{W^0, W^1, \ldots, W^{T-1}\}$.
\end{theorem}
\begin{proof}
    Denote Lyapunov function $\Phi_t$ as follows 
    \begin{equation}
        \Phi_t = f(W^{t}) - f^* + \frac{\gamma\lambda^p_{\max}}{2b(2-b)}\sqfnorm{G^{t} - \nabla f(W^{t})}.
    \end{equation}
    By Lemma~\ref{lem:descent_lemma_page} and Lemma~\ref{lem:aux_lem_mvr}, we have 
    \begin{eqnarray*}
        \Exp{\Phi_{t+1}} &\leq& \Exp{f(W^t)} - f^* -\frac{\gamma\lambda^p_{\min}}{2}\Exp{\sqfnorm{\nabla f(W^t)}}  - \left(\frac{1}{2\gamma} -\frac{L}{2}\right)\Exp{\sqfnorm{W^{t+1}-W^t}}\\
        &&  +\frac{\gamma\lambda^p_{\max}}{2}\Exp{\sqfnorm{G^t - \nabla f(W^t)}}+ \frac{\gamma(1-b)^2\lambda^p_{\max}}{2b(2-b)}\Exp{\sqfnorm{G^{t} - \nabla f(W^{t}) }} \\
        && +  \frac{\gamma(1-b)^2L^2 \lambda^p_{\max}}{2b(2-b)}\Exp{\sqfnorm{W^{t+1} - W^t}} + \frac{\gamma \lambda^p_{\max} b \sigma^2}{2-b}\\
        &\leq& \Exp{\Phi_t} - \frac{\gamma\lambda^p_{\min}}{2}\Exp{\sqfnorm{\nabla f(W^t)}}  + \frac{\gamma \lambda^p_{\max}b \sigma^2}{2-b}\\
        &&  - \left(\frac{1}{2\gamma} -\frac{L}{2} - \frac{\gamma(1-b)^2L^2 \lambda^p_{\max}}{2b(2-b)}\right)\Exp{\sqfnorm{W^{t+1}-W^t}}.
    \end{eqnarray*}
    Selecting $0<\gamma \leq \frac{1}{L\left(1+\sqrt{\frac{ (1-b)^2}{b(2-b)}\lambda^p_{\max}}\right)}$, we obtain
    \begin{eqnarray*}
        \Exp{\Phi_{t+1}} &\leq& \Exp{\Phi_{t}} - \frac{\gamma\lambda^p_{\min}}{2}\Exp{\sqfnorm{\nabla f(W^t)}} + \frac{\gamma \lambda^p_{\max}b \sigma^2}{2-b}.
    \end{eqnarray*}
    Summing over $t$ from $0$ to $T-1$, we get 
    \begin{eqnarray*}
        \frac{\gamma\lambda^p_{\min}}{2}\sum^{T-1}_{t = 0}\Exp{\sqfnorm{\nabla f(W^t)}} &\leq& \Exp{\Phi_{0}} - \Exp{\Phi_{T}} + \frac{\gamma \lambda^p_{\max}b \sigma^2}{2-b} T.
    \end{eqnarray*}
    Finally, dividing both sides by $\frac{\gamma\lambda^p_{\min}}{2}$ yields 
    \begin{eqnarray*}
        \Exp{\sqfnorm{\nabla f(\widetilde{W}^T)}} &\leq& \frac{2\Phi_{0}}{\lambda^p_{\min} \gamma T} + \frac{2b\sigma^2}{2-b}\cdot\frac{\lambda^p_{\max}}{\lambda^p_{\min}},
    \end{eqnarray*}
    where $\widetilde{W}^T$ is drawn uniformly at random from the iterate sequence $\{W^0, W^1, \ldots, W^{T-1}\}$.
\end{proof}

Next we show convergence guarantee for \algname{Bernoulli-LoRA-MVR}, supposing additionally Assumption~\ref{as:pl_condition} holds. 

\subsubsection{Convergence under Polyak-{\L}ojasiewicz Condition}
\begin{theorem}\label{th:B-LORA-MVR-PL}
    Let Assumptions~\ref{as:projection_matrix},~\ref{as:bounded_below},~\ref{as:lipschitz_smoothness},~\ref{as:bounded_variance}, and ~\ref{as:pl_condition}  hold, and let the stepsize satisfy $$ 0<\gamma \leq \min\left\{\frac{1}{L\left(1+\sqrt{\frac{2(1-b)^2}{b(2-b)} \lambda^p_{\max}}\right)}, \frac{b}{2\mu\lambda^p_{\min}}\right\}.$$ Then the iterates of \algname{Bernoulli-LoRA-MVR} (Algorithm~\ref{alg:Bernoulli-LoRA-MVR}) satisfy
        \begin{equation}
             \Exp{f(W^T) - f^*} \leq \left(1-\gamma\mu\lambda^p_{\min}\right)^{T}\Phi_{0} + \frac{b \sigma^2}{(2-b) \mu}\cdot \frac{\lambda^p_{\max}}{\lambda^p_{\min}},
        \end{equation}
        where $\lambda_{\min}^{p} := p\lambda_{\min}^{H_B} + (1-p)\lambda_{\min}^{H_A}$, $\lambda_{\max}^{p} := p\lambda_{\max}^{H_B} + (1-p)\lambda_{\max}^{H_A}$, and $\Phi_0 = f(W^{0}) - f^* + \frac{\gamma\lambda^p_{\max}}{b(2-b)}\sqfnorm{G^{0} - \nabla f(W^{0})}$.
\end{theorem}
\begin{proof}
    Denote Lyapunov function $\Phi_t$ as follows 
    \begin{equation}
        \Phi_t = f(W^{t}) - f^* + \frac{\gamma\lambda^p_{\max}}{b(2-b)}\sqfnorm{G^{t} - \nabla f(W^{t})}.
    \end{equation}
    By Lemma~\ref{lem:descent_lemma_page} and Lemma~\ref{lem:aux_lem_mvr}, we have 
    \begin{eqnarray*}
        \Exp{\Phi_{t+1}} &\leq& \Exp{f(W^t)} - f^* -\frac{\gamma\lambda^p_{\min}}{2}\Exp{\sqfnorm{\nabla f(W^t)}}  - \left(\frac{1}{2\gamma} -\frac{L}{2}\right)\Exp{\sqfnorm{W^{t+1}-W^t}}\\
        &&  +\frac{\gamma\lambda^p_{\max}}{2}\Exp{\sqfnorm{G^t - \nabla f(W^t)}}+ \frac{\gamma(1-b)^2\lambda^p_{\max}}{b(2-b)}\Exp{\sqfnorm{G^{t} - \nabla f(W^{t}) }} \\
        && +  \frac{\gamma(1-b)^2L^2 \lambda^p_{\max}}{b(2-b)}\Exp{\sqfnorm{W^{t+1} - W^t}} + \frac{\gamma \lambda^p_{\max} b \sigma^2}{2-b}\\
        &\leq& \max\left\{1-\gamma\mu\lambda^p_{\min}, 1-\frac{b}{2}\right\}\Exp{\Phi_t}   + \frac{\gamma \lambda^p_{\max}b \sigma^2}{2-b}\\
        &&  - \left(\frac{1}{2\gamma} -\frac{L}{2} - \frac{\gamma(1-b)^2L^2 \lambda^p_{\max}}{b(2-b)}\right)\Exp{\sqfnorm{W^{t+1}-W^t}},
    \end{eqnarray*}
    where in the last inequality we used Assumption~\ref{as:pl_condition}.
    Selecting positive stepsize $\gamma$ satisfying the upper bound assumed in the theorem statement, we obtain
    \begin{eqnarray*}
        \Exp{\Phi_{t+1}} &\leq& \left(1-\gamma\mu\lambda^p_{\min}\right)\Exp{\Phi_{t}} + \frac{\gamma \lambda^p_{\max}b \sigma^2}{2-b}\\
        &\leq& \left(1-\gamma\mu\lambda^p_{\min}\right)^{t+1}\Exp{\Phi_{0}} + \frac{\gamma \lambda^p_{\max}b \sigma^2}{2-b}\sum^{t}_{\tau = 0}\left(1-\gamma\mu\lambda^p_{\min}\right)^{t-\tau}\\
        &\leq& \left(1-\gamma\mu\lambda^p_{\min}\right)^{t+1}\Exp{\Phi_{0}} + \frac{\gamma \lambda^p_{\max}b \sigma^2}{2-b}\sum^{\infty}_{\tau = 0}\left(1-\gamma\mu\lambda^p_{\min}\right)^{\tau}\\
        &=& \left(1-\gamma\mu\lambda^p_{\min}\right)^{t+1}\Exp{\Phi_{0}} + \frac{\gamma \lambda^p_{\max}b \sigma^2}{(2-b)\gamma \mu\lambda^p_{\min}},
    \end{eqnarray*}
    where, in the last equation, we used the formula for the sum of a geometric progression.

\end{proof}

\newpage
\subsection{Analysis of {Bernoulli-LoRA-PAGE}}\label{apx:B-LoRA-PAGE}

\begin{algorithm}[H]
\caption{\algname{Bernoulli-LoRA-PAGE}}\label{alg:Bernoulli-LoRA-PAGE}
\begin{algorithmic}[1]
\STATE \textbf{Parameters:} pre-trained model $W^0 \in \mathbb{R}^{m \times n}$, a vector $G^0 \in \in \mathbb{R}^{m \times n}$, rank $r \ll \min\{m,n\}$, scaling factor $\alpha > 0$, chain length $T$, sketch distribution $\mathcal{D}_S^B$ or $\mathcal{D}_S^A$, Bernoulli probability $p$, probability $q$

\FOR{$t = 0, 1, \ldots, T-1$}
    \STATE Sample $c^t \sim \text{Be}(p)$ \hfill{Bernoulli random variable}
    \IF{$c^t = 1$}
        \STATE Sample $B_S^t \sim \mathcal{D}_S^B$ \hfill{Left sketch}
        \STATE $\hat{A}^t = -\eta \rb{\rbtop{B_S^t}B_S^t}^{\dagger}\rbtop{B_S^t}G^t$ 
        \STATE $W^{t+1} = W^t + \frac{\alpha}{r}B_S^t\hat{A}^t$
    \ELSE
        \STATE Sample $A_S^t \sim \mathcal{D}_S^A$ \hfill{Right sketch}
        \STATE $\hat{B}^t = -\eta g(W^t)\rbtop{A_S^t}\rb{A_S^t\rbtop{A_S^t}}^{\dagger}A_S^t$
        \STATE $W^{t+1} = W^t + \frac{\alpha}{r}\hat{B}^tA_S^t$
    \ENDIF
    \STATE Sample $i_{t+1}$ uniformly at random from $[n]$
    \STATE $ G^{t+1} = \begin{cases}
        \nabla f(W^{t+1}), & \text{with probability}~~ q\\
        G^t + \left(\nabla f_{i_{t+1}}(W^{t+1}) - \nabla f_{i_{t+1}}(W^{t})\right), & \text{with probability}~~ 1-q
        \end{cases}$
\ENDFOR
\end{algorithmic}
\end{algorithm}

There exist several optimal methods for solving a general non-convex optimization problem, e.g. \algname{SPIDER}~\citep{Spider} and \algname{SARAH}~\citep{Sarah}. However, the known lower bound used to establish their optimality works only in the small data regime. \algname{ProbAbilistic Gradient Estimator (PAGE)}~\citep{Li2020PAGE} is a very simple and easy to implement algorithm, known for achieving optimal convergence results in non-convex optimization. \algname{PAGE} uses the full gradient update with probability $q_t,$ or reuses the previous gradient with a small adjustment (at a low computational cost) with probability $1 - q_t.$ A general version of \algname{PAGE} on Riemannian manifolds is considered in~\citep{demidovich2024streamliningriemannianrealmefficient}. In this section we present \algname{Bernoulli-LoRA-PAGE}, a new method within \algname{Bernoulli-LoRA} framework, based on \algname{PAGE} algorithm.

Notice, that the iterates of \algname{Bernoulli-LoRA-PAGE} (Algorithm~\ref{alg:Bernoulli-LoRA-PAGE}) can be rewritten in the following simple way
\begin{eqnarray}
    W^{t+1} &=& W^t - \gamma \hat{G}^t, \quad \text{where}\quad \hat{G}^t = \begin{cases}
        H^t_B G^t,& \text{with probability}~~ p\\
        G^t H^t_A,& \text{with probability}~~ 1-p
    \end{cases}\\
    G^{t+1} &=& \begin{cases}
        \nabla f(W^{t+1}), & \text{with probability}~~ q\\
        G^t + \left(\nabla f_{i_{t+1}}(W^{t+1}) - \nabla f_{i_{t+1}}(W^{t})\right), & \text{with probability}~~ 1-q
    \end{cases} \label{eq:est_page}
\end{eqnarray}

\begin{lemma}
\label{lem:aux_lem_page}
    Let Assumption~\ref{as:lipschitz_smoothness} hold. Then, iterates generated by \algname{Bernoulli-LoRA-PAGE}
    \begin{equation}
         \Exp{\sqfnorm{G^{t+1} - \nabla f(W^{t+1})}} \leq (1-q)\Exp{\sqfnorm{G^{t} - \nabla f(W^{t}) }} +  (1-q)L^2\Exp{\sqfnorm{W^{t+1} - W^t}}.
    \end{equation}
\end{lemma}
\begin{proof}
    Taking the full mathematical expectation, we obtain 
    \begin{eqnarray*}
        \Exp{\sqfnorm{G^{t+1} - \nabla f(W^{t+1})}} &\overset{\eqref{eq:est_page}}{=}& (1-q)\Exp{\sqfnorm{G^{t} - \nabla f(W^{t+1}) + \left(\nabla f_{i_{t+1}}(W^{t+1}) - \nabla f_{i_{t+1}}(W^{t})\right)}}\\
        &\overset{\eqref{eq:bvd}}{=}& (1-q)\Exp{\sqfnorm{G^{t} - \nabla f(W^{t}) }}\\
        &&+  (1-q)\Exp{\sqfnorm{\left(\nabla f_{i_{t+1}}(W^{t+1}) - \nabla f_{i_{t+1}}(W^{t})\right) - \left(\nabla f(W^{t+1}) - \nabla f(W^{t})\right)}}\\
        &\leq&  (1-q)\Exp{\sqfnorm{G^{t} - \nabla f(W^{t}) }}\\
        &&+  (1-q)\Exp{\sqfnorm{\nabla f_{i_{t+1}}(W^{t+1}) - \nabla f_{i_{t+1}}(W^{t})}}\\
        &\leq&  (1-q)\Exp{\sqfnorm{G^{t} - \nabla f(W^{t}) }} +  (1-q)L^2\Exp{\sqfnorm{W^{t+1} - W^t}},
    \end{eqnarray*}
    where in the last inequality we used smoothness of each $f_i$. 
\end{proof}
\subsubsection{Convergence for Smooth Non-Convex Functions}
\begin{theorem}
    Let Assumptions~\ref{as:projection_matrix},~\ref{as:bounded_below}, and ~\ref{as:lipschitz_smoothness} hold, and let the  stepsize satisfy $$
        0<\gamma \leq \frac{1}{L\left(1+\sqrt{\frac{1-q}{q}\lambda^p_{\max}}\right)}.$$ Then the iterates of \algname{PAGE-Bernoulli-LoRA} (Algorithm~\ref{alg:Bernoulli-LoRA-PAGE}) satisfy
        \begin{equation}
            \Exp{\sqfnorm{\nabla f(\widetilde{W}^T)}} \leq \frac{2(f(W^{0}) - f^*)}{\lambda^p_{\min} \gamma T} + q\frac{\sqfnorm{G^{0} - \nabla f(W^{0})}}{ T} \cdot\frac{\lambda^p_{\max}}{\lambda^p_{\min}},
        \end{equation}
        where $\lambda_{\min}^{p} := p\lambda_{\min}^{H_B} + (1-p)\lambda_{\min}^{H_A}$, $\lambda_{\max}^{p} := p\lambda_{\max}^{H_B} + (1-p)\lambda_{\max}^{H_A}$, $\widetilde{W}^T$ is drawn uniformly at random from the iterate sequence $\{W^0, W^1, \ldots, W^{T-1}\}$.
\end{theorem}
\begin{proof}
    Denote Lyapunov function $\Phi_t$ as follows 
    \begin{equation}
        \Phi_t = f(W^{t}) - f^* + \frac{\gamma \lambda^p_{\max}}{2q}\sqfnorm{G^{t} - \nabla f(W^{t})}.
    \end{equation}
    By Lemma~\ref{lem:descent_lemma_page} and Lemma~\ref{lem:aux_lem_page}, we have 
    \begin{eqnarray*}
        \Exp{\Phi_{t+1}} &\leq& \Exp{f(W^t)} - f^* -\frac{\gamma\lambda^p_{\min}}{2}\Exp{\sqfnorm{\nabla f(W^t)}} \\
        &&- \left(\frac{1}{2\gamma} -\frac{L}{2}\right)\Exp{\sqfnorm{W^{t+1}-W^t}} +\frac{\gamma\lambda^p_{\max}}{2}\Exp{\sqfnorm{G^t - \nabla f(W^t)}}\\
        &&  + \frac{\gamma\lambda^p_{\max}(1-q)}{2q}\Exp{\sqfnorm{G^{t} - \nabla f(W^{t}) }} +  \frac{\gamma \lambda^p_{\max}(1-q)L^2}{2q}\Exp{\sqfnorm{W^{t+1} - W^t}}\\
        &\leq& \Exp{\Phi_t} - \frac{\gamma\lambda^p_{\min}}{2}\Exp{\sqfnorm{\nabla f(W^t)}}  - \left(\frac{1}{2\gamma} -\frac{L}{2} - \frac{\gamma(1-q)L^2\lambda^p_{\max}}{2q}\right)\Exp{\sqfnorm{W^{t+1}-W^t}}.
    \end{eqnarray*}
    Selecting $0<\gamma \leq \frac{1}{L\left(1+\sqrt{\frac{1-q}{q}\lambda^p_{\max}}\right)}$, we obtain
    \begin{eqnarray*}
        \Exp{\Phi_{t+1}} &\leq& \Exp{\Phi_{t}} - \frac{\gamma\lambda^p_{\min}}{2}\Exp{\sqfnorm{\nabla f(W^t)}}.
    \end{eqnarray*}
    Summing over $t$ from $0$ to $T-1$, we get 
    \begin{eqnarray*}
        \frac{\gamma\lambda^p_{\min}}{2}\sum^{T-1}_{t = 0}\Exp{\sqfnorm{\nabla f(W^t)}} &\leq& \Exp{\Phi_{0}} - \Exp{\Phi_{T}}.
    \end{eqnarray*}
    Finally, dividing both sides by $\frac{\gamma\lambda^p_{\min}}{2}$ yields
    \begin{eqnarray*}
        \Exp{\sqfnorm{\nabla f(\widetilde{W}^T)}} &\leq& \frac{2\Phi_{0}}{\gamma \lambda^p_{\min}  T} .
    \end{eqnarray*}
    where $\widetilde{W}^T$ is drawn uniformly at random from the iterate sequence $\{W^0, W^1, \ldots, W^{T-1}\}$.
\end{proof}

\subsubsection{Convergence under Polyak-{\L}ojasiewicz Condition}
\begin{theorem}\label{th:B-LORA-PAGE-PL}
    Let Assumptions \ref{as:projection_matrix}, ~\ref{as:bounded_below},~\ref{as:lipschitz_smoothness}, and ~\ref{as:pl_condition} hold, and let the  stepsize satisfy $$
        0<\gamma \leq \min\left\{\frac{1}{L\left(1+2\sqrt{\frac{1-q}{q}\lambda^p_{\max}}\right)}, \frac{q}{2\mu\lambda^p_{\min}} \right\}.$$ Then the iterates of \algname{Bernoulli-LoRA-PAGE} (Algorithm~\ref{alg:Bernoulli-LoRA-PAGE}) satisfy
        \begin{equation}
            \Exp{f(W^T) - f^*} \leq (1-\gamma\mu \lambda^p_{\min})^T \Phi_{0},
        \end{equation}
        where $\lambda_{\min}^{p} := p\lambda_{\min}^{H_B} + (1-p)\lambda_{\min}^{H_A}$, and $\Phi_0 = f(W^{0}) - f^* + \frac{\gamma\lambda^p_{\max}}{q}\sqfnorm{G^{0} - \nabla f(W^{0})}$.
\end{theorem}
\begin{proof}
    Denote Lyapunov function $\Phi_t$ as follows 
    \begin{equation}
        \Phi_t = f(W^{t}) - f^* + \frac{\gamma\lambda^p_{\max}}{q}\sqfnorm{G^{t} - \nabla f(W^{t})}.
    \end{equation}
    By Lemma~\ref{lem:descent_lemma_page} and Lemma~\ref{lem:aux_lem_page}, we have 
    \begin{eqnarray*}
        \Exp{\Phi_{t+1}} &\leq& \Exp{f(W^t)} - f^* -\frac{\gamma\lambda^p_{\min}}{2}\Exp{\sqfnorm{\nabla f(W^t)}}  - \left(\frac{1}{2\gamma} -\frac{L}{2}\right)\Exp{\sqfnorm{W^{t+1}-W^t}}\\
        &&  +\frac{\gamma\lambda^p_{\max}}{2}\Exp{\sqfnorm{G^t - \nabla f(W^t)}}+ \frac{\gamma(1-q)\lambda^p_{\max}}{q}\Exp{\sqfnorm{G^{t} - \nabla f(W^{t}) }} \\
        &&+  \frac{\gamma(1-q)L^2\lambda^p_{\max}}{q}\Exp{\sqfnorm{W^{t+1} - W^t}}\\
        &\leq& (1-\gamma\mu\lambda^p_{\min})\Exp{f(W^t) - f^*} + \left(1-\frac{q}{2}\right)\frac{\gamma \lambda^p_{\max}}{q}\Exp{\sqfnorm{G^t - \nabla f(W^t)}} \\
        && - \left(\frac{1}{2\gamma} -\frac{L}{2} - \frac{\gamma(1-q)L^2\lambda^p_{\max}}{q}\right)\Exp{\sqfnorm{W^{t+1}-W^t}},
    \end{eqnarray*}
    where in the last inequality we used Assumption~\ref{as:pl_condition}.
    Selecting $0<\gamma \leq \min\left\{\frac{1}{L\left(1+2\sqrt{\frac{1-q}{q}\lambda^p_{\max}}\right)}, \frac{q}{2\mu\lambda^p_{\min}} \right\}$, we obtain
    \begin{eqnarray*}
        \Exp{\Phi_{t+1}} &\leq& (1-\gamma\mu\lambda^p_{\min})\Exp{\Phi_{t}}.
    \end{eqnarray*}
Unrolling the recursion, we obtain
    \begin{eqnarray*}
        \Exp{\Phi_{T}} &\leq& (1-\gamma\mu\lambda^p_{\min})^T\Phi_{0}.
    \end{eqnarray*}
\end{proof}

\newpage
\section{Proofs for Federated Learning Extensions}

In recent years, distributed optimization problems and algorithms have become a focal point in the Machine Learning (ML) community. This surge in interest is driven by the need to train modern deep neural networks, which often involve billions of parameters and massive datasets~\citep{LMFewShot, Kolesnikov2019BigT}. To achieve practical training times~\citep{ChuanLi}, parallelizing computations, such as stochastic gradient evaluations, has emerged as a natural solution, leading to the widespread adoption of distributed training algorithms~\citep{Goyal2017AccurateLM, You2019LargeBO, le2023bloom}. Additionally, distributed methods are crucial when data is inherently distributed across multiple devices or clients, often accompanied by privacy constraints—a common scenario in Federated Learning (FL)~\citep{Konecn2016FederatedLS, McMahan2016CommunicationEfficientLO, Kairouz2019AdvancesAO, demidovich2024methodslocalstepsrandom, sadiev2024dont, yi2024cohort}.

We develop several FL methods within the \algname{Bernoulli-LoRA} framework and provide a convergence analysis for them.

\subsection{Analysis of Fed-Bernoulli-LoRA-QGD }\label{apx:B-LoRA-QGD}
\begin{algorithm}[H]
\caption{\algname{Fed-Bernoulli-LoRA-QGD}}\label{alg:Fed-Bernoulli-LoRA-QGD}
\begin{algorithmic}[1]
\STATE \textbf{Parameters:} pre-trained model $W^0 \in \mathbb{R}^{m \times n}$, rank $r \ll \min\{m,n\}$, scaling factor $\alpha > 0$, chain length $T$, sketch distribution $\mathcal{D}_S^B$ or $\mathcal{D}_S^A$, Bernoulli probabilities $p$ and $q$

\FOR{$t = 0, 1, \ldots, T-1$}
    \FOR{any client $l \in [M]$ in parallel}
        \STATE Compute gradient $\nabla f_l(W^{t+1})$ and send compressed version $G^t_l = \cQ_l^t\left(\nabla f_l(W^{t+1})\right)$ to the server
    \ENDFOR
    \STATE $G^{t} = \frac{1}{M}\sum\limits^M_{l=1}G^{t}_l$
    \STATE Sample $c^t \sim \text{Be}(p)$ \hfill{Bernoulli random variable}
    \IF{$c^t = 1$}
        \STATE Sample $B_S^t \sim \mathcal{D}_S^B$ \hfill{Left sketch}
        \STATE $\hat{A}^t = -\eta \rb{\rbtop{B_S^t}B_S^t}^{\dagger}\rbtop{B_S^t} G^t$ 
        \STATE $W^{t+1} = W^t + \frac{\alpha}{r}B_S^t\hat{A}^t$
    \ELSE
        \STATE Sample $A_S^t \sim \mathcal{D}_S^A$ \hfill{Right sketch}
        \STATE $\hat{B}^t = -\eta G^t\rbtop{A_S^t}\rb{A_S^t\rbtop{A_S^t}}^{\dagger}$
        \STATE $W^{t+1} = W^t + \frac{\alpha}{r}\hat{B}^tA_S^t$
    \ENDIF
    \STATE Broadcast $W^{t+1}$ to each client $l \in [M]$
\ENDFOR
\end{algorithmic}
\end{algorithm}

Parallel implementations of SGD have become a prominent area of study due to their impressive scalability. However, one of the primary challenges in parallelizing SGD lies in the substantial communication overhead required to exchange gradient updates across nodes. To address this, numerous lossy compression techniques have been developed, enabling nodes to transmit quantized gradients instead of full gradients. While these methods often work well in practice, they are not universally reliable and may fail to ensure convergence.

To overcome these limitations, Quantized SGD (QSGD) by~\citet{AlistrahQGD} introduces a family of compression techniques that provide both theoretical convergence guarantees and strong empirical performance. QSGD offers a flexible mechanism for balancing communication bandwidth and convergence speed. By adjusting the number of bits transmitted per iteration, nodes can reduce bandwidth usage, albeit at the potential cost of increased variance in the gradient estimates. Different variants of QSGD were considered by~\citet{Horvth2019NaturalCF, Terngrad, panferov2024correlatedquantizationfasternonconvex}.

We consider the following distributed optimization problem:
\begin{equation*}
\min_{W\in \mathbb{R}^{m\times n}} \frac{1}{M}\sum_{l=1}^{M} f_l(W),
\end{equation*}
where  $M$  represents the number of clients. In Federated Learning, a primary bottleneck is the communication overhead between clients and the central server. A common approach to mitigate this issue is communication compression.
\begin{definition}
\label{def:unbiased_compression}
    A randomized operator  $\mathcal{Q}: \mathbb{R}^{m\times n} \rightarrow \mathbb{R}^{m\times n} $ is called an unbiased compression operator (or compressor) if there exists a constant  $\omega > 0$  such that, for any matrix  $W \in \mathbb{R}^{m\times n}$ , the following conditions hold:
    \begin{equation}
    \label{eq:unbiased_compression}
    \mathbb{E}[\mathcal{Q}(W)] = W, \quad \text{and} \quad \Exp{\sqfnorm{\mathcal{Q}(W) - W}} \leq \omega \sqfnorm{W}.
\end{equation}
\end{definition}
To analyze the optimization process, we introduce the following assumption regarding function dissimilarity:
\begin{assumption}
\label{as:function_dissimilairy}
    Let $f^* \eqdef \inf_{W} f(W)$ and $f^*_l \eqdef \inf_W f_l$ for each $l \in [M]$. In the non-convex case, the difference at the optimum is defined as:
    \begin{equation}
        \Delta^* \eqdef f^* -\frac{1}{M}\sum^M_{l=1} f^*_l \geq 0.
    \end{equation}
\end{assumption}
This assumption quantifies the discrepancy between the global optimal function value and the average of the local optimal function values between the clients.

To start convergence analysis, we rewrite the updates for $W^{t}$ and $G^t$ generated by \algname{Fed-Bernoulli-LoRA-QGD} (Algorithm~\ref{alg:Fed-Bernoulli-LoRA-QGD}) as follows 
\begin{eqnarray}
    G^t &=& \frac{1}{M}\sum^M_{l=1} \cQ_l^t\left(\nabla f_l(W^t)\right); \label{eq:est_qgd}\\
    W^{t+1} &=& W^t - \gamma \hat{G}^t, \quad \text{where}\quad \hat{G}^t = \begin{cases}
        H^t_B G^t,& \text{with probability}~~ p\\
        G^t H^t_A,& \text{with probability}~~ 1-p
    \end{cases}.
\end{eqnarray}

To establish the convergence guarantee for \algname{Fed-Bernoulli-LoRA-QGD} (Algorithm~\ref{alg:Fed-Bernoulli-LoRA-QGD}), we first demonstrate that the gradient estimator $G^t$ satisfies Assumption~\ref{as:ABC_assumption}. Once this is verified, the convergence rate follows directly using the same reasoning as in the proof of Theorem~\ref{th:B_LoRA_SGD}.
\begin{lemma}
\label{lem:aux_lemma_qgd}
    Let Assumptions~\ref{as:bounded_below},~\ref{as:lipschitz_smoothness}, and~\ref{as:function_dissimilairy} hold. Then, $G^t$ defined in Algorithm~\ref{alg:Fed-Bernoulli-LoRA-QGD} (see \eqref{eq:est_qgd}) satisfies Assumption~\ref{as:ABC_assumption} with the following constants:
    \begin{equation*}
        A_1  = \frac{L\omega}{M},\quad B_1 = 1,\quad C_1= 2\frac{L\omega \Delta^*}{M}. 
    \end{equation*}
\end{lemma}
\begin{proof}
    First, we show $G^t$ is an unbiased estimator of $\nabla f(W^t)$:
    \begin{equation*}
        \Exp{G^t|W^t} = \frac{1}{M}\sum^M_{l=1} \Exp{\cQ_l^t\left(\nabla f_l(W^t)\right)|W^t} \overset{\eqref{eq:unbiased_compression}}{=}  \frac{1}{M}\sum^M_{l=1} \nabla f_l(W^t)  = \nabla f(W^t).
    \end{equation*}
    Now we establish that $G^t$ satisfies Assumption~\ref{as:ABC_assumption}. Taking the conditional expectation with respect to $W^t$, we have 
    \begin{eqnarray*}
        \Exp{\sqfnorm{G^t}|W^t} &=&  \Exp{\sqfnorm{\frac{1}{M}\sum^M_{l=1} \cQ_l^t\left(\nabla f_l(W^t)\right) - \nabla f(W^t) + \nabla f(W^t)}|W^t}\\
        &\overset{\eqref{eq:bvd}}{=}& \Exp{\sqfnorm{\frac{1}{M}\sum^M_{l=1} \cQ_l^t\left(\nabla f_l(W^t)\right) - \nabla f(W^t) }|W^t} + \sqfnorm{\nabla f(W^t)}\\
        &=& \frac{1}{M^2}\sum^M_{l=1} \Exp{\sqfnorm{ \cQ_l^t\left(\nabla f_l(W^t)\right) - \nabla f_l(W^t) }|W^t} + \sqfnorm{\nabla f(W^t)}\\
        &\overset{\eqref{eq:unbiased_compression}}{\leq}& \frac{\omega}{M^2}\sum^M_{l=1}\sqfnorm{ \nabla f_l(W^t) } + \sqfnorm{\nabla f(W^t)}\\
        &\overset{(\ast)}{\leq}& \frac{2L\omega}{M^2}\sum^M_{l=1}\left(f_l(W^t) - f^*_l\right) + \sqfnorm{\nabla f(W^t)}\\
        &=& 2\frac{L\omega}{M} \left( f(W^t) - f^*\right) +  \sqfnorm{\nabla f(W^t)} + 2\frac{L\omega}{M} \underbrace{\left(f^* - \frac{1}{M}\sum^M_{l=1}f^*_l\right)}_{\eqdef \Delta^*},
    \end{eqnarray*}
    where in $(\ast)$ we used smoothness of each $f_l$
    Thus, we have shown that $G^t$ satisfies Assumption~\ref{as:ABC_assumption} with following constants
    \begin{equation*}
        A_1  = \frac{L\omega}{M},\quad B_1 = 1,\quad C_1= 2\frac{L\omega \Delta^*}{M}. 
    \end{equation*}
\end{proof}
\subsubsection{Convergence for Smooth Non-Convex Functions}
\begin{theorem}
  \label{th:B_LoRA_QGD}  
  Let Assumptions~\ref{as:projection_matrix}~\ref{as:bounded_below}, and ~\ref{as:lipschitz_smoothness} hold, and stepsize satisfy $$0 < \gamma \leq  \min\left\{\frac{1}{L\sqrt{ \frac{\omega}{M} \lambda^p_{\max} T}}, \frac{1}{L}\left(\frac{\lambda^p_{\max}}{\lambda^p_{\min}}\right)^{-1}\right\}.$$ Then iterates generated by \algname{Fed-Bernoulli-LoRA-QGD} (Algorithm~\ref{alg:Fed-Bernoulli-LoRA-QGD}) satisfy
  \begin{equation*}
       \Exp{\sqfnorm{\nabla f(\widetilde{W}^T)}} \leq \frac{6(f(W^0) - f^*)}{\gamma \lambda^p_{\min} T} + \frac{2\gamma L \omega \Delta^* }{M}\frac{\lambda^p_{\max}}{\lambda^p_{\min}},
  \end{equation*}
  where  $\lambda^p_{\min} \eqdef p\lambda^{H_B}_{\min} + (1-p)\lambda^{H_A}_{\min}$, $\lambda^p_{\max} \eqdef p\lambda^{H_B}_{\max} + (1-p)\lambda^{H_A}_{\max}$, and $\widetilde{W}^T$ is chosen at random from $\left\{W^0, W^1,\dots, W^{T-1}\right\}$ with probabilities $\{\frac{w_t}{\mathcal{W}_{T-1}}\}_{t = 0}^{T-1}$, where $w_t = \frac{w_{t-1}}{(1+\gamma^2L^2 \lambda^p_{\max}\nicefrac{\omega}{M})}$, $\mathcal{W}_{T-1} = \sum^{T-1}_{t=0}w_t$, and $w^{-1} > 0$. 
\end{theorem}
\begin{proof}
    By Lemma~\ref{lem:aux_lemma_qgd}, and Theorem~\ref{th:B_LoRA_SGD}, we directly obtain the statement of the theorem. 
\end{proof}
\subsubsection{Convergence under Polyak-{\L}ojasiewicz Condition}
\begin{theorem}\label{th:B-LORA-QGD-PL}
  Let Assumptions~\ref{as:projection_matrix},~\ref{as:bounded_below},~\ref{as:lipschitz_smoothness}, and ~\ref{as:pl_condition} hold, and stepsize satisfy $$0 < \gamma \leq  \min\left\{\frac{\mu}{2L^2 \nicefrac{\omega}{M}} \left(\frac{\lambda^p_{\max}}{\lambda^p_{\min}}\right)^{-1}, \frac{2}{\mu \lambda^p_{\min}}, \frac{1}{L}\left(\frac{\lambda^p_{\max}}{\lambda^p_{\min}}\right)^{-1}\right\}.$$ Then iterates generated by \algname{Fed-Bernoulli-LoRA-QGD} (Algorithm~\ref{alg:Fed-Bernoulli-LoRA-QGD}) satisfy
  \begin{equation*}
      \Exp{f(W^T) - f^*} \leq \left(1- \frac{1}{2}\gamma \mu \lambda^p_{\min}\right)^{T} \left(f(W^0) - f^*\right) + \frac{2\gamma L^2}{\mu }\cdot\frac{\omega}{M}\cdot \frac{\lambda^p_{\max}}{\lambda^p_{\min}}, 
  \end{equation*}
  where $\lambda^p_{\min} \eqdef p\lambda^{H_B}_{\min} + (1-p)\lambda^{H_A}_{\min}$, $\lambda^p_{\max} \eqdef p\lambda^{H_B}_{\max} + (1-p)\lambda^{H_A}_{\max}$.     
\end{theorem}
\begin{proof}
    By Lemma~\ref{lem:aux_lemma_qgd}, and Theorem~\ref{th:B_LORA_SGD_PL}, we directly obtain the statement of the theorem. 
\end{proof}

\newpage

\subsection{Analysis of Fed-Bernoulli-LoRA-MARINA }\label{apx:B-LoRA-MARINA}

\begin{algorithm}[H]
\caption{\algname{Fed-Bernoulli-LoRA-MARINA}}\label{alg:Fed-Bernoulli-LoRA-MARINA}
\begin{algorithmic}[1]
\STATE \textbf{Parameters:} pre-trained model $W^0 \in \mathbb{R}^{m \times n}$, $\{G^0_l\}_{l\in [M]} \in \mathbb{R}^{m \times n}$ rank $r \ll \min\{m,n\}$, scaling factor $\alpha > 0$, chain length $T$, sketch distribution $\mathcal{D}_S^B$ or $\mathcal{D}_S^A$, Bernoulli probabilities $p$ and $q$

\FOR{$t = 0, 1, \ldots, T-1$}
    \STATE Sample $c^t \sim \text{Be}(p)$ \hfill{Bernoulli random variable}
    \IF{$c^t = 1$}
        \STATE Sample $B_S^t \sim \mathcal{D}_S^B$ \hfill{Left sketch}
        \STATE $\hat{A}^t = -\eta \rb{\rbtop{B_S^t}B_S^t}^{\dagger}\rbtop{B_S^t} G^t$ 
        \STATE $W^{t+1} = W^t + \frac{\alpha}{r}B_S^t\hat{A}^t$
    \ELSE
        \STATE Sample $A_S^t \sim \mathcal{D}_S^A$ \hfill{Right sketch}
        \STATE $\hat{B}^t = -\eta G^t\rbtop{A_S^t}\rb{A_S^t\rbtop{A_S^t}}^{\dagger}$
        \STATE $W^{t+1} = W^t + \frac{\alpha}{r}\hat{B}^tA_S^t$
    \ENDIF
    \STATE Broadcast $W^{t+1}$ to each client $l \in [M]$
    \STATE Sample $s^t\sim \text{Be}(q)$
    \FOR{any client $l \in [M]$ in parallel}
        \STATE Compute gradient $\nabla f_l(W^{t+1})$
        \STATE $G^{t+1}_l = \begin{cases}
        \nabla f_l(W^{t+1}), & \text{with probability}~~ q\\
        G^t_l + \cQ_l^t\left(\nabla f_{l}(W^{t+1}) - \nabla f_l(W^{t})\right), & \text{with probability}~~ 1-q
    \end{cases} $
        \STATE Send $G^{t+1}_l$ to the server
    \ENDFOR
    \STATE $G^{t+1} = \frac{1}{M}\sum\limits^M_{l=1}G^{t+1}_l$
\ENDFOR
\end{algorithmic}
\end{algorithm}

\algname{MARINA}~\citep{MARINA} is an advanced method that significantly enhances communication efficiency in non-convex distributed learning across heterogeneous datasets. Its core innovation lies in a communication reduction mechanism that compresses the differences between gradients. The communication complexity bounds for \algname{MARINA} are known to be better than those of all previous first-order methods. Non-smooth convex analysis of \algname{MARINA} with different stepsize strategies can be found in~\citep{sokolov2024marinapsuperiorperformancenonsmooth}. This section is devoted to \algname{Fed-Bernoulli-LoRA-MARINA} (Algorithm~\ref{alg:Fed-Bernoulli-LoRA-MARINA}), a method within the \algname{Bernoulli-LoRA} framework, based on \algname{MARINA} algorithm.

In order to start convergence analysis, we rewrite the updates $W^t, G^t$ generated by \algname{Fed-Bernoulli-LoRA-MARINA}(Algorithm~\ref{alg:Fed-Bernoulli-LoRA-MARINA}):
\begin{eqnarray}
    W^{t+1} &=& W^t - \gamma \hat{G}^t, \quad \text{where}\quad \hat{G}^t = \begin{cases}
        H^t_B G^t,& \text{with probability}~~ p\\
        G^t H^t_A,& \text{with probability}~~ 1-p
    \end{cases} \label{eq:grad_step_marina}\\
    G^{t+1}_l &=& \begin{cases}
        \nabla f_l(W^{t+1}), & \text{with probability}~~ q\\
        G^t_l + \cQ_l^t\left(\nabla f_{l}(W^{t+1}) - \nabla f_l(W^{t})\right), & \text{with probability}~~ 1-q
    \end{cases} \label{eq:est_marina_1}\\
    G^{t+1} & =& \frac{1}{M}\sum^M_{l=1} G^{t+1}_l. \label{eq:est_marina_2}
\end{eqnarray}

\begin{lemma}
\label{lem:aux_lem_marina}
    Let Assumption~\ref{as:lipschitz_smoothness} hold. Then iterates generated by \algname{Fed-Bernoulli-LoRA-MARINA } satisfy
    \begin{equation}
         \Exp{\sqfnorm{G^{t+1} - \nabla f(W^{t+1})}} \leq (1-q)\Exp{\sqfnorm{G^{t} - \nabla f(W^{t}) }} +  (1-q)\frac{\omega L^2}{M}\Exp{\sqfnorm{W^{t+1} - W^t}}.
    \end{equation}
\end{lemma}
\begin{proof}
    Taking the conditional expectation with respect to $W^{t+1}$  and defining $D^{t+1}_l  \eqdef \nabla f_{l}(W^{t+1}) - \nabla f_{l}(W^{t}) $,  $D^{t+1} = \frac{1}{M}\sum^M_{l=1}D^{t+1}_l$, we obtain 
 {\small   \begin{eqnarray*}
        \Exp{\sqfnorm{G^{t+1} - \nabla f(W^{t+1})}|W^{t+1}} &=& (1-q)\Exp{\sqfnorm{G^{t} - \nabla f(W^{t}) + \frac{1}{M}\sum^M_{l=1}\cQ_l^t \left(\nabla f_{l}(W^{t+1}) - \nabla f_{l}(W^{t})\right)}|W^{t+1}}\\
        &\overset{\eqref{eq:bvd}}{=}& (1-q)\sqfnorm{G^{t} - \nabla f(W^{t}) } +  (1-q)\Exp{\sqfnorm{ \frac{1}{M}\sum^M_{l=1}\cQ_l^t \left( D^{t+1}_l\right) - D^{t+1}}  \;|\; W^{t+1}}\\
        &=&  (1-q)\sqfnorm{G^{t} - \nabla f(W^{t}) }+ \frac{1-q}{M^2}\sum^M_{m=1}\Exp{\sqfnorm{\cQ_l^t \left(D^{t+1}_l\right) - D^{t+1}_l}|W^{t+1}}\\
        &\overset{\eqref{eq:unbiased_compression}}{\leq}&  (1-q)\sqfnorm{G^{t} - \nabla f(W^{t}) } +  \frac{(1-q)\omega}{M^2}\sum^M_{l=1}\sqfnorm{\nabla f_{l}(W^{t+1}) - \nabla f_{l}(W^{t})}\\
        &\leq& (1-q)\sqfnorm{G^{t} - \nabla f(W^{t}) } +  \frac{(1-q)\omega L^2}{M}\sqfnorm{W^{t+1} - W^{t}},
    \end{eqnarray*} }
    where in the last inequality we used that the gradient of each $f_l$ is Lipschitz continuous. 
\end{proof}

\subsubsection{Convergence for Smooth Non-Convex Functions}
\begin{theorem}
    Let Assumptions~\ref{as:projection_matrix}, ~\ref{as:bounded_below}, ~\ref{as:lipschitz_smoothness}, and hold, and let the stepsize satisfy $$
        0<\gamma \leq \frac{1}{L\left(1+\sqrt{\lambda^p_{\max}\frac{1-q}{q} \cdot \frac{\omega}{M}}\right)}.$$ Then the iterates of \algname{Fed-Bernoulli-LoRA-MARINA} (Algorithm~\ref{alg:Fed-Bernoulli-LoRA-MARINA}) satisfy
        \begin{equation}
            \Exp{\sqfnorm{\nabla f(\widetilde{W}^T)}} \leq \frac{2\left(f(W^{0}) - f^*\right)}{\gamma\lambda^p_{\min}  T} + \frac{\sqfnorm{G^{0} - \nabla f(W^{0})}}{ q T} \cdot \frac{\lambda^p_{\max}}{\lambda^p_{\min}},
        \end{equation}
        where $\lambda_{\min}^{p} := p\lambda_{\min}^{H_B} + (1-p)\lambda_{\min}^{H_A}$,~$\lambda_{\max}^{p} := p\lambda_{\max}^{H_B} + (1-p)\lambda_{\max}^{H_A}$, and $\widetilde{W}^T$ is drawn uniformly at random from the iterate sequence $\{W^0, W^1, \ldots, W^{T-1}\}$.
\end{theorem}
\begin{proof}
    Denote Lyapunov function $\Phi_t$ as follows 
    \begin{equation}
        \Phi_t = f(W^{t}) - f^* + \frac{\gamma\lambda^p_{\max}}{2q}\sqfnorm{G^{t} - \nabla f(W^{t})}.
    \end{equation}
    By Lemma~\ref{lem:descent_lemma_page} and Lemma~\ref{lem:aux_lem_marina}, we have 
    \begin{eqnarray*}
        \Exp{\Phi_{t+1}} &\leq& \Exp{f(W^t)} - f^* -\frac{\gamma\lambda^p_{\min}}{2}\Exp{\sqfnorm{\nabla f(W^t)}}  - \left(\frac{1}{2\gamma} -\frac{L}{2}\right)\Exp{\sqfnorm{W^{t+1}-W^t}}\\
        &&  +\frac{\gamma \lambda^p_{\max}}{2}\Exp{\sqfnorm{G^t - \nabla f(W^t)}}+ \frac{\gamma(1-q)\lambda^p_{\max}}{2q}\Exp{\sqfnorm{G^{t} - \nabla f(W^{t}) }} \\
        && +  \frac{\gamma(1-q)L^2 \omega\lambda^p_{\max}}{2q M}\Exp{\sqfnorm{W^{t+1} - W^t}}\\
        &\leq& \Exp{\Phi_t} - \frac{\gamma\lambda^p_{\min}}{2}\Exp{\sqfnorm{\nabla f(W^t)}}  - \left(\frac{1}{2\gamma} -\frac{L}{2} - \frac{\gamma(1-q)L^2 \omega \lambda^p_{\max}}{2q M}\right)\Exp{\sqfnorm{W^{t+1}-W^t}}.
    \end{eqnarray*}
    Selecting $0<\gamma \leq \frac{1}{L\left(1+\sqrt{\lambda^p_{\max}\frac{1-q}{q} \cdot \frac{\omega}{M}}\right)}$, we obtain
    \begin{eqnarray*}
        \Exp{\Phi_{t+1}} &\leq& \Exp{\Phi_{t}} - \frac{\gamma\lambda^p_{\min}}{2}\Exp{\sqfnorm{\nabla f(W^t)}}.
    \end{eqnarray*}
    Summing over, we get 
    \begin{eqnarray*}
        \frac{\gamma\lambda^p_{\min}}{2}\sum^{T-1}_{t = 0}\Exp{\sqfnorm{\nabla f(W^t)}} &\leq& \Exp{\Phi_{0}} - \Exp{\Phi_{T}}.
    \end{eqnarray*}
    Finally, we derive 
    \begin{eqnarray*}
        \Exp{\sqfnorm{\nabla f(\widetilde{W}^T)}} &\leq& \frac{2\Phi_{0}}{\lambda^p_{\min} \gamma T} .
    \end{eqnarray*}
    where $\widetilde{W}^T$ is drawn uniformly at random from the iterate sequence $\{W^0, W^1, \ldots, W^{T-1}\}$.
\end{proof}

\subsubsection{Convergence under Polyak-{\L}ojasiewicz Condition}
\begin{theorem}\label{th:B-LORA-MARINA-PL}
    Let Assumptions \ref{as:projection_matrix}, ~\ref{as:bounded_below},~\ref{as:lipschitz_smoothness}, and ~\ref{as:pl_condition} hold, and let the  stepsize satisfy $$
        0<\gamma \leq \min\left\{\frac{1}{L\left(1+\sqrt{2\lambda^p_{\max}\frac{1-q}{q}\cdot \frac{\omega}{M}}\right)}, \frac{q}{2\mu\lambda^p_{\min}} \right\}.$$ Then the iterates of \algname{Fed-Bernoulli-LoRA-MARINA} (Algorithm~\ref{alg:Fed-Bernoulli-LoRA-MARINA}) satisfy
        \begin{equation}
            \Exp{f(W^T) - f^*} \leq (1-\gamma\mu \lambda^p_{\min})^T \Phi_{0},
        \end{equation}
        where $\lambda_{\min}^{p} := p\lambda_{\min}^{H_B} + (1-p)\lambda_{\min}^{H_A}$,  $\lambda_{\max}^{p} := p\lambda_{\max}^{H_B} + (1-p)\lambda_{\max}^{H_A}$, and $\Phi_0 = f(W^{0}) - f^* + \frac{\gamma\lambda^p_{\max}}{q}\sqfnorm{G^{0} - \nabla f(W^{0})}$.
\end{theorem}
\begin{proof}
    Denote Lyapunov function $\Phi_t$ as follows 
    \begin{equation}
        \Phi_t = f(W^{t}) - f^* + \frac{\gamma\lambda^p_{\max}}{q}\sqfnorm{G^{t} - \nabla f(W^{t})}.
    \end{equation}
    By Lemma~\ref{lem:descent_lemma_page} and Lemma~\ref{lem:aux_lem_page}, we have 
    \begin{eqnarray*}
        \Exp{\Phi_{t+1}} &\leq& \Exp{f(W^t)} - f^* -\frac{\gamma\lambda^p_{\min}}{2}\Exp{\sqfnorm{\nabla f(W^t)}}  - \left(\frac{1}{2\gamma} -\frac{L}{2}\right)\Exp{\sqfnorm{W^{t+1}-W^t}}\\
        &&  +\frac{\gamma\lambda^p_{\max}}{2}\Exp{\sqfnorm{G^t - \nabla f(W^t)}}+ \frac{\gamma(1-q)\lambda^p_{\max}}{q}\Exp{\sqfnorm{G^{t} - \nabla f(W^{t}) }} \\
        &&+  \frac{\gamma(1-q)L^2\lambda^p_{\max}}{q}\cdot\frac{\omega}{M}\Exp{\sqfnorm{W^{t+1} - W^t}}\\
        &\leq& (1-\gamma\mu\lambda^p_{\min})\Exp{f(W^t) - f^*} + \left(1-\frac{q}{2}\right)\frac{\gamma \lambda^p_{\max}}{q}\Exp{\sqfnorm{G^t - \nabla f(W^t)}} \\
        && - \left(\frac{1}{2\gamma} -\frac{L}{2} - \frac{\gamma(1-q)L^2\lambda^p_{\max}}{q}\cdot\frac{\omega}{M}\right)\Exp{\sqfnorm{W^{t+1}-W^t}},
    \end{eqnarray*}
    where in the last inequality we used Assumption~\ref{as:pl_condition}.
    Selecting $0<\gamma \leq \min\left\{\frac{1}{L\left(1+\sqrt{\frac{2(1-q)\omega}{qM}\lambda^p_{\max}}\right)}, \frac{q}{2\mu\lambda^p_{\min}} \right\}$, we obtain
    \begin{eqnarray*}
        \Exp{\Phi_{t+1}} &\leq& (1-\gamma\mu\lambda^p_{\min})\Exp{\Phi_{t}}.
    \end{eqnarray*}
    Taking recursion, we have 
    \begin{eqnarray*}
        \Exp{\Phi_{T}} &\leq& (1-\gamma\mu\lambda^p_{\min})^T\Phi_{0}.
    \end{eqnarray*}
\end{proof}

\newpage
\subsection{Analysis of Fed-Bernoulli-LoRA-EF21}\label{apx:B-LoRA-EF21}

\begin{algorithm}[H]
\caption{\algname{Fed-Bernoulli-LoRA-EF21}}\label{alg:Fed-Bernoulli-LoRA-EF21}
\begin{algorithmic}[1]
\STATE \textbf{Parameters:} pre-trained model $W^0 \in \mathbb{R}^{m \times n}$, $\{G^0_l\}_{l\in [M]} \in \mathbb{R}^{m \times n}$ rank $r \ll \min\{m,n\}$, scaling factor $\alpha > 0$, chain length $T$, sketch distribution $\mathcal{D}_S^B$ or $\mathcal{D}_S^A$, Bernoulli probability $p$

\FOR{$t = 0, 1, \ldots, T-1$}
    \STATE Sample $c^t \sim \text{Be}(p)$ \hfill{Bernoulli random variable}
    \IF{$c^t = 1$}
        \STATE Sample $B_S^t \sim \mathcal{D}_S^B$ \hfill{Left sketch}
        \STATE $\hat{A}^t = -\eta \rb{\rbtop{B_S^t}B_S^t}^{\dagger}\rbtop{B_S^t} G^t$ 
        \STATE $W^{t+1} = W^t + \frac{\alpha}{r}B_S^t\hat{A}^t$
    \ELSE
        \STATE Sample $A_S^t \sim \mathcal{D}_S^A$ \hfill{Right sketch}
        \STATE $\hat{B}^t = -\eta G^t\rbtop{A_S^t}\rb{A_S^t\rbtop{A_S^t}}^{\dagger}$
        \STATE $W^{t+1} = W^t + \frac{\alpha}{r}\hat{B}^tA_S^t$
    \ENDIF
    \STATE Broadcast $W^{t+1}$ to each client $l \in [M]$
    \FOR{any client $l \in [M]$ in parallel}
        \STATE Compute gradient $\nabla f_l(W^{t+1})$
        \STATE $G^{t+1}_l = G^t_l +\cC_l^t\left(\nabla f_l(W^{t+1}) - G^t_l\right) $
        \STATE Send $G^{t+1}_l$ to the server
    \ENDFOR
    \STATE $G^{t+1} = \frac{1}{M}\sum\limits^M_{l=1}G^{t+1}_l$
\ENDFOR
\end{algorithmic}
\end{algorithm}

\algname{Error Feedback (EF)}~\citep{Seide2014, stich2018sparsifiedsgdmemory, alistarh2018convergence, richtarik2021ef21, fatkhullin2021ef21, richtarik20223pc, khirirat2024errorfeedbackl0l1smoothnessnormalization}, often referred to as error compensation, is an exceptionally influential mechanism for stabilizing convergence in distributed training of supervised machine learning models, particularly when contractive communication compression techniques are employed. We design \algname{Fed-Bernoulli-LoRA-EF21} within the \algname{Bernoulli-LoRA} framework, based on \algname{EF-21} method. Our theoretical analysis, built on standard assumptions, applies to distributed training in heterogeneous data settings and achieves the best known convergence rates.

Compared to \algname{Fed-Bernoulli-LoRA-MARINA}, in this section we work with the wider class of compression operators called contractive. 
\begin{definition}
    \label{def:contractive_compressor}
    A randomized operator $\cC: \R^{m\times n} \rightarrow \R^{m\times n}$ is called a contractive compression operator (compressor) if it satisfies the following condition: there exists a constant $0< \beta \leq 1$ such that
    \begin{equation}
    \label{eq:contractive_compression}
        \Exp{\sqfnorm{\cC\left(W\right) - W}} \leq (1-\beta) \sqfnorm{W}, \quad \forall~W\in \R^{m\times n}. 
    \end{equation}
\end{definition}

The iterates of \algname{Fed-Bernoulli-LoRA-EF21} can be rewritten as follows 
\begin{eqnarray}
    W^{t+1} &=& W^t - \gamma \hat{G}^t, \quad \text{where}\quad \hat{G}^t = \begin{cases}
        H^t_B G^t,& \text{with probability}~~ p\\
        G^t H^t_A,& \text{with probability}~~ 1-p 
    \end{cases} \\
    G^{t+1}_l &=& G_l^t + \cC_l^t\left(\nabla f_l(W^{t+1}) - G^t_l\right),\quad \forall ~l\in [M] \label{eq:est_ef21_1}\\
    G^{t+1} & =& \frac{1}{M}\sum^M_{l=1} G^{t+1}_l. \label{eq:est_ef21_2}
\end{eqnarray}

\begin{lemma}
\label{lem:aux_lemma_ef21}
    Let Assumption~\ref{as:lipschitz_smoothness} hold. Then for the iterates generated by \algname{Fed-Bernoulli-LoRA-EF21} (Algorithm~\ref{alg:Fed-Bernoulli-LoRA-EF21})satisfy
    \begin{equation*}
        \Exp{\sqfnorm{G^{t+1}_l - \nabla f_l(W^{t+1})}} \leq  \sqrt{1-\beta} \Exp{\sqfnorm{G^{t}_l - \nabla f_l(W^{t})}} + \frac{(1-\beta)L^2}{1-\sqrt{1-\beta}}\Exp{\sqfnorm{W^{t+1} - W^t}}
    \end{equation*}
\end{lemma}
\begin{proof}
    For each $l \in [M]$  we have 
    \begin{eqnarray*}
        \Exp{\sqfnorm{G^{t+1}_l- \nabla f_l(W^{t+1})}} &\overset{\eqref{eq:est_ef21_1},\eqref{eq:est_ef21_2}}{=}& \Exp{\Exp{\sqfnorm{\cC_l^t\left(\nabla f_l(W^{t+1}) - G^t_l\right) - \left(\nabla f_l(W^{t+1}) - G^t_l\right)}|G^{t+1}_l, W^{t+1}}}\\
        &\overset{\eqref{eq:contractive_compression}}{\leq}& \left(1-\beta\right) \Exp{\sqfnorm{G^{t}_l - \nabla f_l(W^{t+1})}}\\
        &\leq&  \left(1-\beta\right)\left(1 +\theta\right) \Exp{\sqfnorm{G^{t}_l - \nabla f_l(W^{t})}} \\
        &&+ \left(1-\beta\right)\left(1 +\frac{1}{\theta}\right) \Exp{\sqfnorm{\nabla f_l(W^{t+1}) - \nabla f_l(W^t)}},
    \end{eqnarray*}
    where in the last inequality we used $\sqfnorm{U+V} \leq \left(1+\theta\right)\sqfnorm{U}+\left(1+\frac{1}{\theta}\right)\sqfnorm{V}$ for any constant $\theta >0$, and matrices $U,V \in \R^{m\times n}$. Taking $\theta = \frac{1}{\sqrt{1-\beta}} - 1$, we acquire 
    \begin{eqnarray*}
        \Exp{\sqfnorm{G^{t+1}_l- \nabla f_l(W^{t+1})}} &\leq& \sqrt{1-\beta}\Exp{\sqfnorm{G^{t}_l - \nabla f_l(W^{t})}} + \frac{1-\beta}{1-\sqrt{1-\beta}} \Exp{\sqfnorm{\nabla f_l(W^{t+1}) - \nabla f_l(W^t)}}\\
        &\leq&  \sqrt{1-\beta} \Exp{\sqfnorm{G^{t}_l - \nabla f_l(W^{t})}} + \frac{(1-\beta) L^2}{1-\sqrt{1-\beta}}  \Exp{\sqfnorm{W^{t+1} - W^t}},
    \end{eqnarray*}
    where in the last inequality we used that the gradient of each $f_l$ is Lipschitz continuous. Summing over  $l$ from $1$ to $M$, we finish the proof. 
\end{proof}

\subsubsection{Convergence for Smooth Non-Convex Functions}
\begin{theorem}
    Let Assumptions~\ref{as:projection_matrix},~\ref{as:bounded_below}, and~\ref{as:lipschitz_smoothness} hold, and let the  stepsize satisfy $$0<\gamma \leq \frac{1}{L\left(1+\frac{\sqrt{\lambda^p_{\max} (1-\beta)}}{1-\sqrt{1-\beta}} \right)}.$$  Then the iterates of \algname{Fed-Bernoulli-LoRA-EF21} (Algorithm~\ref{alg:Fed-Bernoulli-LoRA-EF21}) satisfy
        \begin{equation}
            \Exp{\sqfnorm{\nabla f(\widetilde{W}^T)}} \leq \frac{2(f(W^{0}) - f^*)}{\gamma \lambda^p_{\min} T} +  \frac{\cG^0}{ (1-\sqrt{1-\beta} )T}\cdot\frac{\lambda^p_{\max}}{\lambda^p_{\min}},
        \end{equation}
        where $\lambda_{\min}^{p} := p\lambda_{\min}^{H_B} + (1-p)\lambda_{\min}^{H_A}$, and $\lambda_{\max}^{p} := p\lambda_{\max}^{H_B} + (1-p)\lambda_{\max}^{H_A}$,~ $\widetilde{W}^T$ is drawn uniformly at random from the iterate sequence $\{W^0, W^1, \ldots, W^{T-1}\}$, and $\cG^0 \eqdef  \frac{1}{M}\sum^M_{l=1}\sqfnorm{G^{0}_l- \nabla f_l(W^{0})}$.
\end{theorem}
\begin{proof}
    Denote Lyapunov function $\Phi_t$ as follows 
    \begin{equation}
        \Phi_t = f(W^{t}) - f^* + \frac{\gamma\lambda^p_{\max}}{2(1-\sqrt{1-\beta})}\cdot\frac{1}{M}\sum^M_{l=1}\sqfnorm{G^{t}_l - \nabla f_l(W^{t}) }.
    \end{equation}
    By Lemma~\ref{lem:descent_lemma_page} and Lemma~\ref{lem:aux_lemma_ef21}, we have 
    \begin{eqnarray*}
        \Exp{\Phi_{t+1}} &\leq& \Exp{f(W^t)} - f^* -\frac{\gamma\lambda^p_{\min}}{2}\Exp{\sqfnorm{\nabla f(W^t)}}  - \left(\frac{1}{2\gamma} -\frac{L}{2}\right)\Exp{\sqfnorm{W^{t+1}-W^t}}\\
        &&  +\frac{\gamma \lambda^p_{\max}}{2}\Exp{\sqfnorm{G^t - \nabla f(W^t)}}+ \frac{\gamma\lambda^p_{\max}\sqrt{1-\beta}}{2(1-\sqrt{1-\beta})}\cdot\frac{1}{M}\sum^M_{l=1}\Exp{\sqfnorm{G^{t}_l - \nabla f_l(W^{t}) }} \\
        && +  \frac{\gamma\lambda^p_{\max} L^2(1-\beta)}{2(1-\sqrt{1-\beta})^2}\Exp{\sqfnorm{W^{t+1} - W^t}}\\
        &\leq& \Exp{\Phi_t} - \frac{\gamma\lambda^p_{\min}}{2}\Exp{\sqfnorm{\nabla f(W^t)}}  - \left(\frac{1}{2\gamma} -\frac{L}{2} - \frac{\gamma\lambda^p_{\max} L^2(1-\beta)}{2(1-\sqrt{1-\beta})^2}\right)\Exp{\sqfnorm{W^{t+1}-W^t}}.
    \end{eqnarray*}
    Selecting $0<\gamma \leq \frac{1}{L\left(1+\frac{\sqrt{\lambda^p_{\max} (1-\beta)}}{1-\sqrt{1-\beta}} \right)}$, we obtain
    \begin{eqnarray*}
        \Exp{\Phi_{t+1}} &\leq& \Exp{\Phi_{t}} - \frac{\gamma\lambda^p_{\min}}{2}\Exp{\sqfnorm{\nabla f(W^t)}}.
    \end{eqnarray*}
    Summing over $t$ from $0$ to $T-1$, we get 
    \begin{eqnarray*}
        \frac{\gamma\lambda^p_{\min}}{2}\sum^{T-1}_{t = 0}\Exp{\sqfnorm{\nabla f(W^t)}} &\leq& \Exp{\Phi_{0}} - \Exp{\Phi_{T}}.
    \end{eqnarray*}
    Finally, dividing both sides by $\frac{\gamma\lambda^p_{\min}}{2}$ yields
    \begin{eqnarray*}
        \Exp{\sqfnorm{\nabla f(\widetilde{W}^T)}} &\leq& \frac{2\Phi_{0}}{\gamma\lambda^p_{\min}  T} .
    \end{eqnarray*}
    where $\widetilde{W}^T$ is drawn uniformly at random from the iterate sequence $\{W^0, W^1, \ldots, W^{T-1}\}$.
\end{proof}

\subsubsection{Convergence under Polyak-{\L}ojasiewicz Condition}
\begin{theorem}\label{th:B-LORA-EF21-PL}
    Let Assumptions \ref{as:projection_matrix}, ~\ref{as:bounded_below},~\ref{as:lipschitz_smoothness}, and ~\ref{as:pl_condition} hold, and let the  stepsize satisfy $$0<\gamma \leq \min\left\{\frac{1}{L\left(1+\frac{\sqrt{2\lambda^p_{\max}(1-\beta)}}{1-\sqrt{1-\beta}}\right)}, \frac{1+\sqrt{1-\beta}}{2\mu\lambda^p_{\min}} \right\}$$. Then the iterates of \algname{Fed-Bernoulli-LoRA-EF21} (Algorithm~\ref{alg:Fed-Bernoulli-LoRA-EF21}) satisfy
        \begin{equation}
            \Exp{f(W^T) - f^*} \leq (1-\gamma\mu \lambda^p_{\min})^T \Phi_{0},
        \end{equation}
        where $\lambda_{\min}^{p} := p\lambda_{\min}^{H_B} + (1-p)\lambda_{\min}^{H_A}$, $\lambda_{\max}^{p} := p\lambda_{\max}^{H_B} + (1-p)\lambda_{\max}^{H_A}$, and $\Phi_0 = f(W^{0}) - f^* + \frac{\gamma\lambda^p_{\max}}{1-\sqrt{1-\beta}}\frac{1}{M}\sum^M_{l=1}\sqfnorm{G^{0}_l - \nabla f_l(W^{0})}.$
\end{theorem}
\begin{proof}
    Denote Lyapunov function $\Phi_t$ as follows 
    \begin{equation}
        \Phi_t = f(W^{t}) - f^* + \frac{\gamma\lambda^p_{\max}}{1-\sqrt{1-\beta}}\cdot \frac{1}{M}\sum^M_{l=1}\sqfnorm{G^{t}_l - \nabla f_l(W^{t})}.
    \end{equation}
    By Lemma~\ref{lem:descent_lemma_page} and Lemma~\ref{lem:aux_lemma_ef21}, we have 
    \begin{eqnarray*}
        \Exp{\Phi_{t+1}} &\leq& \Exp{f(W^t)} - f^* -\frac{\gamma\lambda^p_{\min}}{2}\Exp{\sqfnorm{\nabla f(W^t)}}  - \left(\frac{1}{2\gamma} -\frac{L}{2}\right)\Exp{\sqfnorm{W^{t+1}-W^t}}\\
        &&  +\frac{\gamma\lambda^p_{\max}}{2}\cdot \Exp{\sqfnorm{G^t - \nabla f(W^t)}}+ \frac{\gamma\lambda^p_{\max}\sqrt{1-\beta}}{1-\sqrt{1-\beta}}\cdot\frac{1}{M}\sum^{M}_{l=1}\Exp{\sqfnorm{G^{t}_l - \nabla f_l(W^{t}) }} \\
        &&+  \frac{\gamma \lambda^p_{\max} (1-\beta)L^2}{(1-\sqrt{1-\beta})^2}\Exp{\sqfnorm{W^{t+1} - W^t}}\\
        &\leq& (1-\gamma\mu\lambda^p_{\min})\Exp{f(W^t) - f^*} + \frac{\gamma \lambda^p_{\max} \left(1+\sqrt{1-\beta}\right)}{2(1-\sqrt{1-\beta})}\cdot\frac{1}{M}\sum^M_{l=1}\Exp{\sqfnorm{G^t_l - \nabla f_l(W^t)}} \\
        && - \left(\frac{1}{2\gamma} -\frac{L}{2} - \frac{\gamma \lambda^p_{\max} (1-\beta)L^2}{(1-\sqrt{1-\beta})^2}\right)\Exp{\sqfnorm{W^{t+1}-W^t}},
    \end{eqnarray*}
    where in the last inequality we used Assumption~\ref{as:pl_condition}.
    Selecting $0<\gamma \leq \min\left\{\frac{1}{L\left(1+\frac{\sqrt{2\lambda^p_{\max}(1-\beta)}}{1-\sqrt{1-\beta}}\right)}, \frac{1+\sqrt{1-\beta}}{2\mu\lambda^p_{\min}} \right\}$, we obtain
    \begin{eqnarray*}
        \Exp{\Phi_{t+1}} &\leq& (1-\gamma\mu\lambda^p_{\min})\Exp{\Phi_{t}}.
    \end{eqnarray*}
    Taking the recursion, we have 
    \begin{eqnarray*}
        \Exp{\Phi_{T}} &\leq& (1-\gamma\mu\lambda^p_{\min})^T\Phi_{0}.
    \end{eqnarray*}
\end{proof}

\newpage

complete it was that from new reps

\section{Experiments: Missing Details}\label{sec:experiments_extra}

In this section, we provide additional details regarding the experimental setting from Section~\ref{sec:exps}. 
\subsection{Linear Regression with Non-convex Regularization}

\textbf{Full gradient setting.}
We begin by evaluating these methods in a standard optimization setting where full gradients are computed at each iteration. In this regime, we compare \algname{Bernoulli-LoRA-GD} and \algname{RAC-LoRA-GD}. 
\begin{figure}[h]
    \centering
    \begin{subfigure}[b]{0.49\textwidth}
        \centering
        \includegraphics[width=\textwidth]{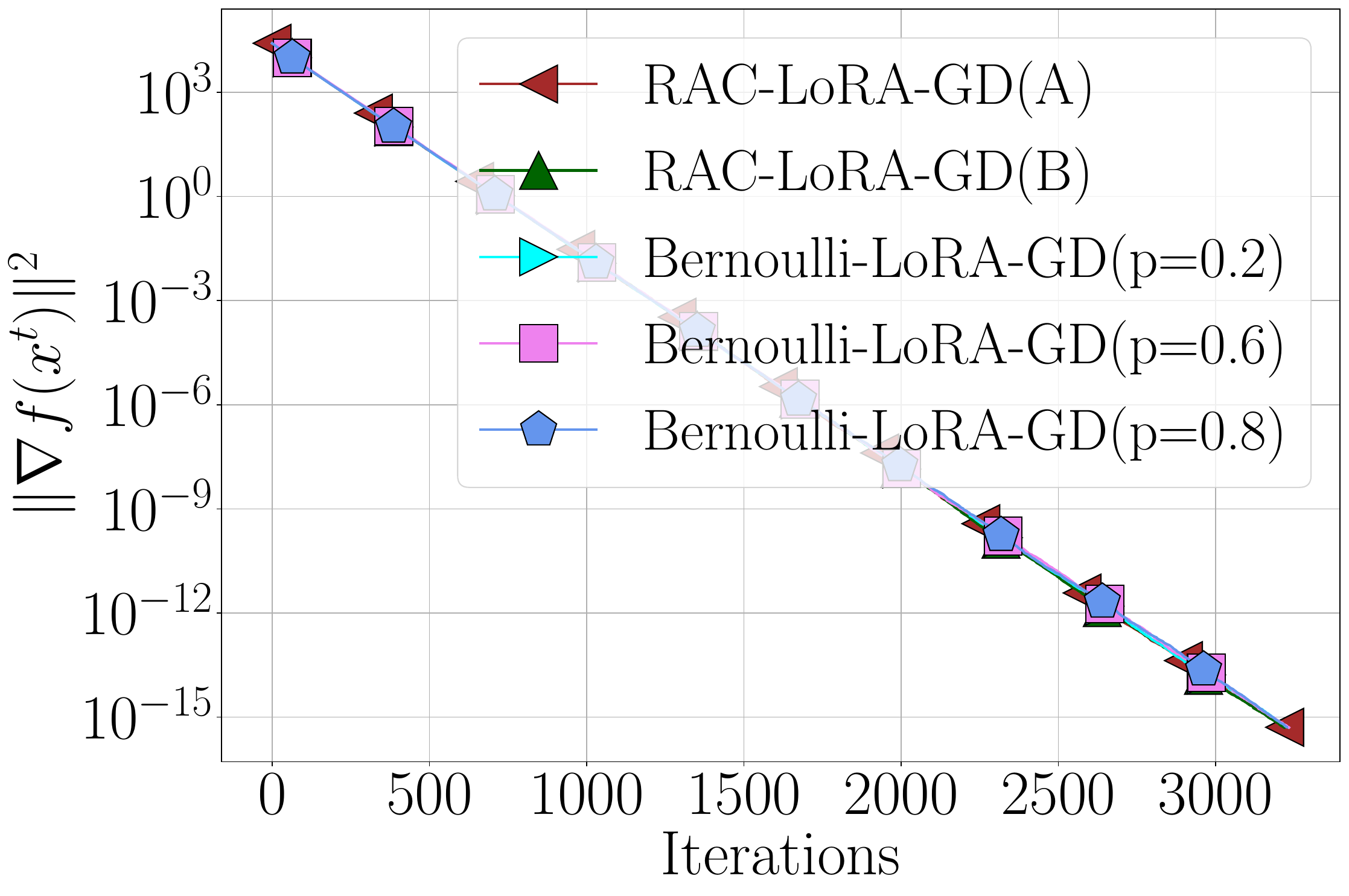}
        \caption{Rank $r=1$.}
    \end{subfigure}
    \begin{subfigure}[b]{0.49\textwidth}
        \centering
        \includegraphics[width=\textwidth]{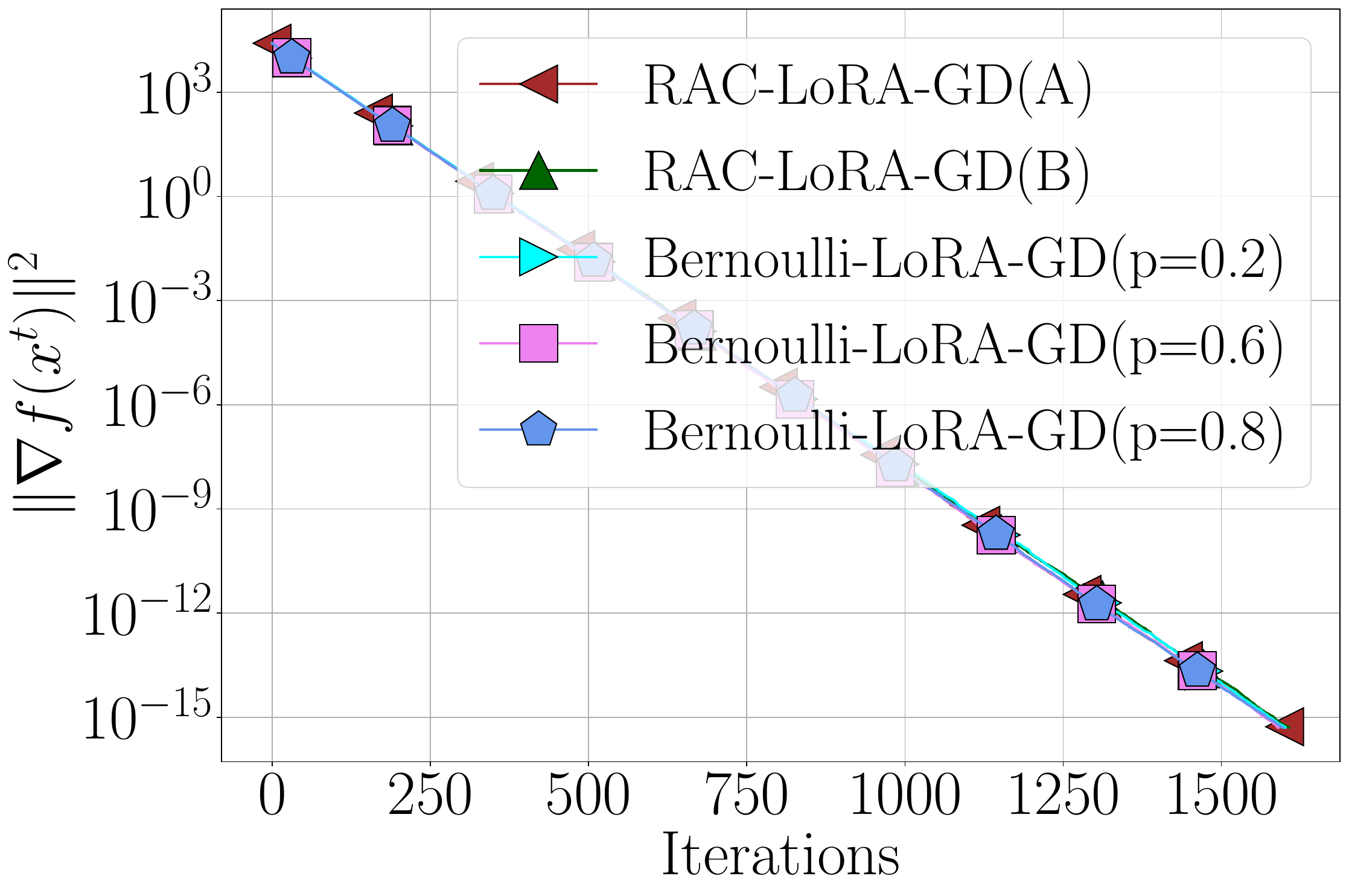}
        \caption{Rank $r=2$.}
    \end{subfigure}
    \caption{Comparison of \algname{RAC-LoRA-GD} and \algname{Bernoulli-LoRA-GD} on linear regression fine-tuning. Curves with $p=0.01,0.2,\dots$ indicate \algname{Bernoulli-LoRA-GD} sampling parameters. \algname{RAC-LoRA-GD(A)} trains $B$ after resampling $A$, while \algname{RAC-LoRA-GD(B)} does the reverse. All methods use $\gamma = \nfr{c}{\hat{L}}$ with $c\in\{1,2\}$ tuned individually.}
    \label{fig:bernoulli_lora_gd}
\end{figure}

Figure~\ref{fig:bernoulli_lora_gd} shows that, across all tested probabilities, \algname{Bernoulli-LoRA-GD} and both variants of \algname{RAC-LoRA-GD} exhibit similar convergence on the linear regression task. This numerical stability suggests that the ratio of updates between $A$ and $B$ has little effect on the performance for this problem. We also observe that higher ranks $r$ produce faster convergence, which aligns with the theoretical $\nfr{r}{n}$ factor in our analysis.

\paragraph{Hardware and Software.}
All algorithms were implemented in Python~3.10 and executed on three different CPU cluster node types: 
\begin{enumerate}
    \item AMD EPYC 7702 64-Core,
    \item Intel(R) Xeon(R) Gold 6148 CPU @ 2.40GHz,
    \item Intel(R) Xeon(R) Gold 6248 CPU @ 2.50GHz.
\end{enumerate}

\paragraph{Implementation Details.}
For each method, we set the stepsize to $\gamma = c / \hat{L}$, where $c$ is a constant multiplier tuned individually for every algorithm. Convergence was monitored by computing the squared norm of the full gradient at each iteration. The algorithms terminated when either a maximum iteration limit was reached or the criterion $\sqnorm{\nabla f(x^t)} \leq 5 \times 10^{-16}$ was satisfied. To ensure reliability, each method was run $20$ times using different random seeds, and all figures show the median performance over these trials.

\paragraph{Datasets.}
The synthetic pre-training dataset $(\widetilde{D}, \widetilde{b})$ was generated using 
\begin{center}
    \texttt{sklearn.datasets.make\_regression}
\end{center}

with moderate noise and a controlled rank structure:
\begin{lstlisting}
wt_D, wt_b = make_regression(n_samples=90000, n_features=4096,
                            n_informative=4096, noise=20.0, 
                            bias=0.0, tail_strength=0.8,
                            effective_rank=64, random_state=42)
\end{lstlisting}
followed by standard scaling. The fine-tuning dataset $(\hat{D}, \hat{b})$ was produced similarly:
\begin{lstlisting}
h_D, h_b = make_regression(n_samples=10000, n_features=4096,
                          n_informative=4096//2, noise=50.0, 
                          bias=10.0, tail_strength=0.9,
                          effective_rank=32, random_state=84)
\end{lstlisting}
and subsequently adjusted with a biased scaling (mean~$1$, standard deviation~$2$).


\end{document}